\documentclass{article} 
\usepackage{iclr2023_conference,times}

\usepackage[hidelinks]{hyperref}
\usepackage{url}

\usepackage[utf8]{inputenc} 
\usepackage[T1]{fontenc}    
\usepackage{booktabs}       
\usepackage{amsfonts}       
\usepackage{nicefrac}       
\usepackage{microtype}      
\usepackage{xcolor}         

\usepackage{wrapfig}
\usepackage{graphicx}
\usepackage{amsmath,bm,bbm}
\usepackage{amsfonts}
\usepackage{xcolor}         
\usepackage{amsthm}        
\usepackage[algoruled,lined,linesnumbered]{algorithm2e} 
\usepackage[capitalize,nameinlink,noabbrev]{cleveref}
\usepackage[skip=0pt,belowskip=-2pt]{caption}
\usepackage{subcaption}
\usepackage{chngcntr}
\usepackage{dblfloatfix}    

\newtheorem{theorem}{Theorem}[section]

\newtheorem{lemma}[theorem]{Lemma}

\theoremstyle{definition}

\theoremstyle{remark}

\newcommand{\E}{\mathbb{E}}

\newcommand{\statespace}{\mathcal{S}}
\newcommand{\actionspace}{\mathcal{A}}

\newcommand{\optionspace}{\mathbf{O}}

\newcommand{\policyspace}{\mathbf{F}}
\newcommand{\numoptions}{k}

\newcommand{\optionset}{\calO}
\newcommand{\policyset}{\calF}
\newcommand{\optionsetspace}{\optionspace^\numoptions}
\newcommand{\policysetspace}{\policyspace^{\cardN}}

\DeclareMathOperator*{\argmax}{\arg\!\max}

\def\vtheta{{\bm{\theta}}}

\def\vv{{\mathbf{v}}}
\def\vd{{\mathbf{d}}}

\def\vr{{\mathbf{r}}}

\def\mP{{\mathbf{P}}}
\def\mI{{\mathbf{I}}}

\def\valpha{{\bm{\alpha}}}
\newcommand{\bbE}{\mathbb{E}}
\newcommand{\bbR}{\mathbb{R}}

\newcommand{\bbI}{\mathbb{I}}
\newcommand{\calS}{\mathcal{S}}
\newcommand{\calA}{\mathcal{A}}
\newcommand{\calR}{\mathcal{R}}
\newcommand{\calO}{\mathcal{O}}
\newcommand{\calI}{\mathcal{I}}

\newcommand{\calF}{\mathcal{F}}
\newcommand{\calH}{\mathcal{H}}

\newcommand{\calN}{\mathcal{N}}
\newcommand{\calP}{\mathcal{P}}
\newcommand{\calX}{\mathcal{X}}

\newcommand{\bfI}{\mathbf{I}}

\newcommand{\abs}[1]{\left\lvert#1\right\rvert}
\newcommand{\cardS}{\abs{\calS}}
\newcommand{\cardA}{\abs{\calA}}

\newcommand{\cardH}{\abs{\calH}}
\newcommand{\cardN}{\abs{\calN}}

\title{Toward Discovering Options that Achieve Faster Planning}

\author{%
  Yi Wan\\
  University of Alberta\\
  Edmonton, Canada\\
  \texttt{\{wan6\}@ualberta.ca}
  \And
  Richard S. Sutton\\
  University of Alberta, DeepMind\\
  Edmonton, Canada\\
  \texttt{\{rsutton\}@ualberta.ca}
}

\iclrfinalcopy 
\begin{document}

\maketitle

\begin{abstract}
We propose a new objective for option discovery that emphasizes the computational advantage of using options in planning. In a sequential machine, the speed of planning is proportional to the number of elementary operations used to achieve a good policy. For episodic tasks, the number of elementary operations depends on the number of options composed by the policy in an episode and the number of options being considered at each decision point. To reduce the amount of computation in planning, for a given set of episodic tasks and a given number of options, our objective prefers options with which it is possible to achieve a high return by composing few options,
and also prefers a smaller set of options to choose from at each decision point.
We develop an algorithm that optimizes the proposed objective. In a variant of the classic four-room domain, we show that 1) a higher objective value is typically associated with fewer number of elementary planning operations used by the option-value iteration algorithm to obtain a near-optimal value function, 2) our algorithm achieves an objective value that matches it achieved by two human-designed options 3) the amount of computation used by option-value iteration with options discovered by our algorithm matches it with the human-designed options, 4) the options produced by our algorithm also make intuitive sense--they seem to move to and terminate at the entrances of rooms.
\end{abstract}

\section{Introduction}
\label{sec: Introduction}
The options framework (Sutton, Precup, Singh 1999) is a way to achieve temporal abstraction, which is perceived as a cornerstone of artificial intelligence. The options are extended courses of actions, defining different ways of behaving. If the agent has a set of options, it could learn a model that predicts the outcomes of executing the options. One of the most important ways options can be useful is that planning with option models could be much faster than planning with action models because options specify jumpy moves. 

A natural question to ask is where options come from, which is a challenging research frontier (Section 17.2, Sutton \& Barto 2018). The first and maybe the most important step toward the option discovery problem is to specify what options should be discovered, which involves defining an objective that can be used to rank options. Existing works provide various objectives. For example, some works argue that good options should lead to subgoal states in the environment and the subgoal states could be, for example, bottlenecks (McGovern \& Barto 2001, Menache et al. 2002, van Dijk \& Polani 2011, Bacon 2013) or salient events (Singh, Barto, Chentanez 2010) in the environment. Some works propose that good options are those such that, when choosing from them, the agent can achieve high performance in a given task(s) (Bacon, Harb, Precup 2017, Khetarpal et al. 2020, Veeriah et al. 2021). Others argue that a good set of options should, for example, be diversified and parts of solutions that achieve a high expected return (eg., Eysenbach et al.\ 2018), accelerate learning in future tasks (Brunskill \& Li 2014), be represented efficiently (Solway et al.\ 2014, Harutyunyan et al.\ 2019), allow representing the value function easier (Konidaris \& Barto 2009), 
improve exploration (Machado et al.\ 2017a;b, Jinnai et al.\ 2019), etc. Among all of the existing objectives, one (Jinnai et al.\ 2019) emphasizes the importance of the role of options in planning. They proposed an algorithm that searches for the smallest set of "point options" such that planning converges in less than a given number of iterations. A point option is a special kind of option -- it can only be initialized at a single state and be terminated at a (possibly different) state. However, it is not appealing to restrict our focus on point options because more general options could achieve much faster planning.



We propose a simple objective that also emphasizes the importance of options in planning and the discovered options are general. We start by considering that all options can be initiated at any state. In this case, our objective prefers options with which fewer options are required to compose good solutions to multiple given tasks. The multi-task setting is critical to our objective. In fact, our objective reduces to the objective proposed by Harb et al.(2018) if there is only one task. We argue in \cref{sec: An Option Discovery Objective for Fast Planning} that, with only one task, the best options degenerate to the entire solution to the given task and are typically not useful for other tasks that we may want to solve with the options. 

For each state, some options may be interesting while others are not. It is appealing for planning to only consider interesting options to save computation. For example, if the agent is full then it is not interesting for the agent to consider how to find food. In order to learn to reduce the number of options considered at each state, we generalize the definition of options by replacing the initiation set of an option with an \emph{interest function}, which maintains the probability of initiating the option at each state. The option's initiation set can then be obtained by sampling according to the interest function and becomes stochastic. Our complete objective encourages options' interests to be smaller. Therefore uninteresting options are less likely to be sampled and considered in planning, resulting in less computation.

To optimize the proposed objective, we derives the gradient of the objective and propose an algorithm that follows an estimate of the gradient. If there is only one task and all options can be initiated everywhere, our algorithm reduces to Harb et al'.s algorithm. We tested our algorithm in the four-room domain, which is typically used to evaluate option discovery algorithms (e.g. Bacon et al.\ 2017, Harb et al.\ 2018). Empirical results show that the algorithm's discovered options are comparable with two human-designed options, which move to entrances of the four rooms, in terms of the number of elementary operations used to achieve a near-optimal value function used by option-value iteration.

\section{The Options Framework with Multiple Tasks} \label{sec: problem setting}

Our new option discovery objective involves solving a finite set of episodic tasks $\calN$, all of which share the same state space $\calS$ and action space $\calA$. The state and action spaces are finite. For each task $n \in \calN$, there is a corresponding set of terminal states $\perp^n$, which is a set containing one or multiple states in $\calS$. The transition dynamics are shared across different tasks, except that starting from a terminal state of a task the agent stays at the state regardless of the action it takes. The reward settings for different tasks are typically different. Let $p^n: \calS \times \calR \times \calS \times \calA \to [0, 1]$ be the transition function of task $n$ with $p^n(s', r \mid s, a) \doteq$ the probability of resulting in state $s'$ and reward $r$ given state-action pair $(s, a)$ and task $n$. 

The agent's interaction with the environment produces a sequence of episodes. The first episode starts from state $S_0$ sampled from an initial state distribution $d_0$. At the same time, a task $N_0$ is chosen randomly from $\calN$ and remains unchanged within the episode. The agent observes $S_0$ and $N_0$ and takes an action $A_0$. The environment then emits task $N_0$'s reward $R_1$ and the next state $S_1$ according to the transition function $p^{N_0}$. 
Such an agent-environment interaction keeps going until the end of the episode when the agent reaches a terminal state in $\perp^{N_0}$. 
Let the time step at which the first episode terminates $T$. A new episode starts at $T+1$, with a new initial state $S_{T+1}$ sampled from $d_0$ and a new task $N_{T+1}$ sampled from $\calN$. The agent may choose its actions following a stationary policy $\pi: \calA \times \calS \to [0, 1]$ with $\pi:(a \mid s) \doteq$ the probability of choosing action $a$ at state $s$. Denote the set of stationary policies $\Pi$.






Given a task $n$, assume that every stationary policy reaches a terminal state with positive probability in at most $\abs{\calS \backslash \perp^n}$ steps, regardless of the initial state. Under this assumption, the value function of a policy $\pi \in \Pi$ in task $n \in \calN$, $v_\pi^n(s)$, can defined: $v_\pi^n(s) \doteq \E[R_{1} + \dots + \gamma^{T - 1} R_T \mid S_0 = s, N_0 = n, A_{0:T-1} \sim \pi]$, where $T$ is a random variable denoting the time step at which the episode terminates and $\gamma \in [0, 1]$ is the discount factor. The optimal value function of task $n$ is defined to be: $v_*^n(s) \doteq \max_{\pi \in \Pi} v_\pi^n(s)$. 
Here the $\max$ always exists. 
Standard reinforcement learning algorithms like Q-learning can be applied to obtain optimal values for these tasks (Section 5.6 Bertsekas \& Tsitsiklis 1996). 

The agent may maintain a set of options, which are behaviors that typically take more than one-step to finish (Sutton et al., 1999). An option $o$ is a tuple $(\calI_o, \pi_o, \beta_o)$, where $\calI_o \subseteq \calS$ is the option's initiation set, $\pi_o \in \Pi$ is the option' policy, and $\beta_o \in \Gamma$, where $\Gamma: \{f \mid f: \calS \to [0, 1]\}$, is the option's termination function. Denote $\optionspace$ as the space of all possible options. That is, $\optionspace \doteq \Gamma \times \Pi \times \Gamma$.

Suppose the agent has a set of $k$ adjustable options, which can be learned from data, in addition to $\cardA$ primitive actions (non-adjustable options), making it a total of $\cardA + k$ options. Denote $\optionset$ as the entire set of options the agent has.
Let $\calH \doteq \{1, 2, \dots, k + \cardA \}$ denote the set of indices of options \footnote{Throughout the paper, all sets are indexed and given a set $\calX$, we use the notation $x_i$ to denote the $i$-th element of $\calX$.} and $\calH^{adj} \doteq \{1, 2, \dots, k\}$ denote the set of indices of adjustable options. Therefore options with indices $k+1, \dots, k + \cardA$ are the primitive actions.
For each task $n$, we define a \emph{meta-preference function} $f^n \in \calS \times \calH \to [0, \infty)$, representing the agent's willing to choose each option at every state for task $n$. Let $\calF \doteq \{f^n : n \in \calN \}$ be the set of mete-preference functions. 
Note that a meta-preference function chooses from indices of options rather than options themselves because the agent's options can change over time and while the indices don't. 

Respecting the fact that the agent chooses from option indices, we make the following definitions for precise presentation.
For each $s$, define the set of indices whose associated options can be initiated at $s$: $\Omega(s) \doteq \{o \in \optionset : s \in \calI_o \}$. Define a function $\pi_\optionset: \calA \times \calS \times \calH \to [0, 1]$ with $\pi_\optionset(a \mid s, h) \doteq \pi_{o_h} (a \mid s)$ $\forall a \in \calA, s \in \calS, h \in \calH$,
and a function $\beta_\optionset: \calS \times \calH \to [0, 1]$ with $\beta_\optionset(s, h) \doteq \beta_{o_h}(s)$, $\forall s \in \calS, h \in \calH$.
Note that the termination probabilities for non-adjustable options (primitive actions) are $1$. That is, $\beta_\optionset (s, h) = 1, \forall s \in \calS, \forall h \in \{k +1, \cdots, k + \cardA \}$. A primitive action's policy chooses the action deterministically. That is, $\pi_\optionset (a \mid s, h) = \bfI(a = a_h), \forall s \in \calS, h \in \{k +1, \cdots, k + \cardA \}, a \in \calA$. 


In order to chose actions, the agent could choose to execute its options. Such a way of behaving is called \emph{call-and-return}. Specifically, at the first time step of each episode, the agent observes a task $N_0$ randomly chosen from $\calN$ and a state $S_0$ randomly chosen from $d_0$, it then chooses an option index $H_0$ with probability 
\begin{align}
    \mu^{N_0}_{\Omega(S_0)}(H_0 \mid S_0) \doteq \frac{f^{N_0}(S_0, H_0)}{\sum_{h \in \Omega(S_0)} {f^{N_0}(S_0, h)}}, \label{eq: mu omega}
\end{align}
and takes action $A_0$ with probability $\pi(A_0 \mid S_0, H_0)$. We call $\mu^{N_0}_{\Omega(S_0)}$ a meta-policy because it chooses from option indices rather than primitive actions. Having observed the next state $S_1$ and reward $S_1$, the agent either keeps following $H_0$ with probability $1 - \beta(S_1, H_0)$ or terminates $H_0$ with probability $\beta(S_1, H_0)$.  The task remains unchanged through the episode $N_0 = N_1 = \ldots$ and the above process keeps going until the episode terminates. For simplicity, let $Z_{t+1}$ denote the termination signal of the episode. That is, $Z_{t+1} = 0$ if the current episode continues ($S_{t+1} \not \in \perp^{N_t}$) and $Z_{t+1} = 1$ if the the current episode is terminates ($S_{t+1} \in \perp^{N_t}$). 

Define $\bar q^n_{\optionset, \policyset}(s, h, a)$ to be the expected cumulative reward with a set of options $\optionset$ and a set of preferences $\policyset$ if the agent starts from state $s$, chooses option $o_h$, and action $a$, and behaves in the call-and-return fashion. Define option-value function $\bar q^n_{\optionset, \policyset}(s, h) \doteq \sum_a \pi_\optionset(a \mid s, h) \bar q^n_{\optionset, \policyset}(s, h, a)$, and value function $\bar v^n_{\optionset, \policyset}(s) \doteq \sum_h \mu^n_{\Omega(s)}(h \mid s) \bar q^n_{\optionset, \policyset}(s, h)$.
It can be seen that the best achievable value function when choosing and following options is the optimal value function. That is, $\max_{\optionset \in \optionsetspace, \policyset \in \policysetspace} \bar v^n_{\optionset, \policyset}(s) = \bar v_*^n(s), \forall n \in \calN, s \in \calS.$
This equation holds because any pair $(\optionset, \policyset)$ defines a non-stationary flat policy (policy over primitive actions). We know for any non-stationary flat policy $\pi$, $\bar v_\pi^n(s) \leq \bar v_*^n(s), \forall n, s$ (Puterman 1994). Also, by setting $\optionset$ and $\policyset$ appropriately, all policies that only choose from primitive actions can be obtained and therefore an optimal policy, thus $\max_{\optionset, \policyset} \bar v_{\optionset, \policyset}^n(s) = \bar v_*^n(s), \forall n \in \calN, s \in \statespace$.

\section{An Option Discovery Objective for Faster Planning}\label{sec: An Option Discovery Objective for Fast Planning}

In this section, we formally define our new objective and explain why it enables faster planning.

The speed of planning is proportional to the number of elementary operations in a sequential machine. To understand how the number of elementary operations used in planning is affected by a set of options, consider applying value iteration, which is a classic iterative planning algorithm, to a set of options. The complexity of the value iteration algorithm for computing a near-optimal policy is the number of operations per-iteration times the number of iterations required to achieve a near-optimal value function, from which a near-optimal policy can be derived. The algorithm requires a model as input. For each iteration, the algorithm queries the model, for each state-option pair, the probability distribution over next states. Therefore the per-iteration complexity is $\sum_{s \in \statespace} \Omega(s) \times \cardS$. This shows that the per-iteration complexity is proportional to the average number of options that can be initiated at each state. 

A direct way to estimate the number of required iterations to achieve a near-optimal value function is to apply value iteration directly. This would be just fine if the set of options are fixed and our only goal is to evaluate these options. However, if the set of options are continually changing, which is unavoidable for any algorithm discovering options, estimating the iterations in this way would have several drawbacks. First, this approach is computationally expensive because whenever an option changes, its model needs to be re-estimated and value iteration needs to be re-applied. Second, it is unclear how to improve the options to reduce the number of iterations. 


We consider estimating the other quantity, the expected number of options being executed by a good policy to reach the terminal state. This quantity can be estimated by a model-free learning algorithm so that neither model learning nor planning is involved. Furthermore, as we will show shortly, it is possible to derive an algorithm that adjusts options to reduce this quantity. 

This quantity is also closely related to the number planning iterations. To understand their relation, consider the MRP shown in \cref{fig: objective vs planning iterations}. Value iteration would need $3$ iterations to find the policy shown in the figure (with all estimated values pessimistically initialized). The expected number of options executed by the policy to reach the terminal state is $2.5$. The difference comes from the stochasticity in option transitions -- without stochasticity, it is clear that the number of iterations used by value iteration to find a policy is exactly the number of options executed by the policy to reach the terminal state.


\begin{wrapfigure}{r}{0.28\textwidth}
\vspace{-1cm}
\centering
\includegraphics[width=0.28\textwidth]{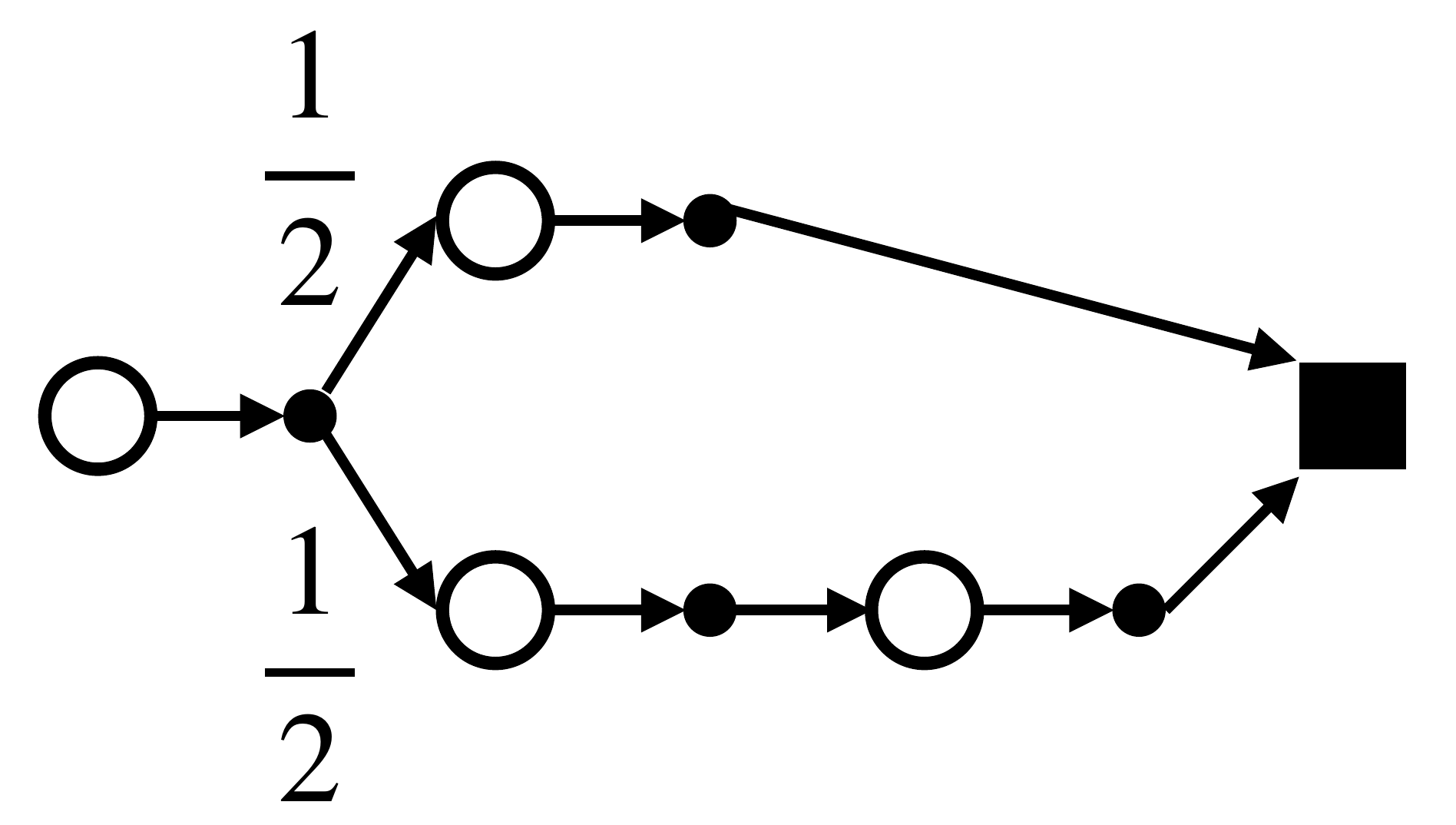}
\caption{Example illustrating of the relation between the number of planning iterations and the number of options being executed by a good policy. The empty circles are normal states. The solid square is a terminal state. The solid circles are options. For each state there are multiple options can be chosen (not shown in the figure) and the one shown in the figure is the best one. From the left state, taking the option shown in the figure moves to the two states in the middle with equal probabilities.
}
\label{fig: objective vs planning iterations}
\vspace{-0.5cm}
\end{wrapfigure}

We are now ready to present our new objective, which we call the \emph{fast planning option discovery objective}. Intuitively, the new objective prefers a set of options such that 1) there exists a policy choosing from the options and achieving a high expected return and 2) few options are considered at each decision point when following the policy. Note that the second point combines the need of having a small average number of options considered at each state and the need of having a small number of options chosen by a policy to reach the terminal state.


Define $\tilde q^n_{\optionset, \policyset}(s, h, a) \doteq - \bbE[|\Omega(S_0)| + |\Omega(S_{t_1})| + \ldots + |\Omega(S_{t_{N-1}})| | S_0 = s, H_0 = h, A_0 = a]$ as the negative expected cumulative number of options that the agent needs to consider before reaching the terminal state, if the agent starts from state $s$, chooses option $o_h$, and action $a$, and follows policy $\mu^n_{\Omega(s)}$. Here $t_i$ denotes the time step at which the $i-1$th option terminates and $t_N$ is the time step at which the terminal state is reached. Let $\tilde q^n_{\optionset, \policyset}(s, h) \doteq \sum_a \pi_\optionset(a \mid s, h) \tilde q^n_{\optionset, \policyset}(s, h, a)$, and $\tilde v^n_{\optionset, \policyset}(s) \doteq \sum_h \mu^n_{\Omega(s)}(h \mid s) \tilde q^i_{\optionset, \policyset}(s, h) - |\Omega(s)|$. 



We propose to maximize the following objective w.r.t. $\optionset, \policyset$:
\begin{align}\label{eq: unconstrained problem}
    J(\optionset, \policyset, c) \doteq \sum_n \sum_s d_0(s) \bar v_{\optionset, \policyset}^{n}(s) + c\sum_n \sum_s d_0(s) \tilde v_{\optionset, \policyset}^{n}(s),
\end{align}
where $c > 0$ is a problem parameter. The objective is a weighted sum of two terms. The first term is the expected values weighted over initial state distribution summed over all tasks. The second term is the total number of options that the agent chooses from within an episode when following meta-policies $\mu^n_{\Omega(s)},\ \forall n, s$, again starting from states sampled from $d_0$, summed over all tasks. The problem parameter $c$ specifies the agent's preferences of the two terms. 

\textbf{Remark:} 
If there is only one task and all options can be initiated at all states, our objective $J$ reduces to the objective proposed by Harb et al. (2018). However, with only one task, the best strategy is to learn an option that solves the entire task and to always choose that option so that there would be only one option executed by the policy and the second term of $J$ is minimized. On the contrary, if there are multiple tasks and the number of tasks is more than the number of options, the agent can not assign for each task an option, and learned options would better be an \emph{overlapped part} of the solutions to different tasks. The next example illustrates this point.

\begin{figure}
\begin{subfigure}{0.25\textwidth}
    \centering
    \includegraphics[width=\textwidth]{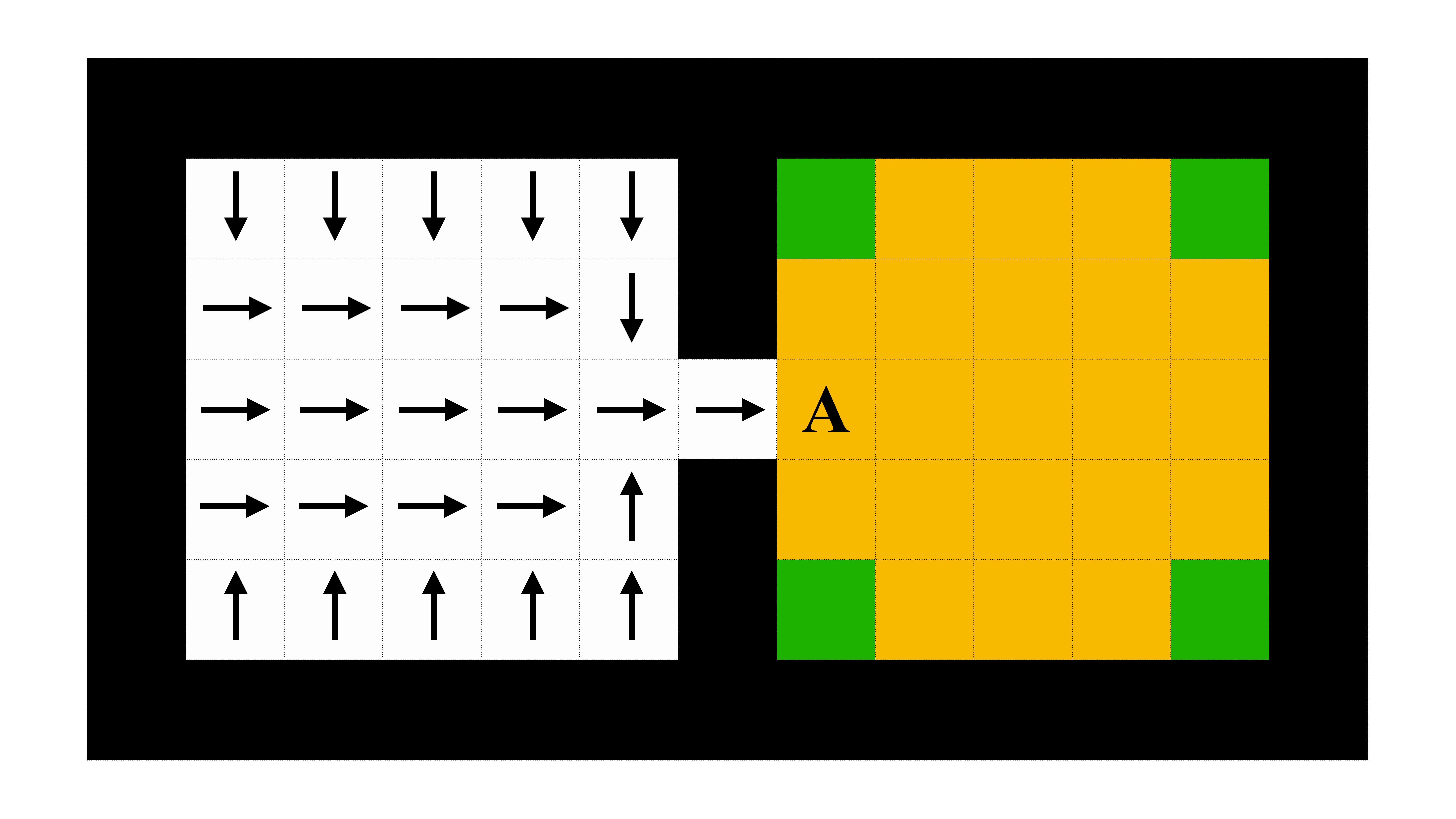}
    \caption{Two-room domain}
    \label{fig: metric illustration}
\end{subfigure}%
\begin{subfigure}{0.25\textwidth}
    \centering
    \includegraphics[width=\textwidth]{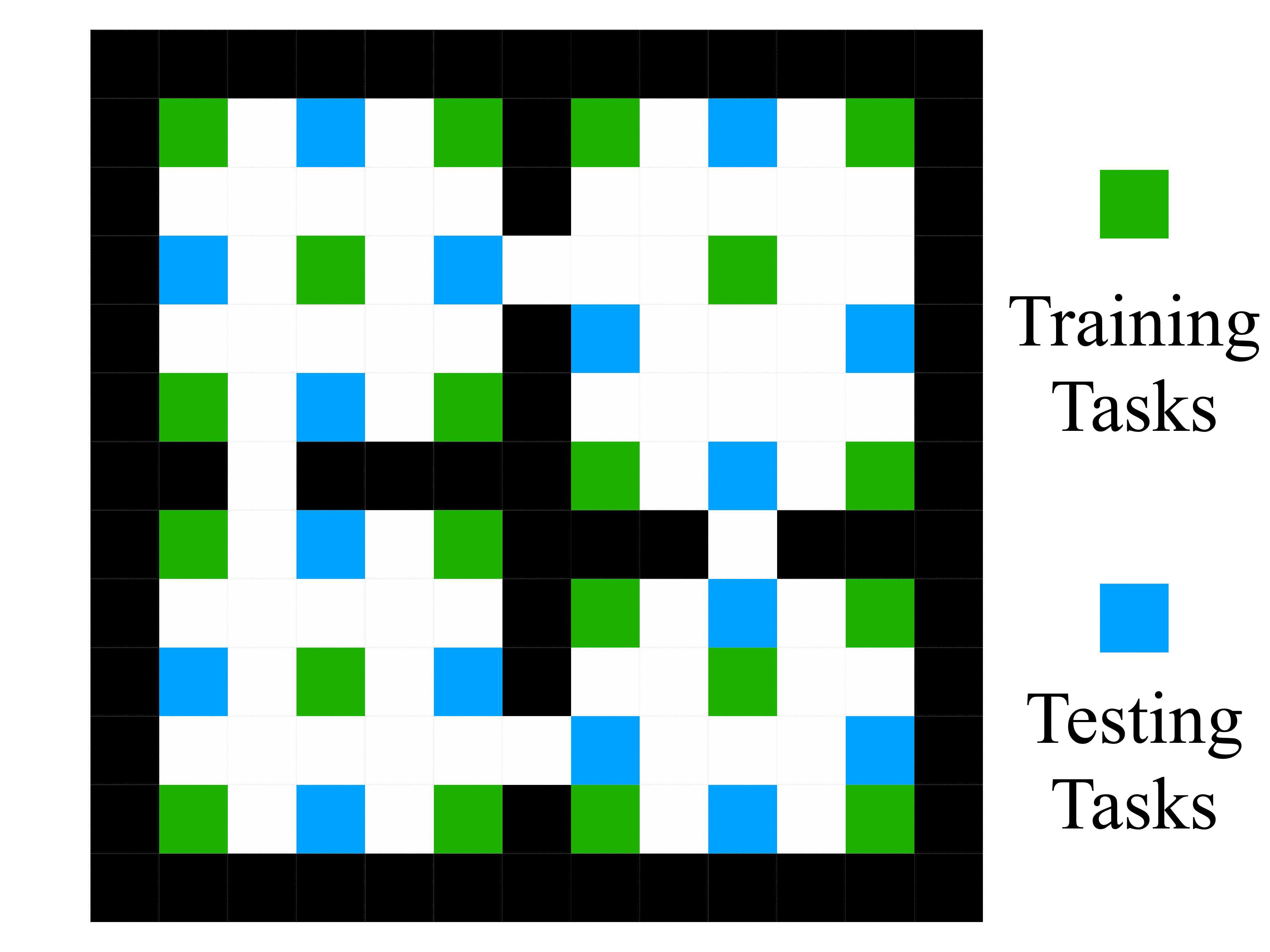}
    \caption{Four-room domain}
    \label{fig: four room domain}
\end{subfigure}%
\begin{subfigure}{0.25\textwidth}
    \centering
    \includegraphics[width=\textwidth]{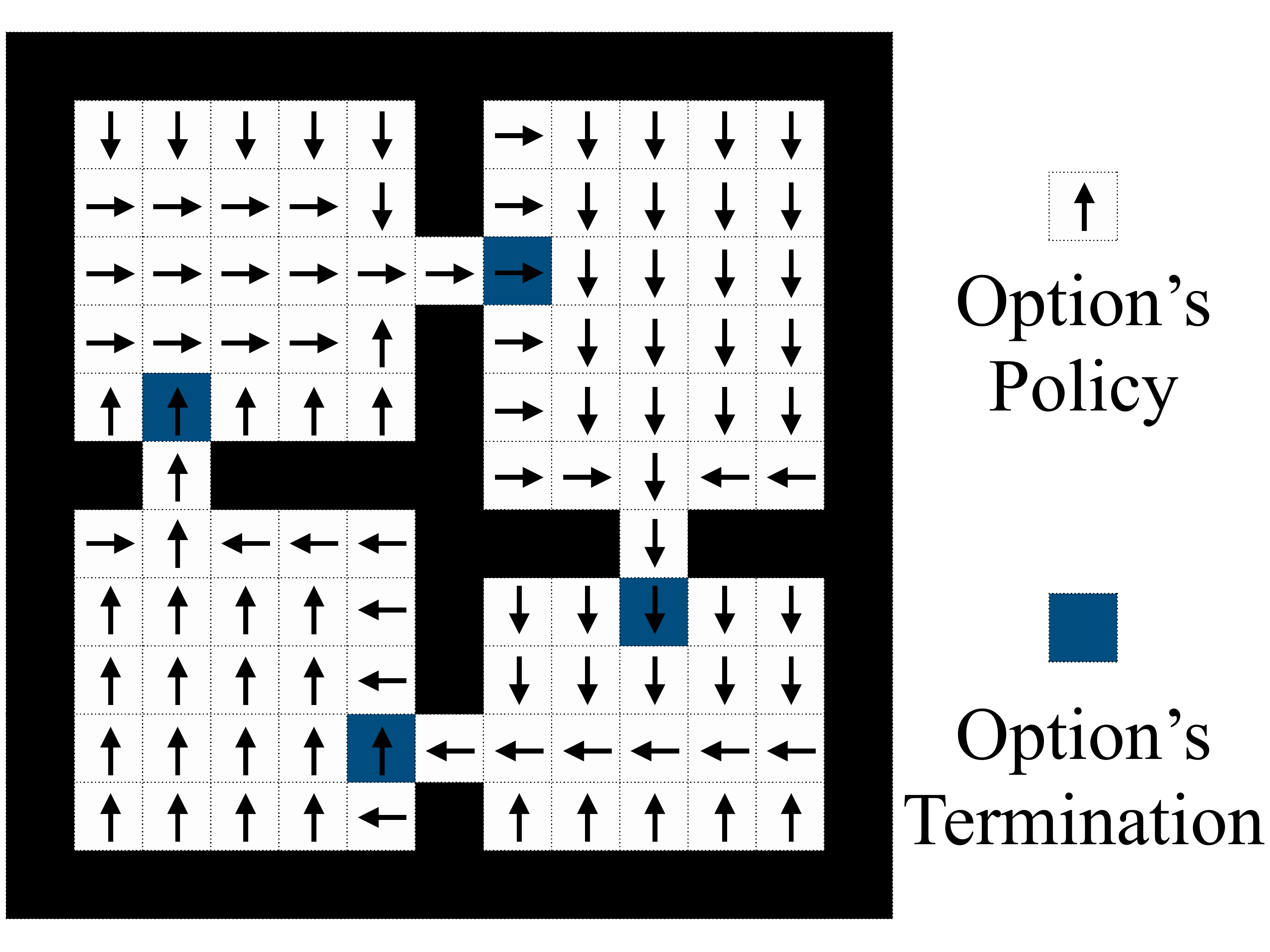}
    \caption{Hallway Option 1}
    \label{fig: Hallway Option 1}
\end{subfigure}%
\begin{subfigure}{0.25\textwidth}
    \centering
    \includegraphics[width=\textwidth]{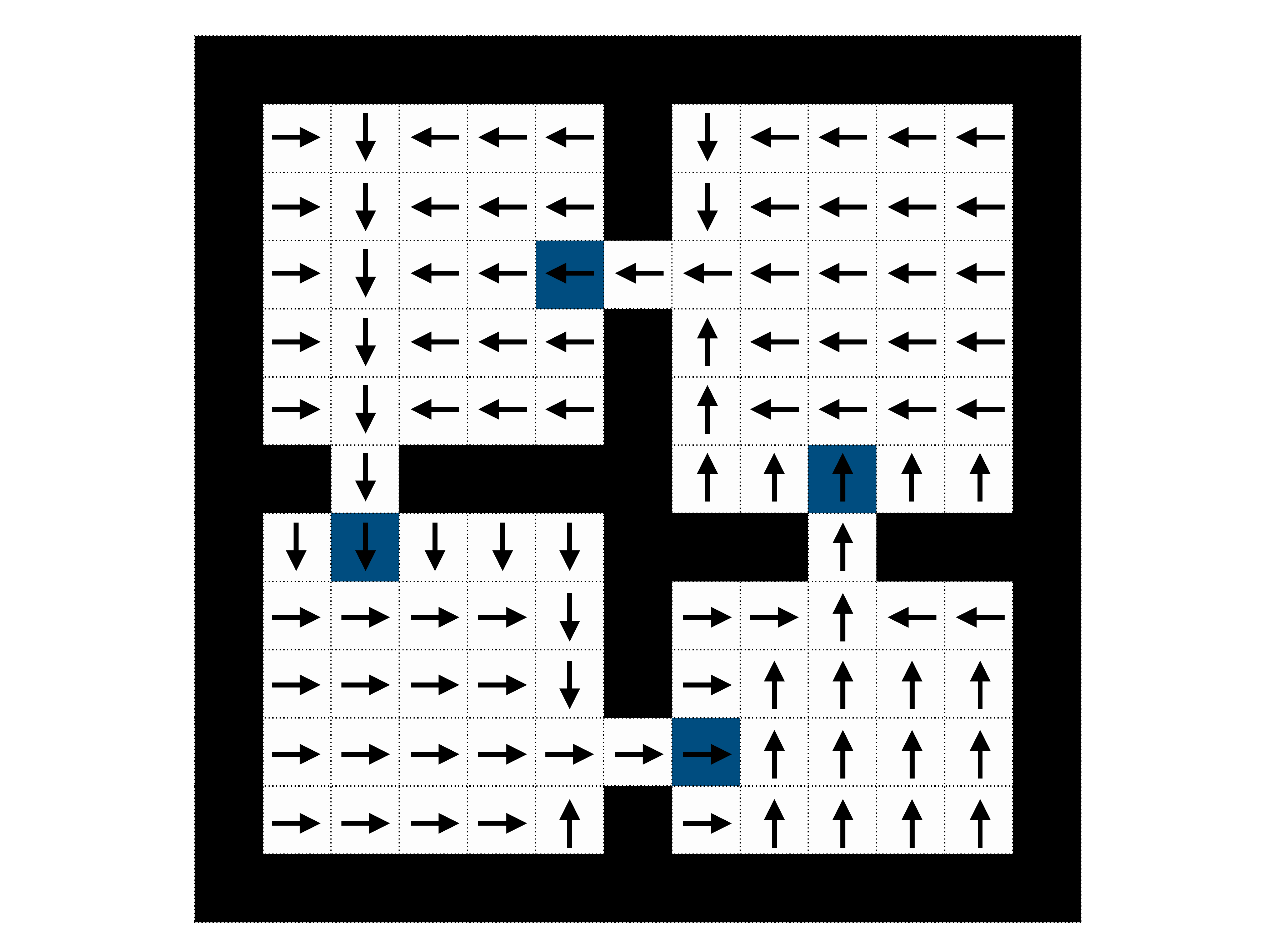}
    \caption{Hallway Option 2}
    \label{fig: Hallway Option 2}
\end{subfigure}%
\caption{(a) Example showing a good option under the our objective would be shared by solutions to multiple tasks. Each green cell marks a terminal state of a task. All rewards are $-1$. The agent is randomly initialized from the left room.
(b) Illustration of the four-room domain. The agent is randomly initialized over all empty (white $+$ green $+$ blue) cells. All rewards are $-1$ and the discount factor is $1$. (c) The policy and the termination probability of the first hallway option. This option's policy travels through the four rooms in the clock-wise direction. This option is terminated whenever the agent moves to a new room (marked by deep blue colored cells).
(d) The policy and the termination probability of the second hallway option, which travels in the counter-clock-wise direction.  
}
\end{figure}
\textbf{Example.}
Consider a two-room domain with four episodic tasks shown in \cref{fig: metric illustration}. Each episode starts from a state randomly picked from the left room. The agent has four primitive actions (\texttt{left}, \texttt{right}, \texttt{up}, \texttt{down}) and one adjustable option. Consider an option that can be initiated everywhere. For cells in the left room and the hallway cell, actions chosen by the option's policy are marked by arrows. For cells in the right room, the option's policy is uniformly random. The option terminates deterministically at yellow cells. Note that the option is shared by solutions of all four tasks. Also, note that if the option terminates at the hallway cell instead of cell $A$, solutions to all four tasks need to take a primitive action \texttt{right}, introducing an additional decision to make. Therefore this option is worse than the one plotted according to our objective. On the other hand, if the option does not terminate at cell $A$, no matter what action is assigned to this cell, the action is not optimal for some tasks. As long as $c$ is relatively smaller than the magnitude of the per-step reward. The hypothesized option induces a worse objective value compared with the one plotted in the figure.

\section{Discovering Options That Can Be Initiated Everywhere}\label{sec: The New Objective}

In this section, we consider the case where all options can be initiated at all states. Therefore set of options considered at each decision point is just the entire set of options (primitive actions plus adjustable options). In this case, the second term in the objective is proportional to the expected number of options being executed in each episode.
Recall that our objective reduces to the objective by Harb et al. (2018) if there is a single task and all options are considered at every state, therefore a multi-task extension of the single-agent tabular version of the Asynchronous Advantage Option-Critic (A2OC) Algorithm (Harb et al.\ 2018) can be applied to optimize the objective \footnote{The A2OC algorithm is a multi-agent algorithm because there are multiple agents sharing parameters simultaneously interacting with their own copies of the environment.}. We call this multi-task extension the \emph{Multi-task Advantage Option-Critic} (MAOC) algorithm. The description of the algorithm is provided in \cref{sec: FPOC}.

We now present an experiment to evaluate 1) whether the MAOC algorithm achieves a high objective value across a set of training tasks and 2) whether, with the set discovered options, planning to obtain near-optimal value functions for a set of testing tasks requires fewer elementary operations in planning. The experiment consists of a training phase and a testing phase. In the training phase, the agent observes a set of training tasks and learns options to solve these tasks. In the testing phase, the agent fixes the learned options, learns a model of these options, and plans with the learned model to solve a set of testing tasks, which are similar to the training tasks. The planning algorithm is the classic value iteration algorithm applied to both primitive actions and learned options. If the algorithm uses fewer operations to achieve near-optimal performance than the same algorithm using only primitive actions, we say that the learned options enable faster planning. In this experiment, the environment is a four-room grid world with $20$ training tasks corresponding to $20$ green cells, and $16$ testing tasks corresponding to $16$ blue cells (\cref{fig: four room domain}). At the beginning of each episode, the agent observes a task that is randomly sampled from the $20$ training tasks in the training phase (or the $16$ testing tasks in the testing phase). The initial state of the episode is uniformly randomly sampled from all empty (white $+$ green $+$ blue) cells. The green/blue cell corresponding to the current task is the only cell that terminates the episode when the agent reaches it. All other cells, including other green/blue cells, are just normal states. The agent has four actions, \texttt{up, down, left, right}, each of which deterministically takes the agent to the intended neighbor cell if the cell is empty. Otherwise, the agent stays at its current cell. The reward is $-1$ for each step, including the step transitioning to the terminal state, regardless of the action the agent takes. 
The experiment consists of $10$ runs, each of which consists of $10^8$ training steps. Every $10^6$ training steps, the agent is evaluated for $500$ episodes. No parameter updates are performed in evaluation episodes. At the beginning of each evaluation episode, the agent chooses an option that it finds to be the best, using a greedy policy w.r.t. its option-value estimate $Q$. The agent then executes the chosen option until it terminates. The agent then repeats the above process until the episode terminates. For each evaluation episode, we compute the compound return, which is the return minus the $c \times |\calH| \times$ number of decision points, where $c$ is chosen to be $0.2$. We use the \emph{average compound return} to denote the compound return averaged over $500$ evaluation episodes. The agent uses two adjustable options. 

We report a learning curve of the algorithm in \cref{fig: MAOC learning curve}. The parameter setting used is the one that results in the highest compound return averaged over the last $5$ times evaluation and $10$ runs. In the same figure, we also show in the dotted line the best compound return given only primitive actions, which terminate every time step. In addition, we show the best compound return given two hallway options (\cref{fig: Hallway Option 1}, \cref{fig: Hallway Option 2}) that move to and terminate at entrances of rooms and the primitive actions. 
Details of the experiment are presented in \cref{app: Details of Experiments}.
\begin{figure*}[h]
\begin{subfigure}{0.48\textwidth}
    \centering
    \includegraphics[width=\textwidth]{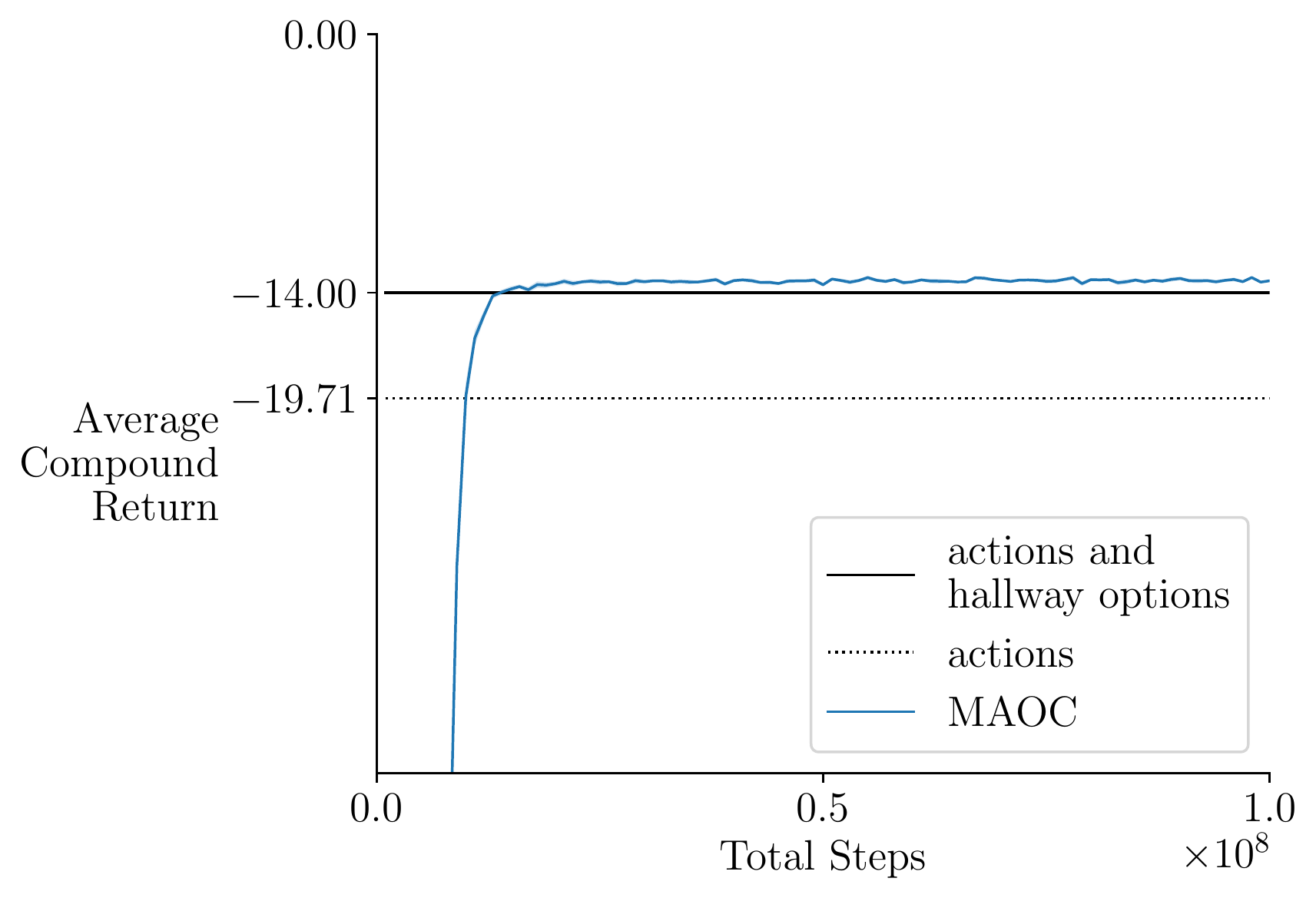}
    \caption{MAOC Learning Curve}
    \label{fig: MAOC learning curve}
\end{subfigure}%
\begin{subfigure}{0.52\textwidth}
\includegraphics[width=\textwidth]{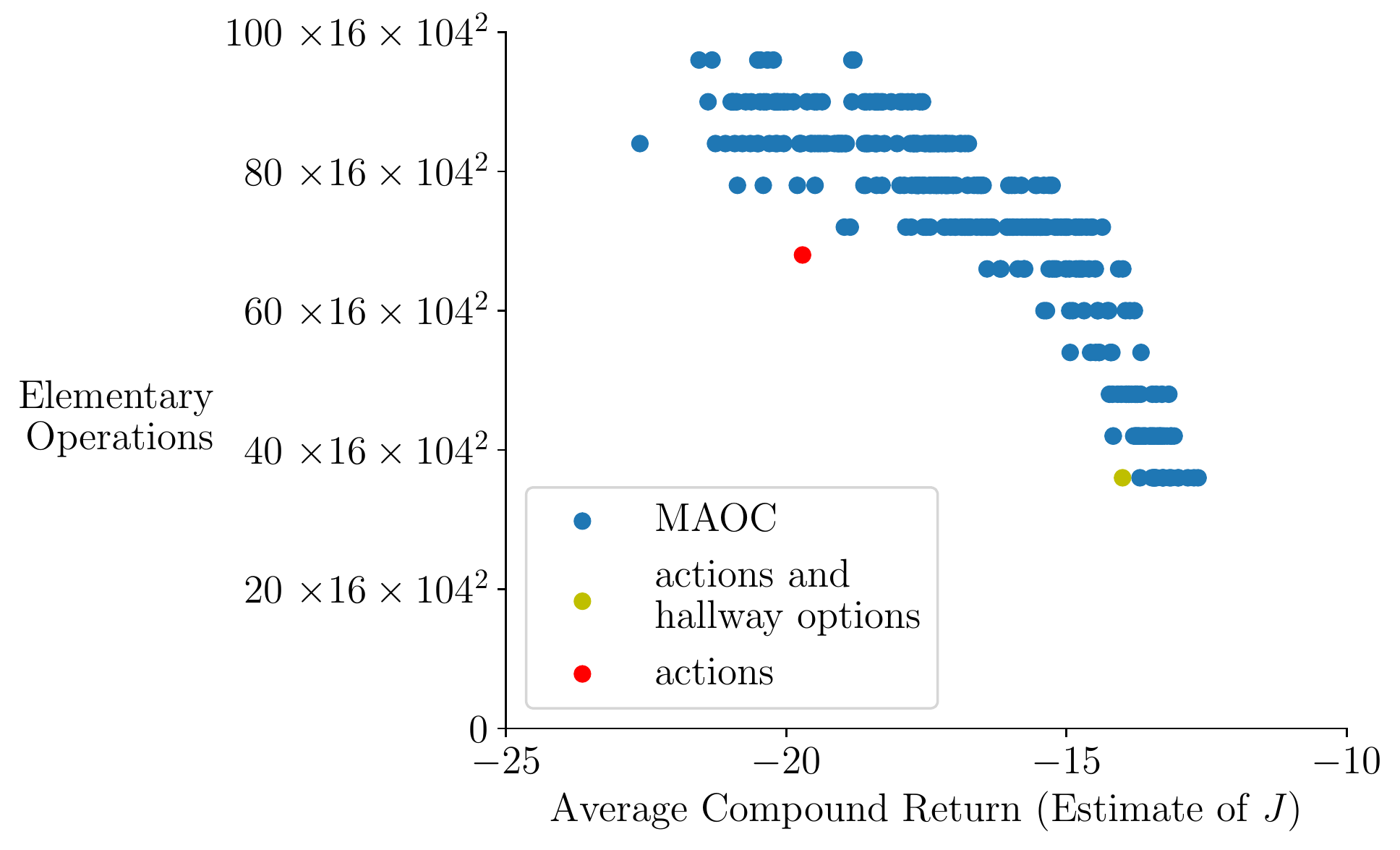}
\caption{Elementary Operations in Planning vs Estimate of $J$}
\label{fig: relation between compound return}
\end{subfigure}
\caption{
(a) Learning curve of the MAOC algorithm (with initiation sets containing all states) shows that the algorithm reliably achieves a compound return that is even slightly better than the return achieved with the two built-in hallway options. The $x$-axis is the number of training steps, the $y$-axis is the average compound return. The dotted line represents the best compound return given only four primitive actions, in which case options (primitive actions) terminate every time step. The solid line represents the best possible compound return given two built-in hallway options and four primitive actions. The shading area represents the standard error over $10$ runs. (b) Relation between the average compound return (estimate of the objective with $c = 0.2$) in a set of training tasks and the number of elementary operations used by option-value iteration to achieve near-optimal values in a set of testing tasks. The best discovered options match the human-designed options in terms of number of elementary operations used in planning.
}
\label{fig: experiment}
\end{figure*}

\Cref{fig: relation between compound return} helps understand the relation between the objective $J$, defined over the set of \emph{training} tasks (green colored cells in \cref{fig: four room domain}), and the number of planning operations used to achieve near-optimal value functions for the set of \emph{testing} tasks (blue colored cells in \cref{fig: four room domain}). Each blue dot is obtained by performing an experiment with a different parameter setting or a different seed. The $x$-axis is the estimate average compound return. The $y$-axis is the total number of operations used to obtain near-optimal value functions for all testing tasks. As a baseline, we plot a yellow dot, corresponding to the number of operations and the best average compound return when the set of options consists of two built-in hallway options (\cref{fig: experiment}) and primitive actions. We also plot in red the number of planning operations given only primitive actions. It can be seen that the average compound return and the number of operations are negatively correlated -- higher estimate of the objective value for training tasks is typically associated with fewer planning operations for testing tasks. In addition, it can be seen that the discovered options in the best runs achieve the same number of planning operations compared with the two built-in options. 

We also plot the learned options in \cref{fig: FPOC learned options}. The two discovered options are pretty similar to the human-designed options. First, in both cases, there is one option that moves in the clockwise direction and the other one moves in the opposite direction. Second, in both cases, the two options terminate at the entrances of rooms. There is one difference between the two sets of options. At the entrance cell of a room, our discovered options sometimes choose to not go directly to the next room, but visit the edges and corners of the room (e.g., see the first option's policy in the upper right room). This is actually because the most of the goals of training tasks are distributed in the corners of the room and this option can be used by these training tasks to reduce the cost term in $J$. On the other hand, for the two built-in options, the agent has to take primitive actions to reach goals in corners of a room once it enters that room. This is also the reason why the discovered options achieve a slightly higher objective value compared with the built-in options.

\begin{figure*}[h]
\begin{subfigure}{0.5\textwidth}
    \centering
    \includegraphics[width=\textwidth]{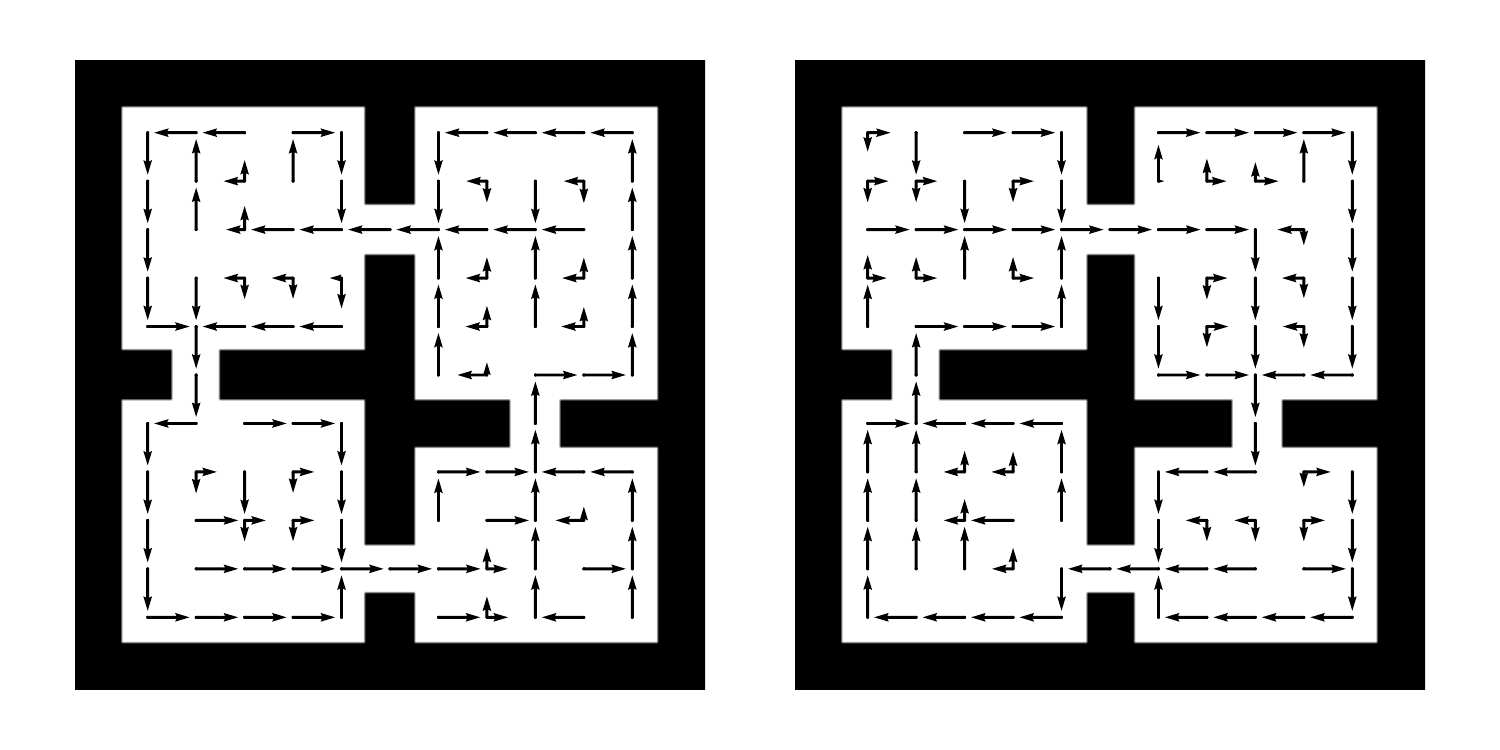}
    \caption{Learned Options' Policy}
\end{subfigure}%
\begin{subfigure}{0.5\textwidth}
    \centering
    \includegraphics[width=\textwidth]{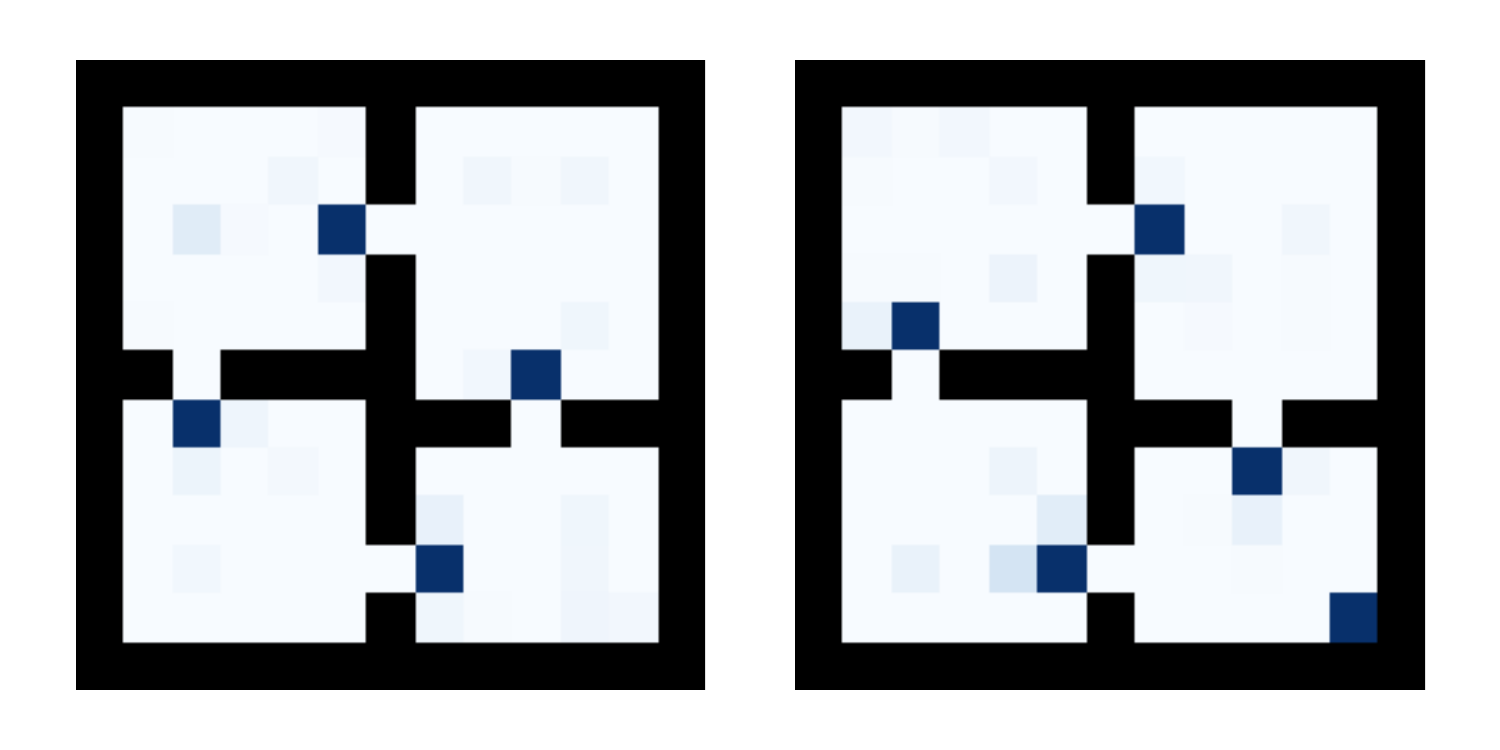}
    \caption{Learned Options' Termination Probabilities}
\end{subfigure}
\caption{
The learned option by the MAOC algorithm using the same parameter setting used to produce the learning curve in \cref{fig: MAOC learning curve}. For termination probabilities, darker blue means terminating with higher probability.
}
\label{fig: FPOC learned options}
\end{figure*}


\section{From Initiation Sets to Interest Functions}

In many of the follow-up papers (Machado et al. 2017a;b, Bacon et al. 2017, Harb et al., 2018, Harutyunyan., 2019) of the original options framework paper by Sutton et al.\ (1999), the initiation sets degenerate as they are assumed to be fixed and contain all states. A natural question to ask is why would we ever want to discover initiation sets. 

One possible justification is that an option's initiation set represents `affordance' of the option. That is, whether or not an option \emph{can be} initiated at a state. However, this justification appears questionable to us because it does not justify why options that can not be executed at certain states are better than options that can be initiated everywhere. The computational justification, introduced in \cref{sec: An Option Discovery Objective for Fast Planning}, appears to be more sound and natural to us. That is, initiation sets would better to be more selective so that fewer options are considered at each state to save computation in planning. We believe that this is the first time that the advantage of using initiation sets is justified from the computational perspective.

The question then becomes how to learn these sets from data. A possible solution is the generate-and-test approach -- one just randomly generates these sets and evaluates the corresponding objective value of $J$, and chooses the one that results in the highest value. However, it is clear that this approach is computationally expensive. We look for a more efficient algorithm searching for initiation sets. To this end, we generalize the definition of options by allowing initiation sets to be stochastic. With this modification, it is possible to apply stochastic gradient descent to reduce the number of options that can be initiated at each state.


We now provide the more general definition of options. An option $o$ is a tuple $(i_o, \pi_o, \beta_o)$, where $i_o \in \Gamma$ is the option's interest function, $\Gamma: \{f \mid f: \calS \to [0, 1]\}$, $\pi_o \in \Pi$ is the option' policy, and $\beta_o \in \Gamma$ is the option's termination function. 
At each state $s$, the options that can be initiated are sampled according to their interest functions: $\Pr(o \in \Omega(s)) = i_o(s)$. In addition, assume that primitive actions can be initiated at all state $i_a(s) = 1 \ \forall s \in \statespace, \forall a \in \actionspace$. Finally, just as what we defined in \cref{sec: problem setting}, given a set of options $\optionset$, define a function $i_\optionset: \calS \times \calH \to [0, 1]$ with $i_\optionset(s, h) \doteq i_{o_h}(s)$, $\forall s \in \calS, h \in \calH$. The definition of $\bar v$ also needs to be generalized accordingly: $\bar v^n_{\optionset, \policyset}(s) \doteq \sum_{\Omega \in \calP(\calH)} \Pr(\Omega \mid s) \sum_h \mu_{\Omega}^n(h \mid s) \bar q^n_{\optionset, \policyset} (s, h)$. Here $\Pr(\Omega \mid s) \doteq \prod_{h \in \Omega} i(s, h) \prod_{\bar h \not \in \Omega} i(s, \bar h)$. Similarly we generalize the definition of $\tilde v$: $\tilde v^n_{\optionset, \policyset}(s) \doteq \sum_{\Omega \in \calP(\calH)} \Pr(\Omega \mid s) \sum_h \mu_{\Omega}^n(h \mid s) \tilde q^n_{\optionset, \policyset} (s, h) - |\Omega|$.

The call-and-return agent-environment interaction is also modified given the new definition of options. Suppose that the agent is solving task $N$ by executing a set of options, with the set $\calH$ denoting the indices of options, and suppose that at state $S$, the previous option terminates and a new option needs to be chosen, the agent then samples a set $\Omega \in \calP(\calH)$, and chooses an option indexed by $H$ with probability $\mu_{\Omega}^N(H \mid S)$. Note that $\Omega$ can not be empty because primitive actions' interests are $1$ and are thus in $\Omega$. 

The idea of using interest functions in options has also been explored by Khetarpal et al.\ (2020). Note that their interest functions are not necessarily probabilities and are different as ours. The most important difference is that the interest functions in our work are used to generate initiation sets while the initiation sets are no longer considered in their work. While they have observed that smaller initiation sets reduce the computational cost, they replaced initiation sets with interest functions, and still had all options to be initiated at each state.


\section{Discovering Interest Functions}\label{sec: Discovering Interest Functions}

In order to maximize the objective $J$ by updating the options and the meta-preference functions, it is common to apply stochastic gradient ascent. Suppose that adjustable options' interests, policies, termination probabilities, as well as the meta-preference functions are jointly parameterized by $\vtheta$. In this case, both $\optionset$ and $\policyset$ are parameterized by $\vtheta$. With this in mind, we replace $\optionset$ and $\policyset$ by $\vtheta$ and present the gradient w.r.t. $J$ in the following theorem. The proof of the theorem is presented in \cref{sec: Gradient of the Objective}.

\begin{theorem}\label{thm: gradient} 
Let $v_\vtheta^{n}(s) \doteq \bar v_\vtheta^{n}(s) + c \tilde v_\vtheta^{n}(s), q_\vtheta^{n}(s, h) \doteq \bar q_\vtheta^{n}(s, h) + c \tilde q_\vtheta^{n}(s, h)$ and $q_\vtheta^{n}(s, h, a) \doteq \bar q_\vtheta^{n}(s, h, a) + c \tilde q_\vtheta^{n}(s, h, a)$. Suppose that $J(\vtheta)$ is differentiable w.r.t. $\vtheta$,
\begin{align*}
    & \nabla J(\vtheta) = \sum_n \sum_{s} d_0 (s) m_\vtheta^n (s) + \sum_{n, s, h} d_{\gamma, \vtheta}^n (s, h) \Bigg ( \sum_{a} \nabla \pi_\vtheta (a \mid s, h) q_\vtheta^n (s, h, a) \\
    & + \gamma \sum_{a, s', r} \pi_\vtheta (a \mid s, h) p(s', r \mid s, a) \Big( - \nabla \beta_\vtheta (s', h)  (q_\vtheta^n (s', h) - v_\vtheta^n (s')) + \beta_\vtheta (s', h) m_\vtheta^n(s') \Big ) \Bigg),
\end{align*}
where $d^n_{\gamma, \vtheta} (s, h) \doteq \bbE \left [\sum_{t=0}^T \gamma^t  \bfI(S_t = s, H_t = h) \mid S_0 \sim d_0, \vtheta, n \right]$
is the discounted number of time steps on average option $o_h$ is used at state $s$ in a single episode, and
\begin{align*}
    & m_\vtheta^n (s) \doteq \sum_{\Omega \in \calP(\calH)} \Pr(\Omega \mid s) \sum_{h \in \Omega} \mu_{\Omega, \vtheta}^n(h \mid s) \Bigg ( \sum_{\bar h \in \Omega} \nabla \log i_\vtheta (s, \bar h) + \sum_{\bar h \not \in \Omega} \nabla \log (1 - i_\vtheta(s, \bar h)) \\
    & + \nabla \log \mu^n_{\Omega, \vtheta} (h \mid s)  \Bigg) q^n_\vtheta(s, h) - c \sum_{h \in \calH} \nabla i_\vtheta (s, h).
\end{align*}
\end{theorem}
The policy update term $\sum_{a} \nabla \pi_\vtheta (a \mid s, h) q_\vtheta^n (s, h, a)$, termination probability update term $- \nabla \beta_\vtheta (s', h)  (q_\vtheta^n (s', h) - v_\vtheta^n (s'))$, and meta-policy update term $\sum_{h \in \Omega} \mu_{\Omega, \vtheta}^n(h \mid s) \nabla \log \mu^n_{\Omega, \vtheta} (h \mid s) q^n_\vtheta(s, h)$ also appear in the gradient of the option-critic objective (Bacon et al.\ 2017) and the objective by Harb et al.\ (2017).

The update to the interest function
\begin{align*}
    & \sum_{\Omega \in \calP(\calH)} \Pr(\Omega \mid s) \sum_{h \in \Omega} \mu_{\Omega, \vtheta}^n(h \mid s) \Bigg ( \sum_{\bar h \in \Omega} \nabla \log i_\vtheta (s, \bar h) + \sum_{\bar h \not \in \Omega} \nabla \log (1 - i_\vtheta(s, \bar h)) \Bigg) q^n_\vtheta(s, h) \\
    & - c \sum_{h \in \calH} \nabla i_\vtheta (s, h),
\end{align*}
is new. It can be understood in the following way. The second term $- c \sum_{h \in \calH} \nabla i_\vtheta (s, h)$ term encourages all interests to be smaller. Now for the first term, if $\sum_{h \in \Omega} \mu_{\Omega, \vtheta}^n(h \mid s) q^n_\vtheta(s, h)$ is high, it means that the set $\Omega$ contains good options, then when one updates the interest function using the gradient of $J$, the interests of all the options in $\Omega$ will be increased by a large amount and the interests of all the options not in $\Omega$ will be decreased by a large amount.

Our full algorithm, which we call Fast Planning Option-Critic (FPOC), is a tabular algorithm that searches for a pair $(\optionset, \policyset)$ that maximizes $J$ using a sample of the gradient given in \cref{thm: gradient}. Due to the space limit, we present the full algorithm in \cref{sec: FPOC}.

Applying FPOC to the four-room domain the same way as we did in \cref{sec: The New Objective}, we obtain \cref{fig: Adaptive interests}. This time in order to show the effect of learning the interest function, we increase the adjustable options to $8$, because with only $2$ adjustable options, it seems that both of them are quite useful for many of the states in order to achieve a high return and their interests would better to be high. Readers can refer to \cref{fig: FPOC study planning iterations vs estimate of J} for figures with $2$ and $4$ adjustable options. As a comparison, we show the same figure but for MAOC in \cref{fig: All interests equal to one.}. Note that for MAOC, many of the runs result in options that require more than $100 \times 16 \times 104^2$ operations and are not shown in the figure. It is clear that learning the interest function significantly reduces the elementary operations. Visualization of the discovered options as well as other empirical results can be found in \cref{app: Details of Experiments}.

\begin{figure*}[h]
\begin{subfigure}{0.5\textwidth}
\centering
\includegraphics[width=\textwidth]{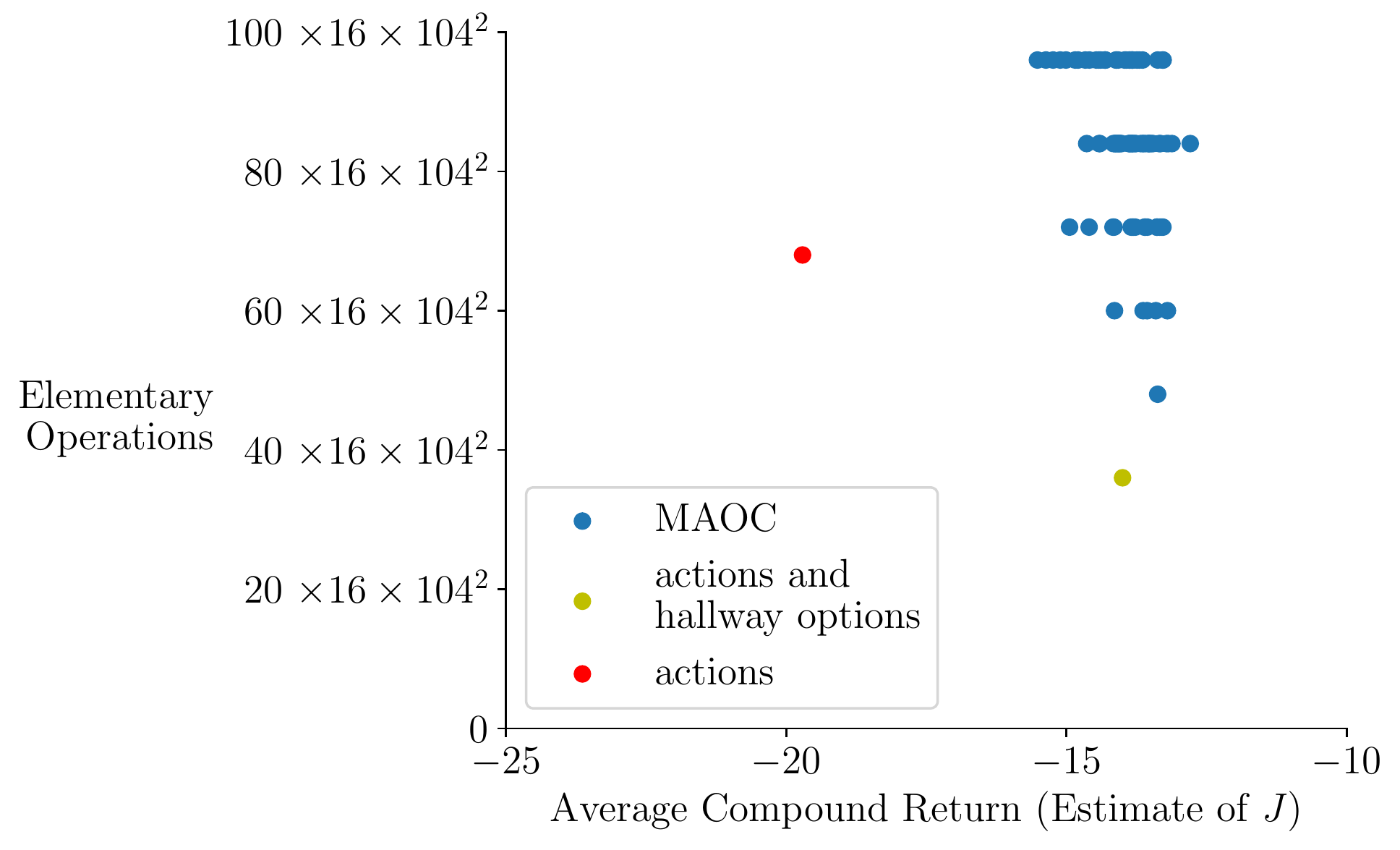}
\caption{MAOC}
\label{fig: All interests equal to one.}
\end{subfigure}%
\begin{subfigure}{0.5\textwidth}
\centering
\includegraphics[width=\textwidth]{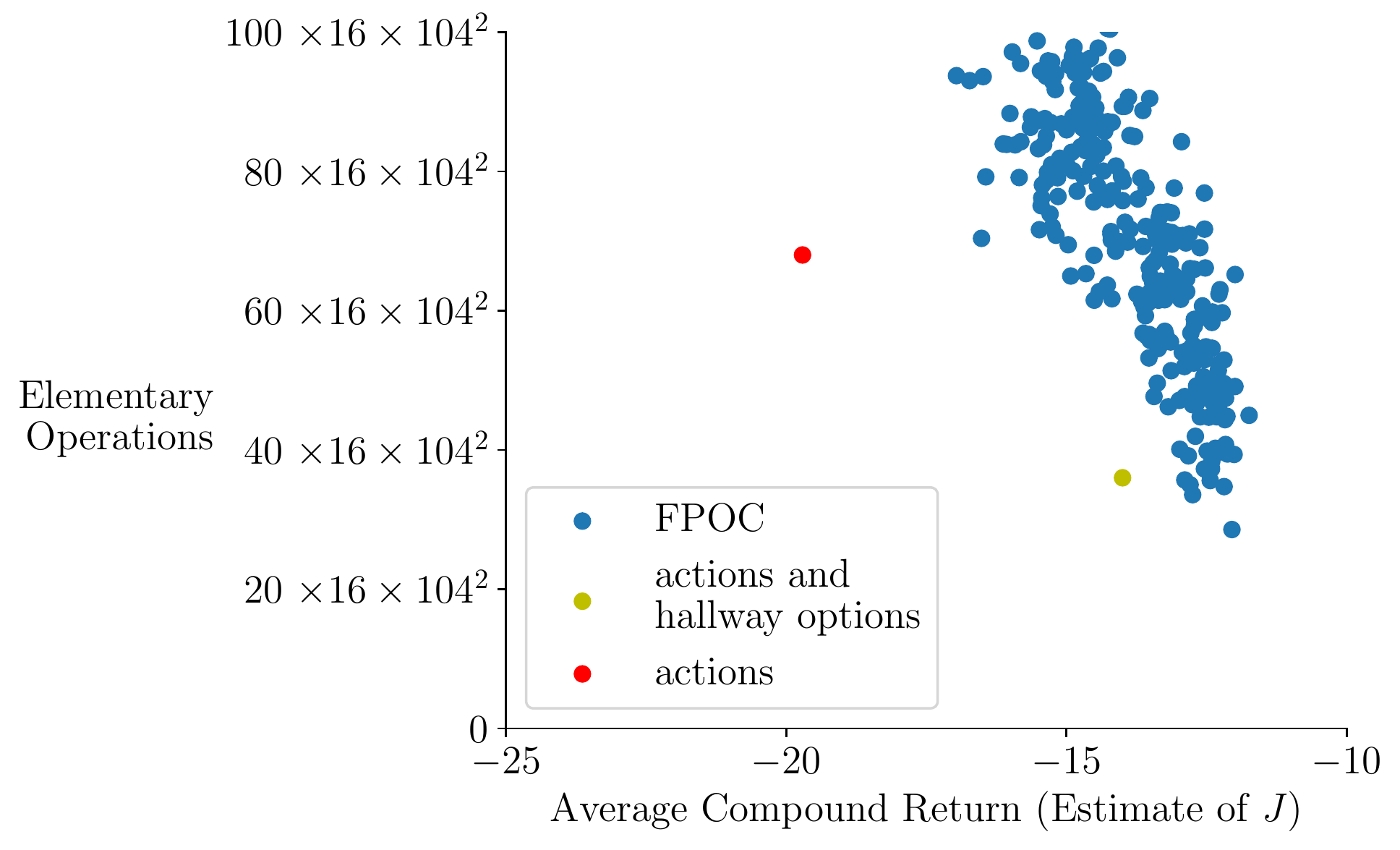}
\caption{FPOC}
\label{fig: Adaptive interests}
\end{subfigure}
\caption{
A comparison between the MAOC algorithm and the FPOC algorithm with $8$ adjustable options. It can be seen that learning to reduce number of options being considered at each state significantly reduces the number of operations used by option-value iteration.
}
\label{fig: experiment}
\end{figure*}

\section{Discussion and Future Work}
We have introduced a new objective for option discovery that highlights the computational advantage of planning with options when there are a set of given tasks and a given number of options. To the best of our knowledge, this is the first objective that explicitly advocates discovering general options that achieve fast planning.
Optimizing our objective maximizes the performance of solutions to the given tasks while minimizing the number of options used to compose the solutions as well as the number of options considered at each decision point. 
Empirically, we showed that higher objective values for a set of training tasks are associated with fewer elementary operations to achieve near-optimal value functions for a set of testing tasks, which are similar to the training ones. 
To optimize the proposed objective, a key step is to generalize the classic definition of options so that initiation sets are no longer fixed, but are sampled following the options' interest functions. We proposed a new algorithm that optimizes the objective by following the sample gradient of the objective.  
In the four-room domain, we show that the proposed algorithm, with or without learning to adapt the interest function, achieves a high objective value and discovers options that are comparable with two human-designed options in terms of the number of planning operations used, when applying option-value iteration with the discovered options. 
We believe that our objective and algorithm provide a new perspective for option discovery and the more general temporal abstraction problem. 

There are several ways in which our work is limited. The most important way is that our FPOC algorithm only treats the tabular setting. However, we do not see technical difficulty of extending this algorithm to the function approximation (FA) setting and would like to leave the analysis of the FA algorithm as a future work.
The second way in which our work is limited in that we assume a given set of tasks. We believe that these tasks are sub-tasks that the agent finds interesting/useful to solve in order to achieve better performance in the agent's main task and should be discovered by the agent itself. Discovering sub-tasks or sub-problems is a fundamental open research problem. Several typical recent approaches to the problem include Sutton et al. (2022), Veeriah et al. (2021), Nachum et al. (2018), and Vezhnevets et al. (2017). The third way in which our work is limited is that the number of options is assumed to be fixed and specified by a human. 

One interesting future work is to discover options that terminate at fewer states, would further reduces the computation complexity of planning. This idea has been explored by Harutyunyan et al.\ (2019), but were inspired from a different perspective. They argued that an option that terminates at fewer possible states is preferred because it has a shorter encoding and thus takes less amount of memory to maintain. It would be interesting incorporate their idea in our objective properly so that the resulting objective prefers fewer number of possible terminal states. 


\section*{Acknowledgements}
The authors were supported by DeepMind, Amii, NSERC, and CIFAR. The authors wish to thank Martha Steenstrup, Arsalan Sharifnassab, Abhishek Naik, Zaheer Abbas, and Chenjun Xiao for their valuable feedback during various stages of the work. Computing resources were provided by Compute Canada.

\newpage
\section*{References}

Bacon, P. L. (2013). \emph{On the bottleneck concept for options discovery.} McGill University.

Bacon, P. L. (2018). \emph{Temporal Representation Learning.} McGill University.

Bacon, P. L., Harb, J., Precup, D. (2017). The Option-Critic Architecture. \emph{AAAI Conference on Artificial Intelligence}.

Bertsekas, D. P., Tsitsiklis, J. N. (1996). \emph{Neuro-dynamic programming}. Athena Scientific.




Brunskill, E., Li, L. (2014). PAC-inspired Option Discovery in Lifelong Reinforcement Learning. \emph{International Conference on Machine Learning}.




Eysenbach, B., Gupta, A., Ibarz, J., Levine, S. (2018). Diversity is All You Need: Learning Skills without a Reward Function. \emph{International Conference on Learning Representations}.

Harb, J., Bacon, P. L., Klissarov, M.,  Precup, D. (2018). When waiting is not an option: Learning options with a deliberation cost. \emph{AAAI Conference on Artificial Intelligence}.

Harutyunyan, A., Dabney, W., Borsa, D., Heess, N., Munos, R.,  Precup, D. (2019). The termination critic.\emph{ International Conference on Artificial Intelligence and Statistics}.

Jinnai, Y., Abel, D., Hershkowitz, D., Littman, M.,  Konidaris, G. (2019). Finding options that minimize planning time. \emph{International Conference on Machine Learning}.

Jinnai, Y., Park, J. W., Abel, D.,  Konidaris, G. (2019). Discovering options for exploration by minimizing cover time. \emph{International Conference on Machine Learning}.

Khetarpal, K., Klissarov, M., Chevalier-Boisvert, M., Bacon, P. L.,  Precup, D. (2020). Options of interest: Temporal abstraction with interest functions. \emph{AAAI Conference on Artificial Intelligence}.


Konidaris, G., Barto, A. (2009). Skill discovery in continuous reinforcement learning domains using skill chaining. \emph{Neural Information Processing Systems}.



Machado, M. C., Bellemare, M. G., Bowling, M. (2017). A Laplacian Framework for Option Discovery in Reinforcement Learning. \emph{International Conference on Machine Learning}.

Machado, M. C., Rosenbaum, C., Guo, X., Liu, M., Tesauro, G.,  Campbell, M. (2017). Eigenoption discovery through the deep successor representation. \emph{International Conference on Learning Representations}.

McGovern, A., Barto, A. G. (2001). Automatic Discovery of Subgoals in Reinforcement Learning using Diverse Density. \emph{International Conference on Machine Learning}.




Menache, I., Mannor, S., \& Shimkin, N. (2002). Q-Cut—Dynamic Discovery of Sub-goals in Reinforcement Learning. \emph{European Conference on Machine Learning}.

Nachum, O., Gu, S. S., Lee, H.,  Levine, S. (2018). Data-efficient hierarchical reinforcement learning. \emph{Neural Information Processing Systems}.


Puterman, M. L. (1994). \emph{Markov Decision Processes: Discrete Stochastic Dynamic Programming.} John Wiley \& Sons.

Şimşek, Ö., \& Barto, A. G. (2004). Using Relative Novelty to Identify Useful Temporal Abstractions in Reinforcement Learning. \emph{International Conference on Machine Learning}.

Singh, S., Barto, A. G., Chentanez, N. (2004). Intrinsically Motivated Reinforcement Learning. \emph{Neural Information Processing Systems.}

Solway, A., Diuk, C., Córdova, N., Yee, D., Barto, A. G., Niv, Y.,  Botvinick, M. M. (2014). Optimal behavioral hierarchy. \emph{PLoS Computational Biology}.

Sutton, R. S., Precup, D., Singh, S. (1999). Between MDPs and Semi-MDPs: A Framework for Temporal Abstraction in Reinforcement Learning. \emph{Artificial Intelligence}.


Sutton, R. S.,  Barto, A. G. (2018). \emph{Reinforcement Learning: An Introduction.} MIT Press.


van Dijk, S. G.,  Polani, D. (2011). Grounding Subgoals in Information Transitions. \emph{IEEE Symposium on Adaptive Dynamic Programming and Reinforcement Learning}.


Veeriah, V., Zahavy, T., Hessel, M., Xu, Z., Oh, J., Kemaev, I., van Hasselt, H., Silver, D., Singh S. (2021). Discovery of Options via Meta-Learned Subgoals. \emph{Neural Information Processing Systems}.

Vezhnevets, A. S., Osindero, S., Schaul, T., Heess, N., Jaderberg, M., Silver, D.,  Kavukcuoglu, K. (2017). Feudal networks for hierarchical reinforcement learning. \emph{International Conference on Machine Learning}.


\newpage
\appendix
\section{The Fast Planning Option-Critic Algorithm}\label{sec: FPOC}

First, note that MAOC is just FPOC but with all interest functions being fixed to $1$ (and no updates are performed to update these functions) and we will only describe FPOC in this section.

\textbf{Remark:} Differences between MAOC and A2OC are 1) MAOC is an algorithm designed for the multi-task setting introduced in \cref{sec: problem setting} while A2OC is designed for the single-task setting, 2) MAOC is designed for the single-agent setting while A2OC has multiple agents simultaneously interacting with its own copy of the environment and sharing parameters, 3) FPOC updates for all options while A2OC only updates the value of the option being executed, and 4) MAOC is a one-step algorithm while A2OC is an $n$-step algorithm.

We are now ready to present the FPOC algorithm. FPOC maintains an $\abs{\calN} \times \abs{\calS} \times \abs{\calH}$ table of option-value estimate $Q$ and a $\abs{\calS} \times \abs{\calH^{adj}}$ table of termination preferences $W^\beta$ used compute termination probabilities of adjustable options. Given $W^\beta$ and an adjustable option $h \in \calH^{adj}$, the probability of terminating option $o_h$ at $s$ is obtained by applying the sigmoid function to the preference: $\beta_\optionset(s, h) \doteq \text{sigmoid}(W^\beta(s, h))$. The algorithm also maintains a $\abs{\calS} \times \abs{\calH^{adj}} \times \abs{\calA}$ table of policy preferences $W^\pi$ used to compute policies of adjustable options. Given $W^\pi$, a state $s$, and an adjustable option $h \in \calH^{adj}$, the probabilities of taking different actions of option $o_h$ at state $s$ are obtained by applying the softmax function to the preferences of the corresponding actions. $\pi_\optionset(a \mid s, h) \doteq e^{W^\pi(s, h, a)} / \sum_{\bar a} e^{W^\pi(s, h, \bar a)}$. Finally, the algorithm maintains a $\abs{\calS} \times \abs{\calH^{adj}}$ table of interest preferences used compute interests of adjustable options. Given $W^i$, an adjustable option $h \in \calH^{adj}$, the probability of including option $o_h$ in the initiation set of $s$ is obtained by applying the sigmoid function to the preference: $i_\optionset(s, h) \doteq \text{sigmoid}(W^i(s, h))$. The set of options $\optionset$ is parameterized by $W^\pi$ and $W^\beta$. With this in mind, we omit the superscript $\optionset$ in $\beta_\optionset$, $\pi_\optionset$, and $i_\optionset$ for simplicity. 

FPOC does not maintain the set of meta-preference functions $\policyset$ to compute the meta-policy. Instead, the meta-policy is an $\epsilon$-greedy policy w.r.t. the option-value estimates for all options that can be initiated. Formally,

\begin{align*}
    \mu_{\Omega}^n(h \mid s) = 
    \begin{cases}
    & 0 \quad h \not \in \Omega\\
    & (1 - \epsilon) / |\argmax_{x \in \Omega} Q^n(s, x)| + \epsilon / |\Omega| \quad h \in \argmax_{x \in \Omega} Q^n(s, x)\\
    & \epsilon / |\Omega| \quad \text{otherwise}
    \end{cases}.
\end{align*}

The action taken at time step $t$, $A_t$, is chosen according to the $o_{H_t}$'s policy $\pi(\cdot \mid S_t, H_t)$. 

Given that $\policyset$ is not maintained and all other estimates are maintained using tables, we can obtain an estimate of $m_\vtheta^n$ in the following way. We first sample a set of options $\Omega$ according to the interest functions. An unbiased estimate of $m^n_{\vtheta}(s, x)$ can be obtained:
\begin{align}\label{eq: naive estimate of m}
    & \hat m^n_{\Omega, \vtheta}(s, x) \doteq \sum_{h \in \Omega} \mu_{\Omega, \vtheta}^{n} (h \mid s) \Bigg ( \sum_{\bar h \in \Omega} \nabla \log i_\vtheta(s, \bar h) + \sum_{\bar h \not \in \Omega} \nabla \log (1 - i_\vtheta(s, \bar h)) \nonumber\\
    & \nabla \log \mu^n_{\Omega, \vtheta} (h \mid s) \Bigg) q^n_\vtheta (s, h) - c \sum_{h \in \calH} \nabla i_\vtheta(s, h) \nonumber\\
    & =  \sum_{h \in \Omega} \mu_{\Omega}^{n} (h \mid s) \Bigg ( \sum_{\bar h \in \Omega} \nabla_{W^i(s, x)} \log i(s, \bar h) + \sum_{\bar h \not \in \Omega} \nabla_{W^i(s, x)} \log (1 - i(s, \bar h)) \Bigg) Q^n (s, h) \nonumber\\\
    & - c \sum_{h \in \calH} \nabla_{W^i(s, x)} i(s, h)\nonumber\\
    & = \sum_{h \in \Omega} \mu_{\Omega}^{n} (h \mid s) \Bigg ( \bfI(x \in \Omega) (1 -  i(s, x)) + \bfI(x \not \in \Omega) ( -  i(s, x)) \Bigg) Q^n (s, h) - c i(s, x) (1 - i(s, x))\nonumber\\
    & = \Bigg ( \bfI(x \in \Omega) - i(s, x) \Bigg)  \sum_{h \in \Omega} \mu^n_{\Omega} (h \mid s) Q^n (s, h) - c i(s, x) (1 - i(s, x)),
\end{align}
where $i(s, x) = \text{Sigmoid}(W^i(s, x))$, for all $s \in \statespace$, $x \in \calH$. 




Note that the first term in the above estimate has a very high variance, especially when the magnitude of $\sum_{h \in \Omega} \mu^n_{\Omega} (h \mid s) Q^n (s, h)$ is large. The variance is due to stochasticity of having $x$ in the set $\Omega$ or not. To understand this, consider that $c=0$ and $i(s, x) = 0.5$, fix all elements except for $x$ in $\Omega$, the update when $x$ is in $\Omega$ and the update when $x$ is not in $\Omega$ are in opposite directions and have similar magnitude!

We propose to use the following estimate, which generally has lower variance. We first sample a set of options $\Omega$ according to their interest functions. Then for each $x$, we construct two sets from $\Omega$, one of which contains $x$ while the other one doesn't. Specifically, $\Omega_u^+ \doteq \Omega \bigcup \{u\}$ is a set that contains all elements in $\Omega$ and also contains $u$, and let $\Omega_u^- = \Omega \backslash \{u\}$ is the set that contains all elements in $\Omega$ except for $u$ (if $u \not \in \Omega$, $u \not \in \Omega_u^-$). We obtain a new estimate 
\begin{align*}
    M^n(s, x) = i(s, x) \hat m_{\Omega_x^+, \vtheta}^n (s, x) + (1 - i(s, x)) \hat m_{\Omega_x^-, \vtheta}^n (s, x).
\end{align*}
Note that the r.h.s. is the weighted average of including $x$ and not including it in the set.
\begin{align} \label{eq: final estimate of m}
    & M^n(s, x) \nonumber \\
    & = i(s, x)  \left (\sum_{h \in \Omega_x^+} \mu^n_{\Omega_x^+} (h \mid s) \Bigg ( \bfI(x \in \Omega_x^+) - i(s, x) \Bigg) Q^n (s, h) - c i(s, x) (1 - i(s, x)) \right) + \nonumber \\
    & = (1 - i(s, x))  \left ( \sum_{h \in \Omega_u^-} \mu^n_{\Omega_x^-} (h \mid s) \Bigg ( \bfI(x \in \Omega_x^-) - i(s, x) \Bigg) Q^n (s, h) - c i(s, x) (1 - i(s, x)) \right ) \nonumber\\
    & = i(s, x) (1 - i(s, x) )  \left ( \sum_{h \in \Omega^+_x} \mu_{\Omega_x^+}^n(h \mid s) Q^n (s, h) - \sum_{h \in \Omega_x^-} \mu_{\Omega_x^-}^n(h \mid s) Q^n (s, h) - c\right ).
\end{align}

Note that $\sum_{h \in \Omega_x^+} \mu_{\Omega_x^+}^n(h \mid s) Q^n (s, h) - \sum_{h \in \Omega_x^-} \mu_{\Omega_x^-}^n(h \mid s) Q^n (s, h) - c$ has clear meaning -- it is the advantage of including the option $h$ given $\Omega$. Note that the first two terms is the difference of values of including $x$ in the initiation set or not, given all other options in $\Omega$. The third term $c$ is the extra computation cost when taking one more option into consideration. To better understand the above estimate, as an example, consider the case when option $x$ is not likely to be chosen by $\mu$ at some state, then the first two terms almost cancel out and only the third term is left. In this case $\tilde m$ is negative and updating $i$ using $\tilde m$ decreases the interest at this state.

Note that $c$ is a parameter of the objective and does not necessarily be used by the algorithm. In our algorithm, we use $M^n(s, x)$ but replace $c$ with $\bar c$. In our experiments, we have tested different values of $\bar c$.



In addition to an estimate of $m$, we also need an estimate of $v$. Again this can be constructed using a sample $\Omega$ directly:
\begin{align*}
    \hat v_\vtheta^n(s) \doteq \sum_{h \in \Omega} \mu_{\Omega}^n(h \mid s) Q(n, s, h) - c |\Omega|.
\end{align*}
Similar as what we did for $m$, we use the following estimate for $v$:
\begin{align}\label{eq: estimate of v}
    & V^n(s) \doteq \nonumber\\
    & \frac{1}{|\calH^{adj}|}\sum_{x \in \calH^{adj}} \left ( i(s, x) \sum_{h \in \Omega_{x}^+} \mu_{\Omega_x^+}^n(h \mid s) Q^n(s, h) + (1 - i(s, x)) \sum_{h \in \Omega_{x}^-} \mu_{\Omega_x^-}^n(h \mid s) Q^n(s, h) \right)\nonumber \\
    & - c \sum_{x \in \calH^{adj}} i(s, x).
\end{align}
Similar as what we did for $M^n$, when we use $V^n(s)$ in our algorithm, we replace $c$ with $\bar c$. In the experiments, we have tested different values of $\bar c$.

\textbf{FPOC Update Rules.}

FPOC updates $W^\pi$ with
\begin{align*}
    W^\pi(S_t, H_t, a) \overset{\alpha}{\gets} \left ( (\bfI(a=A_t) - \pi(a \mid S_t, H_t)) \delta_t(H_t) + \eta \frac{\partial \text{Ent}(\pi(\cdot \mid S_t, H_t))}{\partial W^\pi(S_t, H_t, a)} \right ), \ \forall a \in \calA,
\end{align*}
and keep all other elements unchanged. Here $a \overset{\alpha}{\gets} b$ means $a = a + \alpha b$. $\alpha$ is a step-size, $\bfI_{a=A_t} - \pi(a \mid S_t, H_t) = \partial \log \pi(A_t \mid S_t, H_t) / \partial W^\pi(S_t, H_t, a), \forall a \in \actionspace$, 
$\text{Ent}(\pi(\cdot \mid s, h))$ is the entropy of $\pi(\cdot \mid s, h)$, and $\partial \text{Ent}(\pi(\cdot \mid s, h))/\partial W^\pi(s, h, a) = -\pi(a \mid s, h) (\log \pi(a \mid s, h) + \text{Ent}(\pi(\cdot \mid s, h))$ is the partial derivative of the entropy (see \cref{lemma: entropy}),
and 
\begin{align}\label{eq: td error}
    \delta_t(H_t) & \doteq R_{t+1} + U_t(S_{t+1}, H_t) - Q_t^{N}(S_t, H_t)
\end{align}
is the temporal difference (TD) error 
with
\begin{align}
    U^{N_t}(S_{t+1}, H_t) \doteq \begin{cases}
        \gamma (1 - Z_{t+1}) V^{N}(S_{t+1}) & \text{w.p. }  \beta_{t+1}\\
        \gamma (1 - Z_{t+1}) Q^{N}(S_{t+1}, H_t) & \text{w.p. } 1 - \beta_{t+1}
    \end{cases}, \label{eq: sample U}
\end{align} 
where $\beta_{t+1} \doteq \beta(S_{t+1}, H_t)$, and $V^N(S_{t+1})$ is defined in \eqref{eq: estimate of v}.
FPOC updates $W^\beta$ with
\small
\begin{align*}
    W^\beta(S_{t+1}, H_t) \overset{\alpha}{\gets} \gamma (Z_{t+1} - 1)\beta_{t+1} (1 - \beta_{t+1}) \left(Q^{N}(S_{t+1}, H_t) - V^{N}(S_{t+1}) - \eta \frac{\partial \text{Ent}([\beta_{t+1}, 1 - \beta_{t+1}])}{\partial W^\beta(S_{t+1}, H_t)} \right),
\end{align*}
\normalsize
where $\beta_{t+1} (1 - \beta_{t+1}) = \partial \beta_{t+1} / \partial W^\beta(S_{t+1}, H_t)$,  
$\partial \text{Ent}([\beta_{t+1}, 1 - \beta_{t+1}])/\partial W^\beta(S_{t+1}, H_t) = \beta_{t+1} (1 - \beta_{t+1}) \log ((1 - \beta_{t+1})/\beta_{t+1})$ (see \cref{lemma: entropy}).

FPOC updates $W^i$ with
\begin{align*}
    W^i(S_t, h) \overset{\alpha}{\gets} \gamma \max(\beta_{t}, Z_{t-1}) \Bigg ( M^n(S_t, h) + \eta \frac{\partial \text{Ent} ([i(S_t, h), 1 - i(S_t, h)])}{\partial W^i(S_{t}, h)}\Bigg), \ \forall h \in \calH^{adj},
\end{align*}
where $M^n(S_t, h)$ is defined in \eqref{eq: final estimate of m}, $\partial \text{Ent}([i(S_t, h), 1 - i(S_t, h)])/\partial W^i(S_{t+1}, h) = i(S_t, h) (1 - i(S_t, h)) \log ((1 - i(S_t, h))/i(S_t, h))$.
Here we use $\max(\beta_t, Z_{t-1})$ so that the agent always updates $W^i$ for the first state of the episode (remember that $Z_{t-1} = 1$ if $S_t$ is the initial state).

Finally, our algorithm updates $Q$ for all options $h \in \calH$ with
\begin{align}\label{eq: update one q}
    Q^{N}(S_t, h) & \gets Q^{N}(S_t, h) +
     \alpha \rho_t(h) \delta_t(h),
\end{align}
where $\rho_t(h) \doteq \max \left(1, \frac{\pi(A_t \mid S_t, h)}{\pi(A_t \mid S_t, H_t)} \right)$ is the clipped importance sampling ratio. 

\begin{algorithm}[H]
\DontPrintSemicolon
\SetAlgoLined
\KwIn{exploration parameter $\epsilon \in [0, 1]$, stepsize $\alpha \in (0, 1)$, cost of choosing options $\bar c \geq 0$, weight for entropy regularization $\eta \geq 0$, discount factor $\gamma \leq 1$.}
\SetKwInput{AP}{Algorithm parameters}
\SetKwRepeat{Do}{do}{while}
Initialize $Q^n \in \bbR^{\cardS \times \cardH}, \forall n \in \calN, W^\pi \in \bbR^{\cardS \times |\calH^{adj}| \times \cardA}, W^\beta \in \bbR^{\cardS \times |\calH^{adj}|}, W^i \in \bbR^{\cardS \times |\calH^{adj}|}$ arbitrarily (e.g., $0$).\\
Sample initial task $N$ from $\calN$, initial state $S \sim d_0$.\\
\While{still time to train}
{
Sample an initiation set $\Omega \sim \prod_{h \in \calH} i(S, h)$.\\
$H \sim \epsilon\text{-greedy}(\{Q^N(S, h), \forall h \in \Omega\})$, $A \sim \pi(\cdot \mid S, H)$.\\
Take action $A$, observe next reward $R$, next state $S'$, episode termination signal $Z$.\\
Obtain $V^N(S')$ using $\eqref{eq: estimate of v}$.\\
$\beta \gets \beta(S', H)$.\\
$
\delta(h) \gets R - Q^N(S, h) + \begin{cases}
\gamma (1 - Z) V^N(S') & \text{w.p. } \beta\\
\gamma (1 - Z) Q^{N}(S', h) & \text{w.p. } 1 - \beta
\end{cases}, \forall h \in \calH
$.\\
\If {$H \in \calH^{adj}$}{
\For{$a = 1, 2, \dots, \actionspace$}{
$W^\pi(S, H, a) \overset{\alpha}{\gets} ((\mathbf{1}_{a=A} - \pi(a \mid S, H)) \delta(H) - \eta \pi(a \mid S, H) (\log \pi(a \mid S, H) + \text{Ent}(\pi(\cdot \mid S, H)))$.
}
$W^\beta(S', H) \overset{\alpha}{\gets} \gamma (Z - 1)  \beta (1 - \beta) \Big(Q^{N}(S', H) - V - \eta \log \frac{(1 - \beta)} {\beta}\Big)$.
}
Obtain $M^N(S, h)$ for all $h \in \calH^{adj}$ using \eqref{eq: final estimate of m}.\\
$i(h) \gets i(S, h)\ \forall h \in \calH^{adj}$.\\
$W^i(S, h) \overset{\alpha}{\gets} \gamma \beta^- \left(M^N(S, h) + \eta i(h) (1 - i(h)) \log \frac{1 - i(h)}{i(h)} \right),\ \forall h \in \calH^{adj}$.\\
$\rho(h) \gets \max\left(1, \frac{\pi(A \mid S, h)}{\pi(A \mid S, H)}\right)$.\\
$Q^{N}(S, h) \overset{\alpha}{\gets} \rho(h) \delta(h), \forall h \in \calH$ .\\
$\beta^- = \beta$.\\
\If {$Z = 1$} 
{
Sample $N$ from $\calN$.\\
Sample $S \sim d_0$.\\
$\beta^- = 1$.
}
}
\caption{The Fast Planning Option-Critic Algorithm}
\label{algo:On-Policy Actor-Critic + Q-Learning Algorithm}
\end{algorithm}

\section{Proofs}


\subsection{Gradient of the Objective}\label{sec: Gradient of the Objective}
For simplicity, in the proof we omit $\vtheta$ when writing functions that depend on it. These functions include $i_\vtheta, \pi_\vtheta$, $\beta_\vtheta$, $\mu_\vtheta$, $v_\vtheta$, $q_\vtheta$, $d_{\alpha, \gamma, \vtheta}$, $J(\vtheta)$ etc. We also omit it in the gradient operator $\nabla_\vtheta$. Finally, we omit the task index and the following arguments apply to all tasks.

By definition,
\begin{align*}
v(s) & = \sum_h \sum_{\Omega \in \calP(\calH)} \Pr(\Omega \mid s) \frac{\bfI(h \in \Omega)f(s, h)}{\sum_{\bar h \in \Omega} f(s, \bar h)} q(s, h) - c \sum_h i(s, h) \\
& = \sum_{\Omega \in \calP(\calH)} \Pr(\Omega \mid s) \sum_{h \in \Omega} \frac{f(s, h)}{\sum_{\bar h \in \Omega} f(s, \bar h)} q(s, h) - c \sum_h i(s, h) 
\end{align*}
Consider the gradient of the first term.
\begin{align*}
    & \nabla \sum_{\Omega \in \calP(\calH)} \Pr(\Omega \mid s) \sum_{h \in \Omega} \frac{f(s, h)}{\sum_{\bar h \in \Omega} f(s, \bar h)} q(s, h) \\ 
    & = \nabla \sum_{\Omega \in \calP(\calH)} \Pr(\Omega \mid s) \sum_{h \in \Omega} \mu_\Omega(h \mid s) q(s, h) \\ 
    & = \sum_{\Omega \in \calP(\calH)} \left ( \nabla \Pr(\Omega \mid s) \right) \sum_{h \in \Omega} \mu_\Omega(h \mid s) q(s, h) + \sum_{\Omega \in \calP(\calH)} \Pr(\Omega \mid s) \sum_{h \in \Omega} \nabla \mu_\Omega(h \mid s)  q(s, h) \\
    & + \sum_{\Omega \in \calP(\calH)} \Pr(\Omega \mid s) \sum_{h \in \Omega} \mu_\Omega(h \mid s) \nabla q(s, h) \\
    & = \sum_{\Omega \in \calP(\calH)} \Pr(\Omega \mid s) \sum_{h \in \Omega} \mu_\Omega(h \mid s) \left ( \nabla \log \Pr(\Omega \mid s) \right) q(s, h) + \sum_{\Omega \in \calP(\calH)} \Pr(\Omega \mid s) \sum_{h \in \Omega} \mu_\Omega(h \mid s) \left ( \nabla \log \mu_\Omega(h \mid s) \right ) q(s, h) \\
    & + \sum_{\Omega \in \calP(\calH)} \Pr(\Omega \mid s) \sum_{h \in \Omega} \mu_\Omega(h \mid s) \nabla q(s, h) \\
    & = \sum_{\Omega \in \calP(\calH)} \Pr(\Omega \mid s) \sum_{h \in \Omega} \mu_\Omega(h \mid s) \Bigg (\nabla \log \Pr(\Omega \mid s)  q(s, h) + \nabla \log \mu_\Omega(h \mid s) q(s, h) + \nabla q(s, h) \Bigg )
\end{align*}

\begin{align*}
    \nabla \log \Pr(\Omega \mid s) & = \nabla \log \prod_{\bar h \in \Omega} i(s, \bar h) \prod_{\bar h \not \in \Omega} (1 - i(s, \bar h))\\
    & = \nabla \sum_{\bar h \in \Omega} \log i(s, \bar h) + \nabla \sum_{\bar h \not \in \Omega} \log (1 - i(s, \bar h))\\
    & = \sum_{\bar h \in \Omega} \nabla \log i(s, \bar h) + \sum_{\bar h \not \in \Omega} \nabla \log (1 - i(s, \bar h))
\end{align*}


Therefore,
\begin{align*}
    & \nabla v(s) = \sum_{\Omega \in \calP(\calH)} \Pr(\Omega \mid s) \sum_{h \in \Omega} \mu_\Omega(h \mid s) \nabla q(s, h) \\
    & + \sum_{\Omega \in \calP(\calH)} \Pr(\Omega \mid s) \sum_{h \in \Omega} \mu_\Omega(h \mid s) \\
    & \Bigg (\Big ( \nabla \sum_{\bar h \in \Omega} \log i(s, \bar h) + \nabla \sum_{\bar h \not \in \Omega} \log (1 - i(s, \bar h)) + \nabla \log \mu_\Omega(h \mid s) \Big ) q(s, h) \Bigg ) \\
    & - c \sum_{h \in \calH} \nabla i(s, h)\\
    & = \sum_{h \in \calH} \mu(h \mid s) \nabla q(s, h) + m(s),
\end{align*}
where $\mu(h \mid s) \doteq \sum_{\Omega \in \calP(\calH)} \Pr(\Omega \mid s) \sum_{h \in \Omega} \mu_\Omega(h \mid s) $, and
\begin{align*}
   m(s) & \doteq \sum_{\Omega \in \calP(\calH)} \Pr(\Omega \mid s) \sum_{h \in \Omega} \mu_\Omega(h \mid s) \\
   & \Bigg (\Big ( \nabla \sum_{\bar h \in \Omega} \log i(s, \bar h) + \nabla \sum_{\bar h \not \in \Omega} \log (1 - i(s, \bar h)) + \nabla \log \mu_\Omega(h \mid s) \Big ) q(s, h) \Bigg ) - c \sum_{h \in \calH} \nabla i(s, h).
\end{align*}


\begin{align*}
    \nabla J & = \nabla \sum_{s} \alpha(s) v (s)\\
    & = \sum_{s} \alpha(s) \left (m(s) +  \sum_{h} \mu(s, h) \nabla q(s, h) \right )
\end{align*}

\begin{align}
    \nabla q  (s, h) & =  \nabla \sum_a \pi(a \mid s, h) q  (s, h, a) \nonumber\\
    & = \sum_a \left (\nabla \pi  (a \mid  s, h) q  (s, h, a) + \pi  (a \mid s, h) \nabla q  (s, h, a) \right ) \label{eq: grad_qso}
\end{align}

\begin{align}
    \nabla q  (s, h, a) & = \gamma \sum_{s'} p(s' \mid s, a) \nabla \left( \beta (s', h) v  (s') + (1 - \beta  (s', h)) q  (s', h) \right)  \nonumber \\
    & = \gamma \sum_{s'} p(s' \mid s, a) (\nabla q  (s', h) - \nabla (\beta  (s', h) ( q  (s', h) - v  (s')) ) \label{eq: grad_qsoa}
\end{align}

\begin{align}
    \nabla (\beta  (s', h) ( q  (s', h) - v  (s') )) & = \nabla \beta  (s', h)  (q  (s', h) - v  (s')) + \beta  (s', h) \nabla (q  (s', h) - v  (s')) \label{eq: grad_betaadv}
\end{align}

\begin{align*}
    \nabla (q  (s', h) - v  (s')) = \nabla q  (s', h) - m(s') - \sum_{h'} \mu(h' \mid s') \nabla q(s', h')
\end{align*}

substitute back to \eqref{eq: grad_betaadv}, we have
\begin{align*}
     & \nabla (\beta(s', h) ( q(s', h) - v(s')))\\
     & = \nabla \beta(s', h)  (q(s', h) - v(s')) + \beta(s', h) \Bigg ( \nabla q  (s', h) - m(s') - \sum_{h'} \mu(h' \mid s') \nabla q(s', h') \Bigg)
\end{align*}

substitute back to \eqref{eq: grad_qsoa}, we have

\begin{align*}
    & \nabla q(s, h, a) = \gamma \sum_{s'} p(s' \mid s, a) \Bigg ( \nabla q(s', h) \\
    & - \Bigg ( \nabla \beta(s', h)  (q(s', h) - v(s')) + \beta(s', h) \Bigg ( \nabla q  (s', h) - m(s') - \sum_{h} \mu(h' \mid s') \nabla q(s', h') \Bigg) \Bigg) \Bigg)
\end{align*}

substitute back to \eqref{eq: grad_qso}, we have
\begin{align}
    & \nabla q(s, h) \nonumber\\
    & = \sum_{a}  \nabla \pi(a \mid s, h)  q(s, h, a) + \gamma \sum_{a, s'} \pi(a \mid s, h)  p(s' \mid s, a) \Bigg ( \nabla q(s', h) \nonumber \\
    & - \Bigg ( \nabla \beta(s', h)  (q(s', h) - v(s')) + \beta(s', o) \Bigg ( \nabla q  (s', h) - m(s') - \sum_{h'} \mu(h' \mid s') \nabla q(s', h') \Bigg) \Bigg) \Bigg) \nonumber\\
    & = \sum_{a}  \nabla \pi(a \mid s, h)  q(s, h, a) \nonumber\\
    & + \gamma \sum_{a, s'} \pi(a \mid s, h) p(s' \mid s, a) \left ( - \nabla \beta(s', h)  (q(s', h) - v(s')) + \beta(s', h) m(s') \right) \nonumber \\
    & + \gamma \sum_{a, s'} \pi(a \mid s, h) p (s' \mid s, a) \left ( (1 - \beta(s', h) ) \nabla q(s', h) + \beta(s', h) \sum_{h'} \mu(h' \mid s')  \nabla q(s', h')  \right ) \label{eq: grad_qso_extended}
\end{align}

Let

\begin{align*}
    H(s, h)
    & \doteq \sum_{a} \nabla \pi(a \mid s, h) q(s, h, a) + \gamma \sum_{a, s'} \pi(a \mid s, h) p(s' \mid s, a) \\
    & \Bigg( - \nabla \beta(s', h)  (q(s', h) - v(s')) + \beta(s', h) m(s') \Bigg )
\end{align*}

Then we have \eqref{eq: grad_qso_extended} is equal to

\begin{align}
    & H(s, h) + \gamma \sum_{a, s'} \pi(a \mid s, h) p (s' \mid s, a) \left ( (1 - \beta(s', h)) \nabla q(s', h) + \beta(s', h) \sum_{h'} \mu(h' \mid s')  \nabla q(s', h') \right ) \nonumber \\
    & =  \sum_{a} \pi(a \mid s, h) \Bigg (H(s, h) + \gamma \sum_{s', h'} p(s' \mid s, a) \left ( (1 - \beta(s', h)) \bfI(h' = h) + \beta(s', h) \mu(h' \mid s') \right)  \nabla q(s', h') \Bigg) \label{recursive H(s, o)}
\end{align}

Now let's construct a new MDP by defining $(s, h)$ as a new state $\bar{s}$, $H(s, h)$ as the reward $\bar{r}(\bar{s})$ for state $\bar{s} = (s, h)$, $a$ as the action $\bar{a}$ and the transition probability $\bar{p}(\bar{s}' \mid \bar{s}, \bar{a}) = \Pr(s', h' \mid s, h, a) = p(s' \mid s, a) ((1 - \beta(s', h)) \bbI_{h' = h} + \beta(s', h) \mu(h' \mid s'))$. Defined a new policy $\bar{\pi}(\bar{a} \mid \bar{s}) = \pi(a \mid s, h)$ in the new MDP. Then \eqref{recursive H(s, o)} can be rewritten to

\begin{align*}
    \nabla q (\bar{s}) = \sum_{\bar{a}} \bar{\pi}(\bar{a} \mid \bar{s}) (r(\bar{s}) + \gamma \sum_{\bar{s}'} \bar{p}(\bar{s}' \mid \bar{s}, \bar{a}) \nabla q (\bar{s}))
\end{align*}

The above equation is a Bellman equation and we see that $\nabla q (s, h)$ is the value function for this new MDP for policy $\bar{\pi}$, i.e., $\bar{v}_{\bar{\pi}} = \nabla q (s, h)$

Furthermore write the above Bellman equation in vector form.

\begin{align*}
    \bar{\vv}_{\bar{\pi}} = (\mI - \gamma \bar{\mP}_{\bar{\pi}})^{-1} \bar{\vr}_{\bar{\pi}}
\end{align*}

With a little abuse of notation, define $\alpha (s, h) \doteq \alpha(s) \mu(h \mid s)$ and $\bar{\alpha} (\bar{s}) \doteq \alpha (s, h)$

\begin{align*}
    \sum_{s} \alpha(s) \sum_h \mu(h \mid s) \nabla q(s, h)
    & = \sum_{s, h} \alpha (s, h) \nabla q (s, h) \\
    & = \sum_{\bar{s}} \bar{\alpha} (\bar{s}) \bar{v}_{\bar{\pi}}(\bar{s})\\
    & = \bar{\valpha}^\top \bar{\vv}_{\bar{\pi}}\\
    & = \bar{\valpha}^\top (\mI - \gamma \bar{\mP}_{\bar{\pi}})^{-1} \bar{\vr}_{\bar{\pi}}
\end{align*}

$\bar{\vd} \doteq \bar{\valpha}_\mu^\top (\mI - \gamma \bar{\mP}_{\bar{\pi}})^{-1}$ is the $|\calS|$-vector of the $\bar{d}(\bar{s})$ for all $\bar{s} \in \bar{\calS}$, where $\bar{d}(\bar{s})$ is the discounted weighting of state $\bar{s}$ in the new MDP. According to our definition, it is also the discounted weighting of state-option pair $(s, h)$, $d_{\alpha, \mu}(s, h)$ in the old MDP.

\begin{align*}
    & = \vd_{\alpha}^\top \bar{\vr}_{\bar{\pi}} \\
    & = \sum_{\bar{s}} d_{\bar{\alpha}_\mu} (\bar{s}) \bar{r}_{\bar{\pi}} (\bar{s})\\
    & = \sum_{s, h} d_{\alpha, \mu} (s, h) H(s, h)\\
    & = \sum_{s, h} d_{\alpha, \mu} (s, h) \sum_{a} \nabla \pi(a \mid s, h) q(s, h, a) \\
    & + \gamma \sum_{a, s'} \pi(a \mid s, h) p(s' \mid s, a) \Bigg( - \nabla \beta(s', h)  (q(s', h) - v(s')) + \beta(s', h) m(s') \Bigg )
\end{align*}

Finally we have

\begin{align*}
    & \nabla \sum_s \alpha(s) v_\mu(s) = \sum_{s} \alpha(s) m(s) + \sum_{s, h} d_{\alpha, \mu} (s, h) \Bigg ( \sum_{a} \nabla \pi(a \mid s, h) q (s, h, a) \\
    & + \gamma \sum_{a, s'} \pi(a \mid s, h) p(s' \mid s, a) \left( - \nabla \beta(s', h)  (q(s', h) - v(s')) + \beta(s', h) m(s') \right ) \Bigg)
\end{align*}
\subsection{Derivative of the Entropy Term}
\begin{lemma}\label{lemma: entropy}
For any $s \in \statespace, h \in \calH, \bar a \in \actionspace$, 
\begin{align*}
    & \frac{\partial H(\pi(\cdot \mid s, h))}{\partial W^\pi(s, h, \bar a)} \\
    &= -\pi(\bar a \mid s, h) \log \pi(\bar a \mid s, h) - \pi(\bar a \mid s, h) H(\pi(\cdot \mid s, h)),\\
    & \frac{\partial H([\beta(s, h), 1 - \beta(s, h)])}{W^\beta(s, h)} \\
    &= \beta(s, h) (1 - \beta(s, h)) \log \frac{(1 - \beta(s, h))} {\beta(s, h)}.
\end{align*}
\end{lemma}
\begin{proof}
Fix a $s \in \statespace$ and a $h \in \calH$, for simplicity, we use $p_a$ to denote $\pi(a \mid s, h)$ and use $w_a$ to denote $W^\pi(s, h, a)$.
\begin{align*}
    H(\pi(\cdot \mid s, h)) & = - \sum_a p_a \log p_a\\
    & = - \sum_a \frac{e^{w_a}}{\sum_{a'} e^{w_{a'}}} \log \frac{e^{w_a}}{\sum_{a'} e^{w_{a'}}}
\end{align*}

\begin{align*}
    & \frac{\partial H(\pi(\cdot \mid s, o))}{\partial w_{\bar a}} \\
    &= - \sum_a \frac{\partial \frac{e^{w_{a}}}{\sum_{a'} e^{w_{a'}}}}{\partial w_{\bar a}} \left(\log \frac{e^{w_{ a}}}{\sum_{a'} e^{w_{a'}}} + 1 \right)\\
    & = - \frac{e^{w_{\bar a}} \sum_{a'} e^{w_{a'}} - e^{w_{\bar a}} e^{w_{\bar a}}}{(\sum_{a'} e^{w_{ a'}})^2} \left(\log \frac{e^{w_{\bar a}}}{\sum_{a'} e^{w_{a'}}} + 1 \right)\\
    & - \sum_{a \neq \bar a} \frac{0 - e^{w_{a}}e^{w_{\bar a}}}{(\sum_{a'} e^{w_{a'}})^2} \left(\log \frac{e^{w_{a}}}{\sum_{a'} e^{w_{a'}}} + 1 \right)\\
    &= - (p_{\bar a} - p_{\bar a} p_{\bar a} )(\log p_{\bar a} + 1) + \sum_{a \neq \bar a} p_a p_{\bar a} (\log p_a + 1)\\
    & = - (p_{\bar a} - p_{\bar a} p_{\bar a} )\log p_{\bar a}  + \sum_{a \neq \bar a} p_a p_{\bar a} \log p_a \\
    & = - p_{\bar a}(1 - p_{\bar a} )\log p_{\bar a}  + \sum_{a \neq \bar a} p_a p_{\bar a} \log p_a \\
    & = - p_{\bar a}(1 - p_{\bar a} )\log p_{\bar a}  + \sum_{a} (p_a p_{\bar a} \log p_a ) - p_{\bar a} p_{\bar a} \log p_{\bar a} \\
    & = -p_{\bar a} \log p_{\bar a} - p_{\bar a} H(\pi(\cdot \mid s, h))
\end{align*}

Fix a $s \in \statespace$ and a $h \in \calH$, for simplicity, We use $p$ to denote $\beta(s, h)$ and use $w$ to denote $W^\beta(s, h)$.

\begin{align*}
    & \frac{\partial H([p, 1 - p])}{\partial w}\\
    &= \frac{\partial (- p \log p - (1 - p) \log (1 - p))}{\partial w}\\
    & = - \frac{\partial p}{\partial w} \log p - \frac{\partial p}{\partial w} - \frac{\partial (1 - p)}{\partial w} \log (1 - p) - \frac{\partial (1 - p)}{\partial w}\\
    & = - \frac{\partial p}{\partial w} ( \log p + 1) + \frac{\partial p}{\partial w} ( \log (1 - p) + 1)\\
    & = - \frac{\partial p}{\partial w} ( \log p - \log (1 - p)) \\
    & = - \frac{\partial p}{\partial w}  \log \frac{p}{(1 - p)} \\
    & = p (1 - p) \log \frac{(1 - p)} {p}
\end{align*}
\end{proof}

\subsection{Details of Experiments}\label{app: Details of Experiments}

\subsubsection{Testing Phase Details}
The testing phase consists of two stages: a model learning stage and a planning stage. In the model learning stage, we fix the learned options from the training phase and let the agent learn an option model in the training/testing tasks. In the planning stage, we apply option-value iteration with the learned model to obtain option values. Greedy meta-policies are then derived from the option value estimates.

Specifically, in the model learning stage, the agent interacts with the environment and learns option models for both training and testing tasks using the intra-option model learning algorithm introduced by Sutton, Precup, \& Singh (1999). The model learning stage lasts for $1000000$ steps and produces a reward model $r^n(s, o)$ and a dynamics model $p^n(s' \mid s, o)$ for each training task or testing task $n$. The planning stage involves applying the option-value iteration algorithm to solve either a set of training tasks or a set of testing tasks, which are similar to the training ones. Before applying planning, we first sample for each state $s$, a set of options, $\Omega(s)$ that planning algorithm may choose from. Then for each iteration, we apply the option-value iteration update
\begin{align*}
    Q^n(s, o) \gets r^n(s, o) + \sum_{s'} \gamma p^n(s' \mid s, o) \max_{o' \in \Omega(s')}Q^n(s', o'),
\end{align*}
for all $n \in \calN$, $s \in \calS$, and $o \in \Omega(s)$.
We then record the value error averaged over all states
\begin{align*}
    \text{Error} \doteq \frac{1}{|\calN| |\calS|}\sum_n \sum_s \bar v_*(s) - \argmax_{o \in \Omega(s)} Q(n, s, o).
\end{align*}
Option-value iteration is terminated when the average value error is less than $0.1$. Suppose that $k$ iterations are used before termination. The total elementary operations used to solve all tasks is then
\begin{align*}
    k \sum_s |\Omega(s)| \cardS |\calN|.
\end{align*}

\subsubsection{Tested Parameter Settings}

The tested parameters in the experiment presented in \cref{sec: FPOC} are summarized in \cref{tab:FPOC}.
\begin{table}[h]
    \centering
    \begin{tabular}{|c|c|}
        Number of adjustable options $k$ & $2, 4, 8$\\
        Cost of choosing options $\bar c$ & $0, 0.2, 0.4, 0.6, 0.8, 1$\\
        Entropy regularization weight $\eta$ & $0, 0.0025, 0.05, 0.1, 0.2, 0.4$\\
        Step-size $\alpha$ & $0.01$\\
        Exploration parameter $\epsilon$ & $0.1$\\
        Discount factor $\gamma$ & $1$
    \end{tabular}
    \caption{Tested Parameters for MAOC and FPOC.}
    \label{tab:FPOC}
\end{table}


\subsubsection{Best Three Parameter Settings}
We rank parameter settings by their resulting compound return averaged over the last five times evaluation and 10 runs in the training phase. 

The best three parameter settings for MAOC are summarized in \cref{tab:MAOC best five parameter settings}. The best parameter setting with two adjustable options is used to produce the learning curve shown in \cref{fig: MAOC learning curve}. 
The best three parameter settings for FPOC are summarized in \cref{tab:FPOC best five parameter settings}. It can be seen that in almost all cases, the choice of $\bar c$ is $0.2$, which is equal to $c$. This shows that, given the problem parameter $c$, one can just set $\bar c$ to be $c$ and does not perform a parameter search over $\bar c$.

\begin{table}[h]
    \centering
    \begin{tabular}{|c|c|c|c|}
        $k$ & $\bar c$ & $\eta$ & Estimate of $J$ \\
        $2$ & $0.2$ & $0.05$ & $- 13.33$\\
        $2$ & $0.2$ & $0.1$ &$- 13.33$\\
        $2$ & $0.2$ & $0.0025$ & $- 13.46$ \\
        $4$ & $0.2$ & $0.05$ & $- 13.23$\\
        $4$ & $0.2$ & $0.1$ & $-13.32$ \\
        $4$ & $0.2$ & $0.2$ & $- 13.37$ \\
        $8$ & $0.2$ & $0.05$ & $- 13.39$\\
        $8$ & $0.2$ & $0.1$ & $- 13.58$\\
        $8$ & $0.2$ & $0.2$ & $- 13.65$
    \end{tabular}
    \caption{MAOC's best three parameter settings with 2, 4, and 8 adjustable options.}
    \label{tab:MAOC best five parameter settings}
\end{table}

\begin{table}[h]
    \centering
    \begin{tabular}{|c|c|c|c|}
        $k$ & $\bar c$ & $\eta$ & Estimate of $J$ \\
        $2$ & $0.2$ & $0.05$ & $- 13.15$\\
        $2$ & $0.2$ & $0.1$ &$- 13.16$\\
        $2$ & $0.4$ & $0.2$ & $- 13.46$ \\
        $4$ & $0.2$ & $0.05$ & $- 12.62$\\
        $4$ & $0.2$ & $0.0025$ & $-12.78$ \\
        $4$ & $0.2$ & $0$ & $- 12.79$ \\
        $8$ & $0.2$ & $0$ & $- 12.21$\\
        $8$ & $0.2$ & $0.0025$ & $- 12.23$\\
        $8$ & $0.2$ & $0.05$ & $- 12.34$
    \end{tabular}
    \caption{FPOC's best three parameter settings with 2, 4, and 8 adjustable options.}
    \label{tab:FPOC best five parameter settings}
\end{table}


\subsubsection{Parameter Study}

To understand the MAOC and the FPOC algorithms' sensitivity of their parameters, we vary each parameter and choose other parameters to be best and plot the resulting learning curves in \cref{fig: parameter study}. 

\begin{figure*}[h]
\begin{subfigure}{0.5\textwidth}
    \centering
    \includegraphics[width=\textwidth]{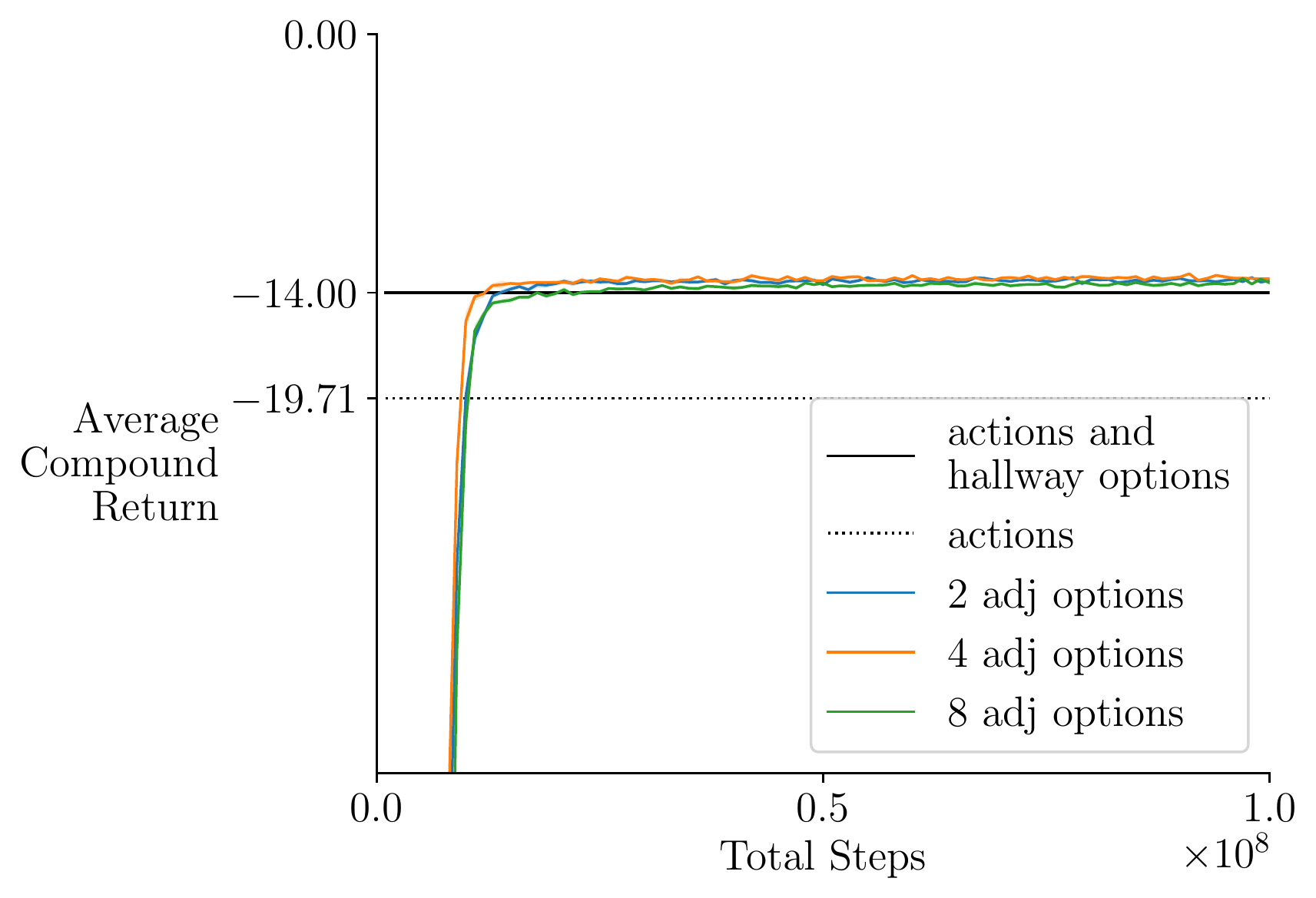}
    \caption{Number of adjustable options (MAOC)}
    \label{fig: MAOC adj options study}
\end{subfigure}%
\begin{subfigure}{0.5\textwidth}
    \centering
    \includegraphics[width=\textwidth]{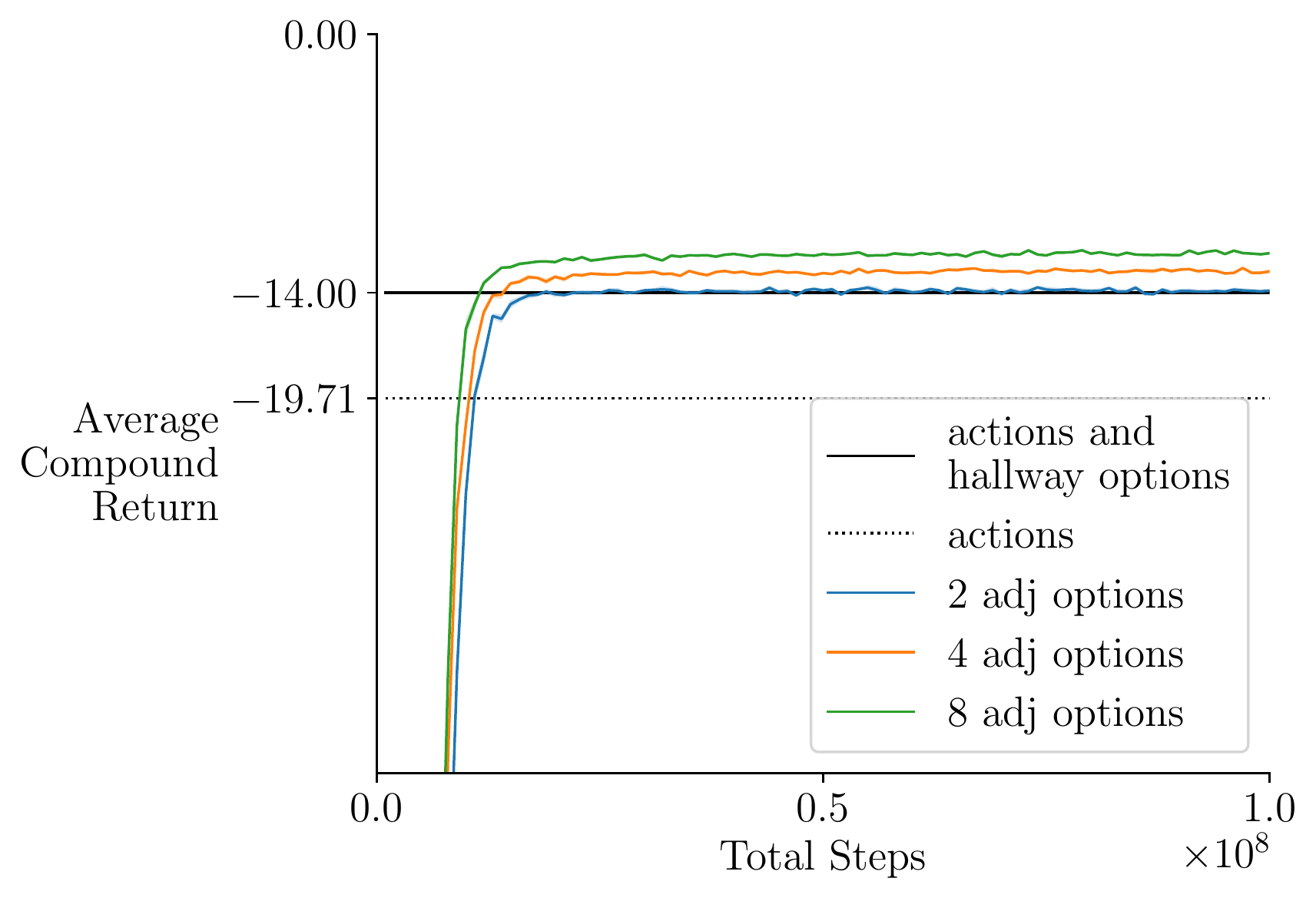}
    \caption{Number of adjustable options (FPOC)}
    \label{fig: FPOC adj options study}
\end{subfigure}
\begin{subfigure}{0.5\textwidth}
    \centering
    \includegraphics[width=\textwidth]{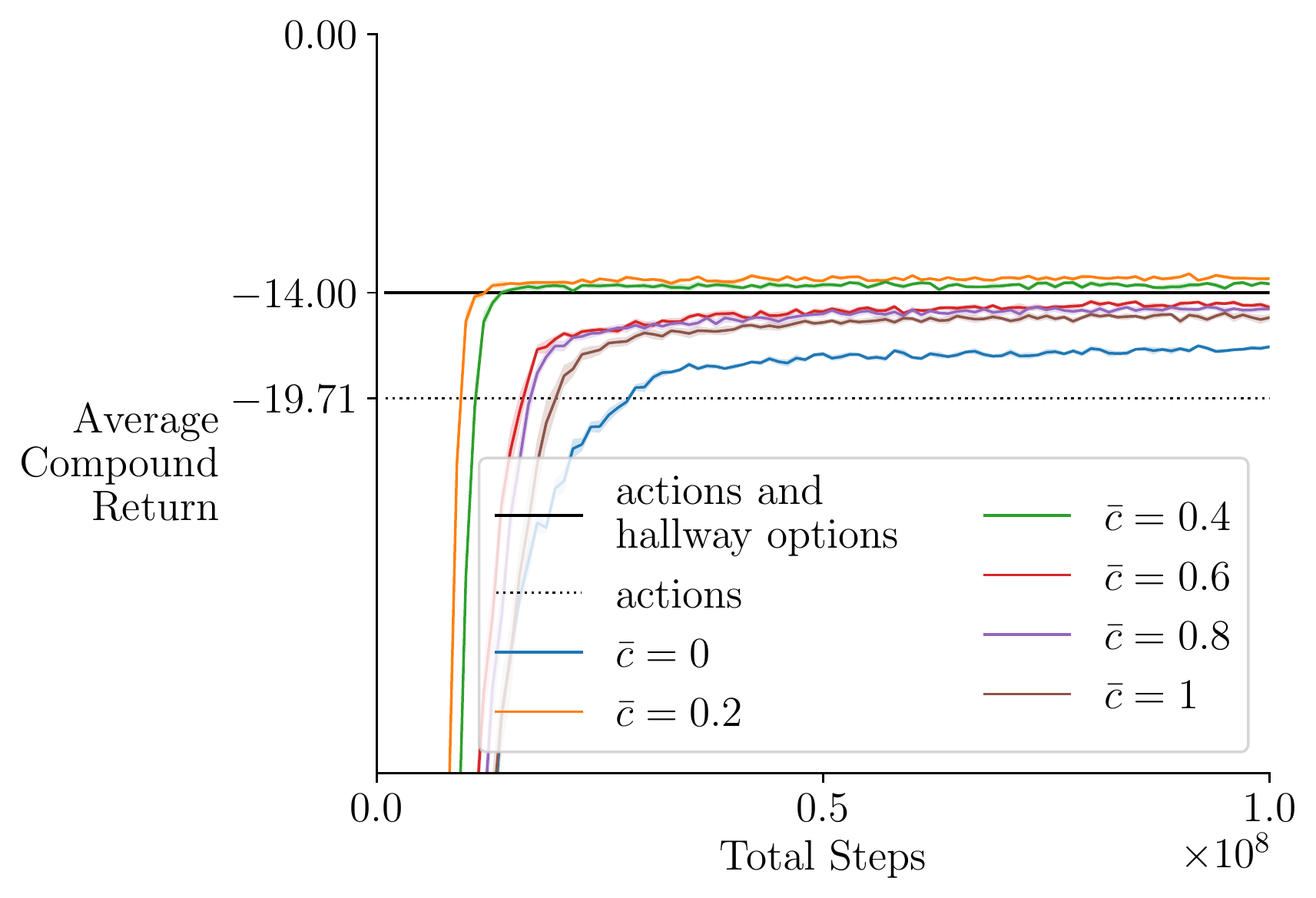}
    \caption{$\bar c$ (MAOC)}
    \label{fig: MAOC c study}
\end{subfigure}%
\begin{subfigure}{0.5\textwidth}
    \centering
    \includegraphics[width=\textwidth]{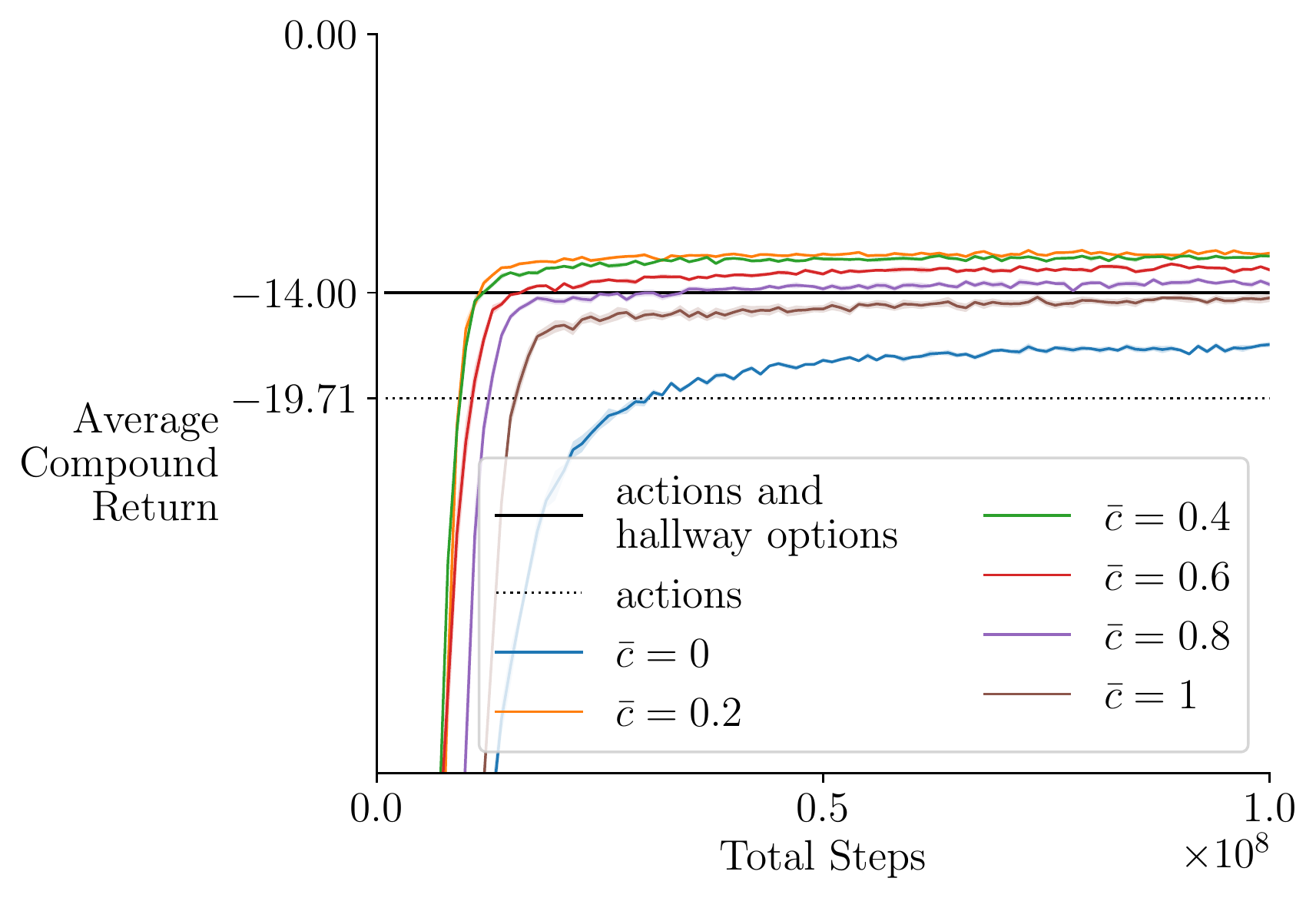}
    \caption{$\bar c$ (FPOC)}
    \label{fig: FPOC c study}
\end{subfigure}
\begin{subfigure}{0.5\textwidth}
    \centering
    \includegraphics[width=\textwidth]{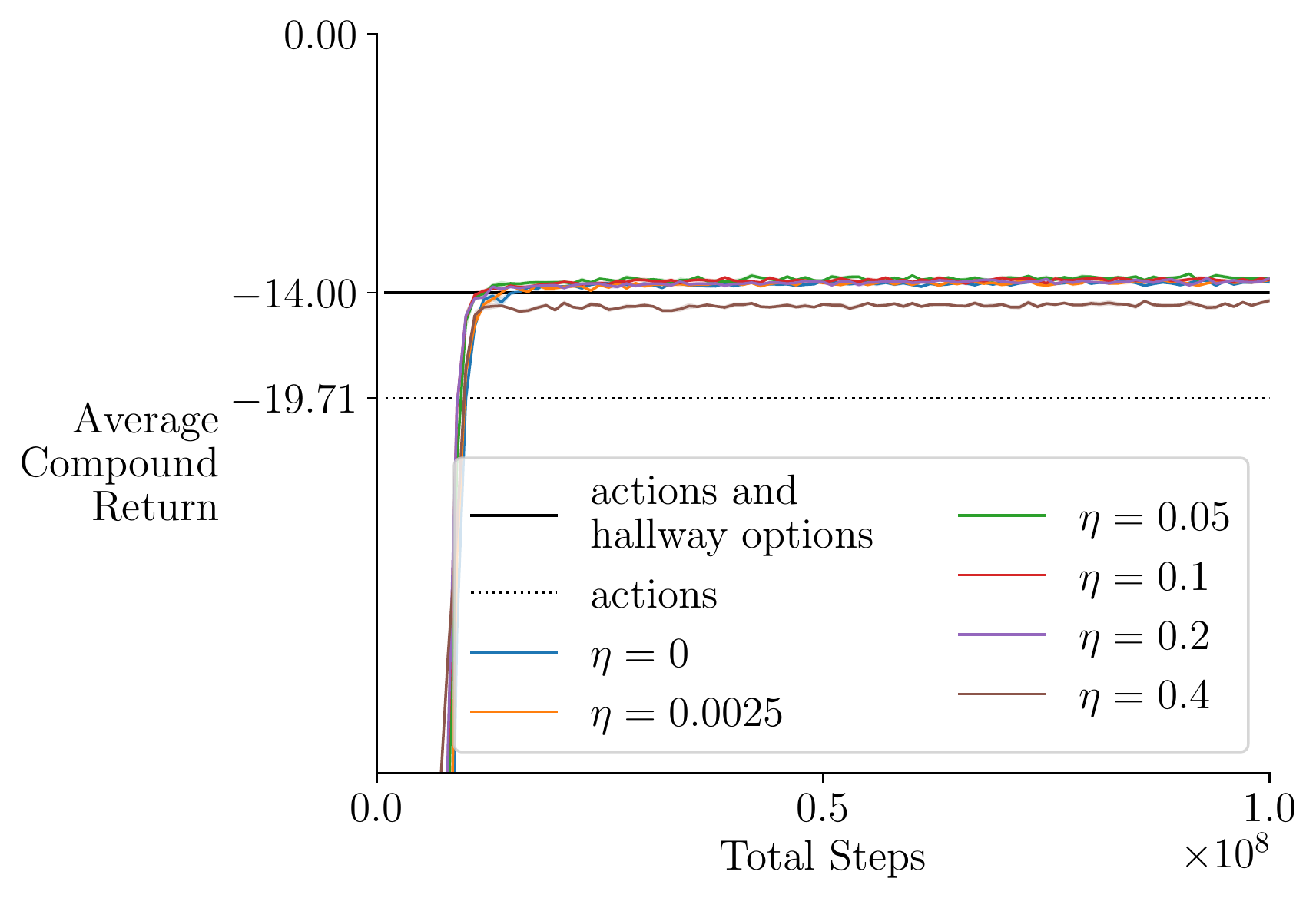}
    \caption{$\eta$ (MAOC)}
    \label{fig: MAOC eta study}
\end{subfigure}%
\begin{subfigure}{0.5\textwidth}
    \centering
    \includegraphics[width=\textwidth]{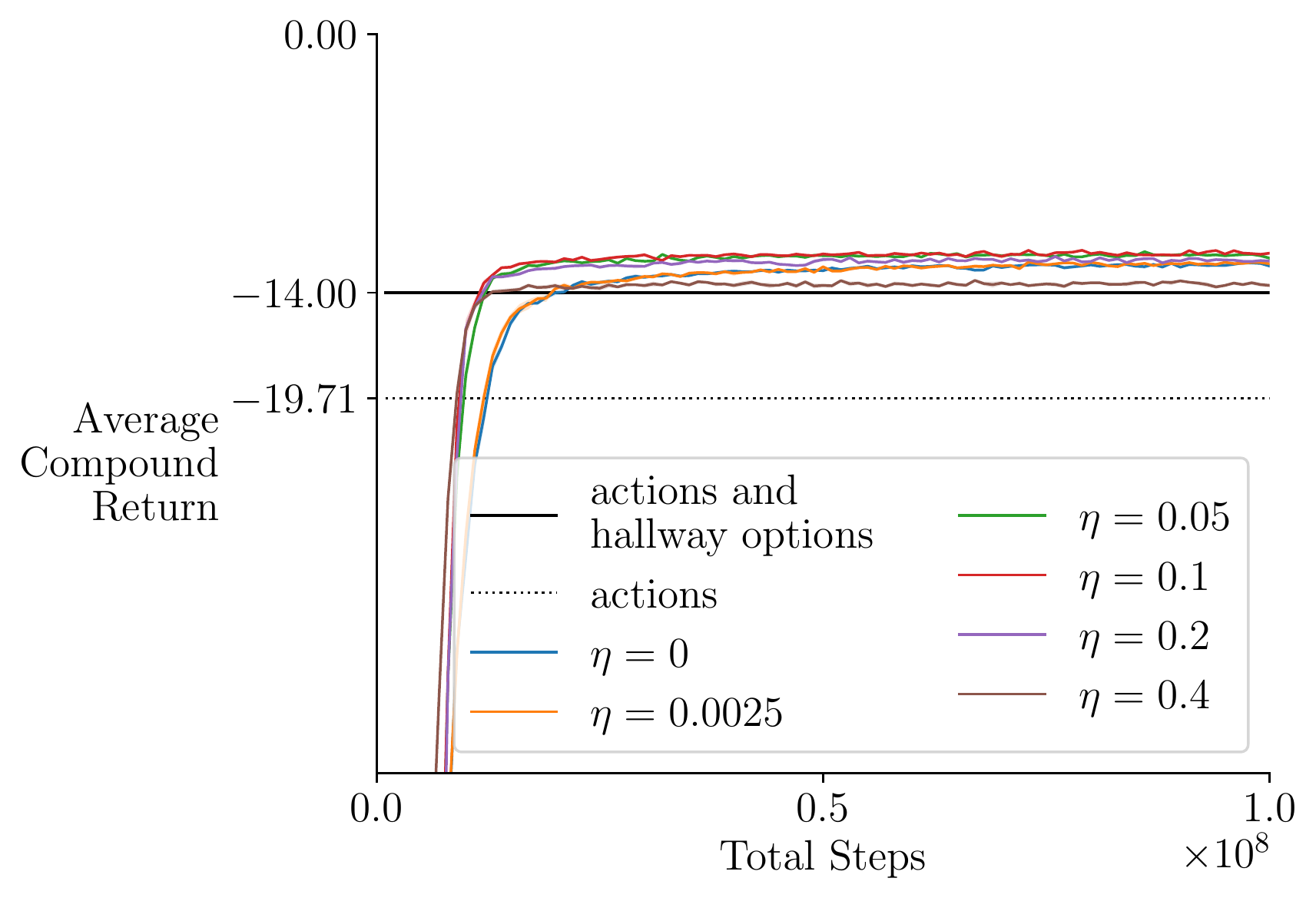}
    \caption{$\eta$ (FPOC)}
    \label{fig: FPOC eta study}
\end{subfigure}
\caption{
The MAOC and FPOC algorithms' sensitivity to their parameters. In each sub-plot, we show the learning curves produced with different values of one of the algorithm's parameters. Other parameters are chosen to be the best. The $x$-axis is the number of time-steps, the $y$-axis is the compound return with $c = 0.2$, averaged over $500$ evaluation episodes. a) For MAOC, the algorithm's asymptotic performance is not sensitive to the choice of number of options. b) FPOC with more adjustable options achieve a better objective value. c) The best $\bar c$ is $0.2$, which, unsurprisingly, equals to $c$. d) Again, the best $\bar c$ is $0.2$. e) No clear improvement with the entropy term. f) The entropy term plays a more important role for FPOC. With $\eta = 0.1$, the algorithm not only achieves faster learning but also obtains a slightly higher objective value eventually.
}
\label{fig: parameter study}
\end{figure*}

\subsubsection{Relation Between $J$ (with $c = 0.2$) and Elementary Operations Used by Option-Value Iteration}
We show in \cref{fig: MAOC study planning iterations vs estimate of J} the relation between the average compound return (with $c = 0.2$) estimated in the training tasks and the number of elementary required to solve the training or testing tasks using the option-value iteration algorithm given a perfect model of the learned options. We show the same figure for FPOC in \cref{fig: FPOC study planning iterations vs estimate of J}.

\begin{figure*}[h]
\begin{subfigure}{0.5\textwidth}
    \centering
    \includegraphics[width=\textwidth]{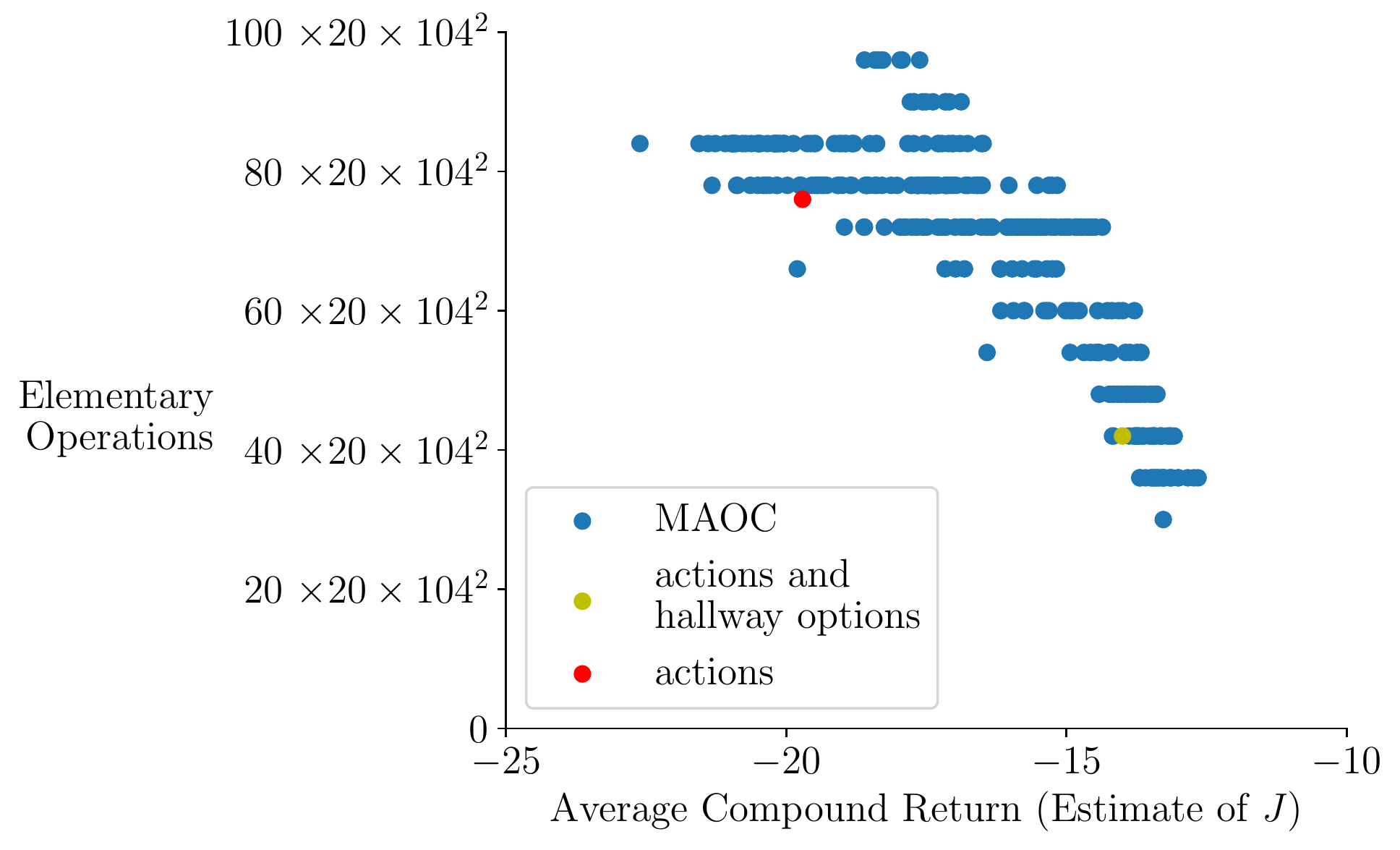}
    \caption{$k = 2$, Training Tasks}
\end{subfigure}%
\begin{subfigure}{0.5\textwidth}
    \centering
    \includegraphics[width=\textwidth]{figures/MAOC_planning_test_2_adj_options.pdf}
    \caption{$k = 2$, Testing Tasks}
\end{subfigure}
\begin{subfigure}{0.5\textwidth}
    \centering
    \includegraphics[width=\textwidth]{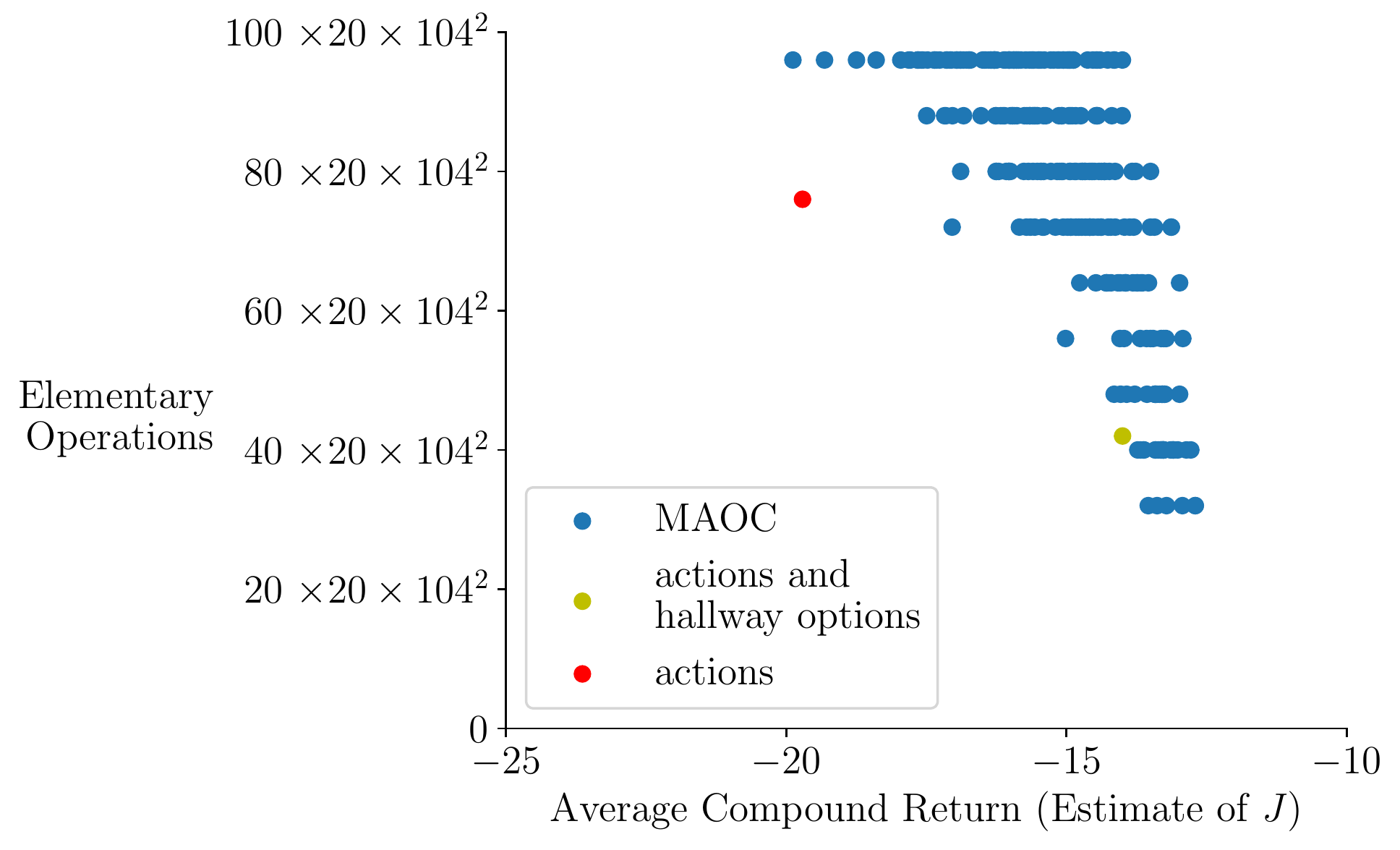}
    \caption{$k = 4$, Training Tasks}
\end{subfigure}%
\begin{subfigure}{0.5\textwidth}
    \centering
    \includegraphics[width=\textwidth]{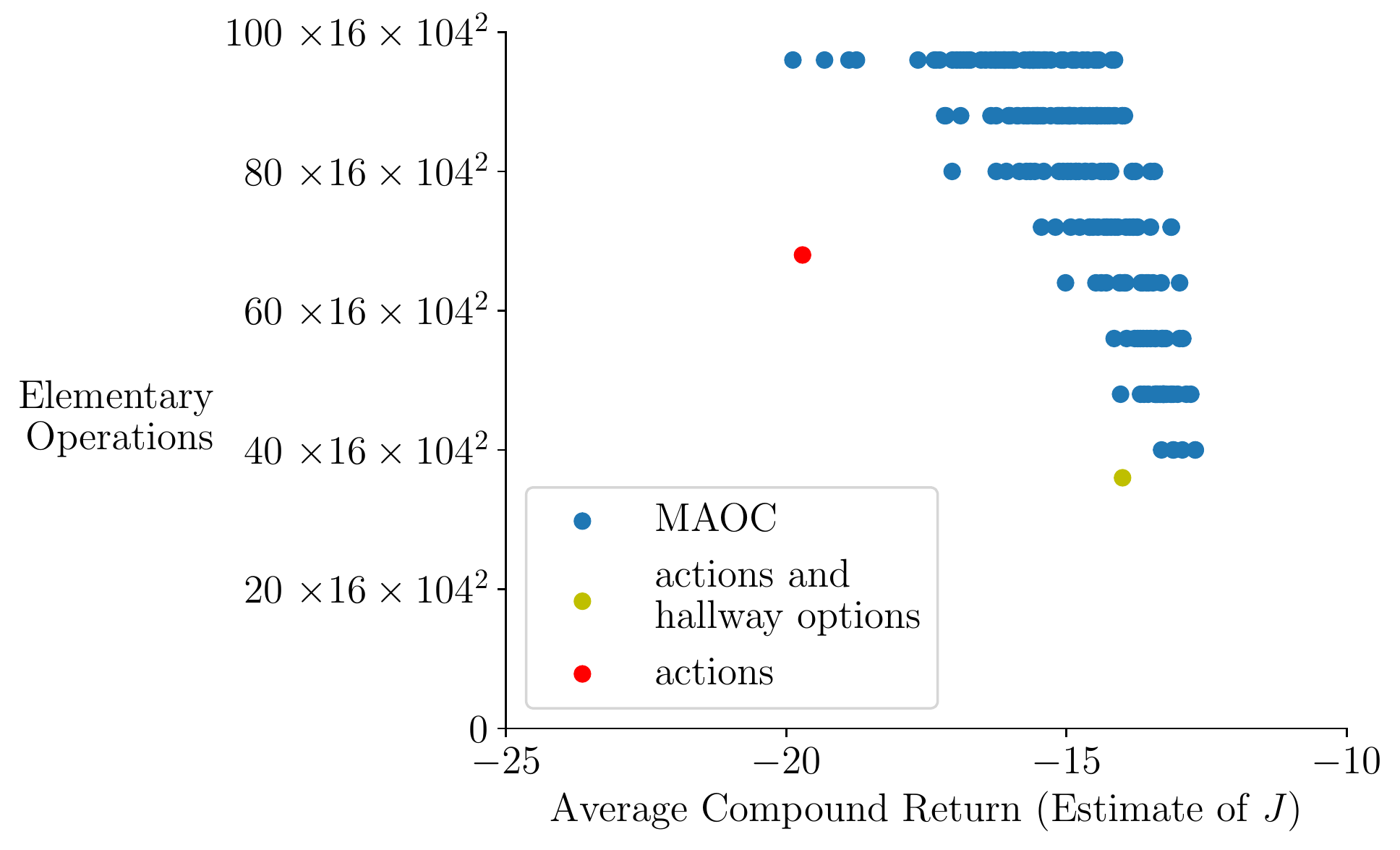}
    \caption{$k = 4$, Testing Tasks}
\end{subfigure}
\begin{subfigure}{0.5\textwidth}
    \centering
    \includegraphics[width=\textwidth]{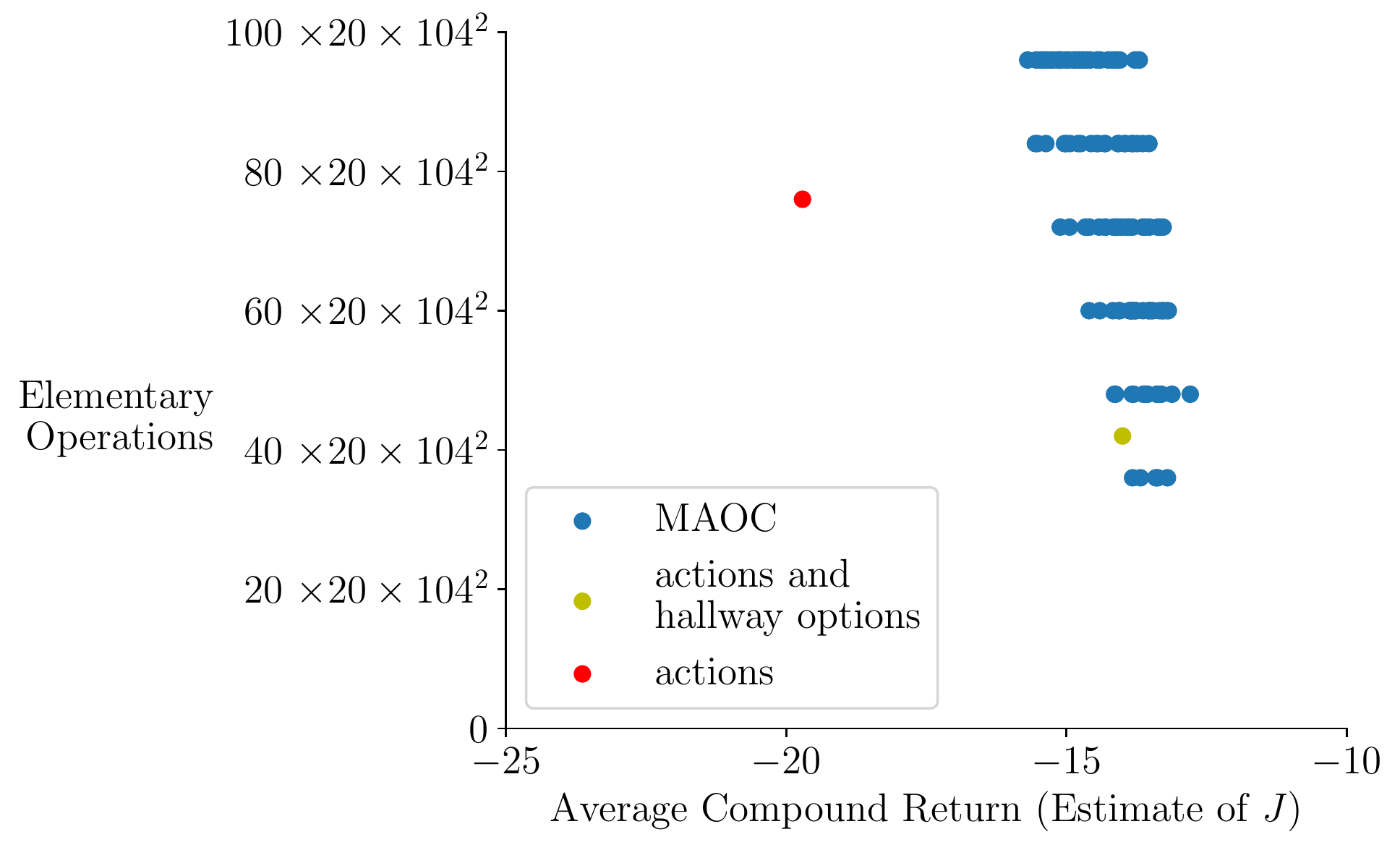}
    \caption{$k = 8$, Training Tasks}
\end{subfigure}%
\begin{subfigure}{0.5\textwidth}
    \centering
    \includegraphics[width=\textwidth]{figures/MAOC_planning_test_8_adj_options.pdf}
    \caption{$k = 8$, Testing Tasks}
\end{subfigure}%
\caption{The number of elementary operations used to achieve near-optimal performance in the set of \emph{training} and the set of \emph{testing} tasks by the MAOC algorithm.
}
\label{fig: MAOC study planning iterations vs estimate of J}
\end{figure*}

\begin{figure*}[h]
\begin{subfigure}{0.5\textwidth}
    \centering
    \includegraphics[width=\textwidth]{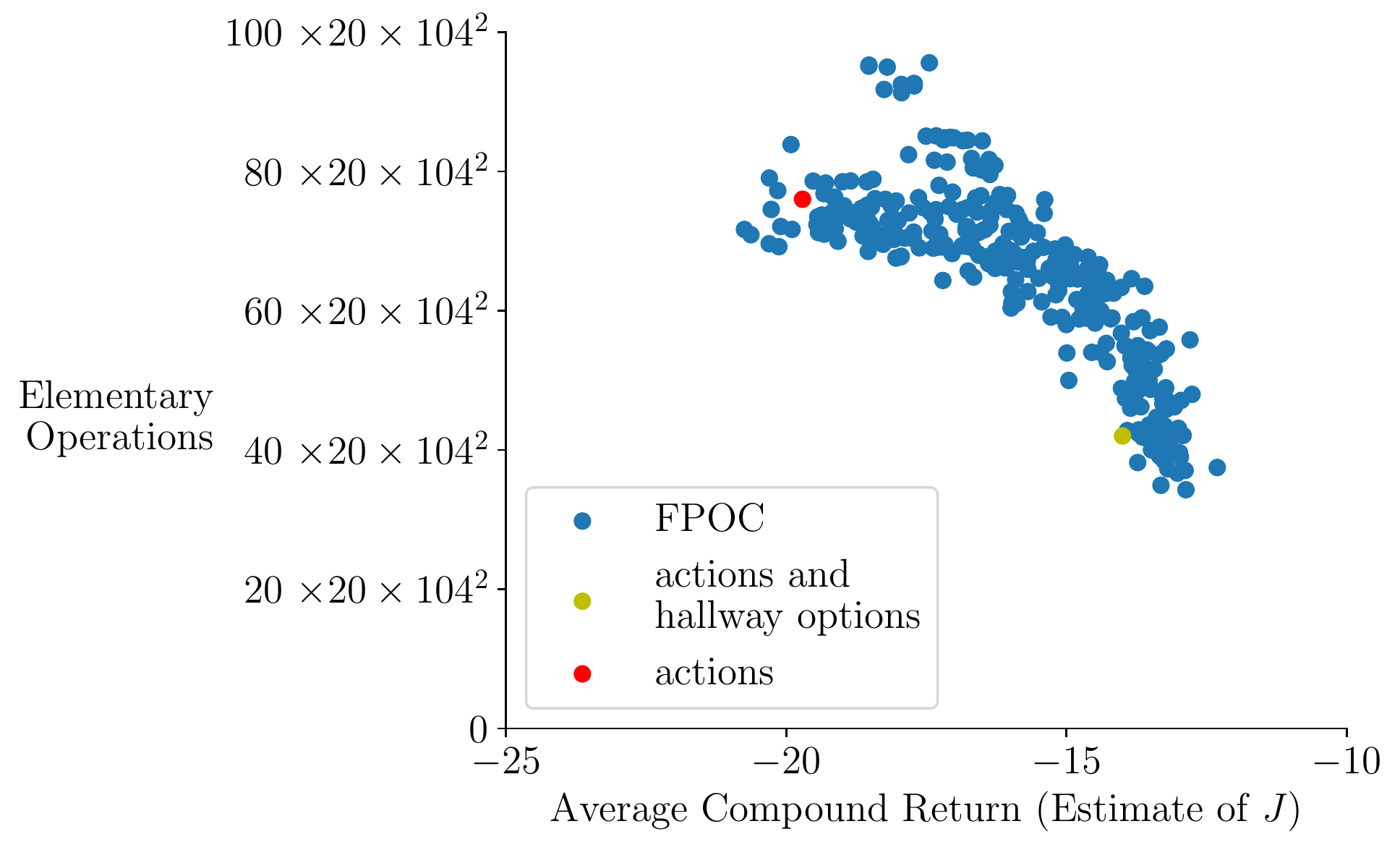}
    \caption{$k = 2$, Training Tasks}
\end{subfigure}%
\begin{subfigure}{0.5\textwidth}
    \centering
    \includegraphics[width=\textwidth]{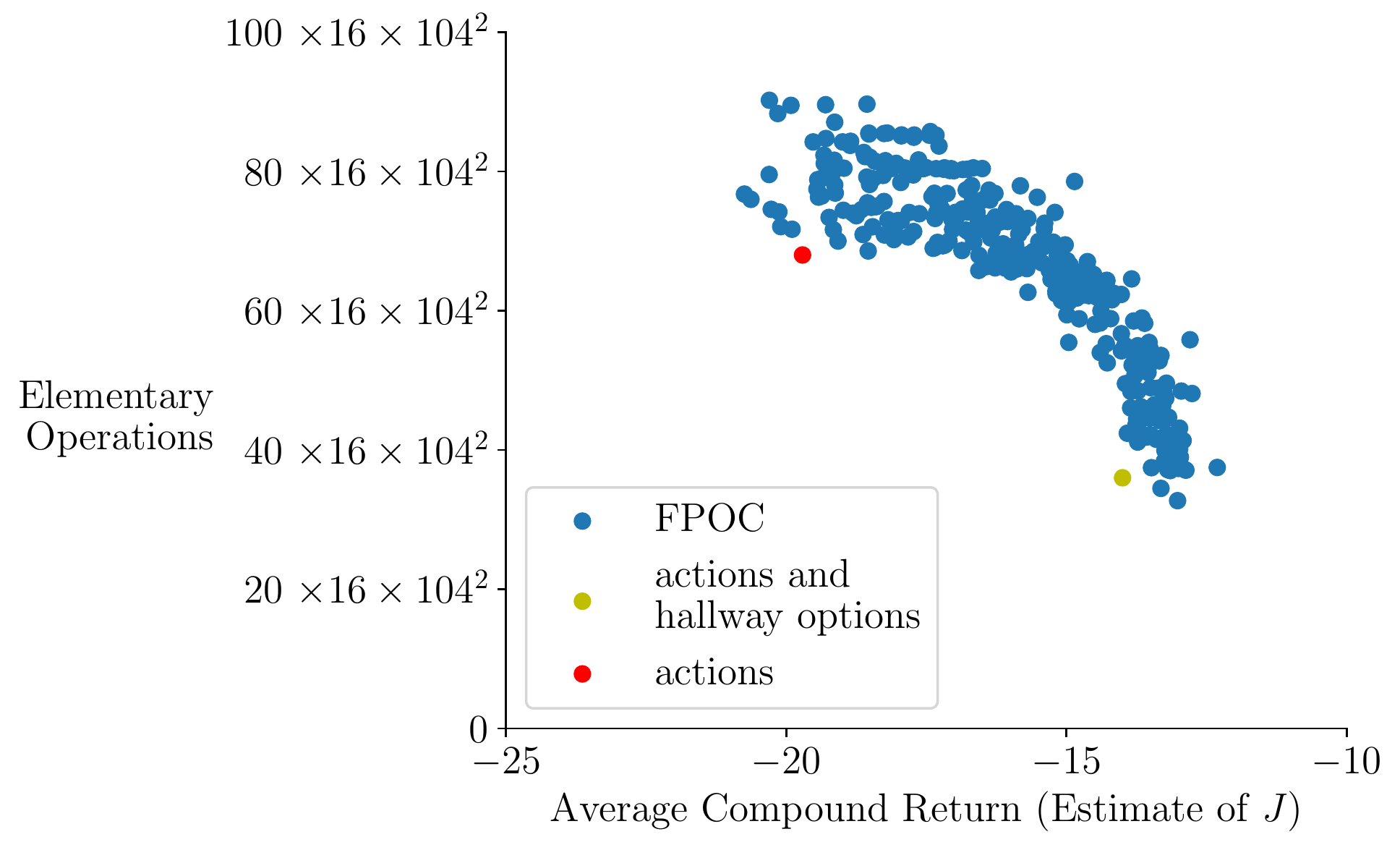}
    \caption{$k = 2$, Testing Tasks}
\end{subfigure}
\begin{subfigure}{0.5\textwidth}
    \centering
    \includegraphics[width=\textwidth]{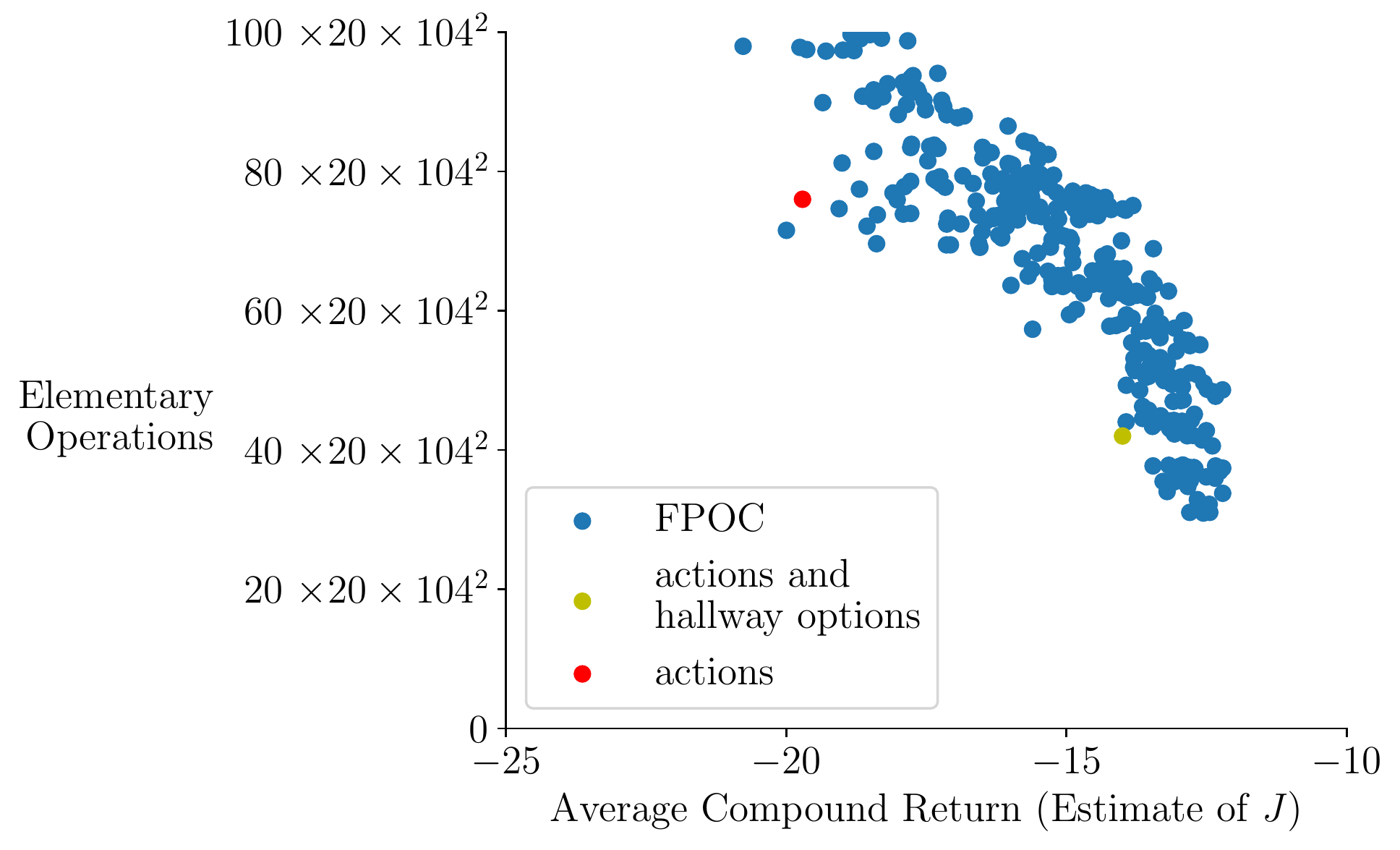}
    \caption{$k = 4$, Training Tasks}
\end{subfigure}%
\begin{subfigure}{0.5\textwidth}
    \centering
    \includegraphics[width=\textwidth]{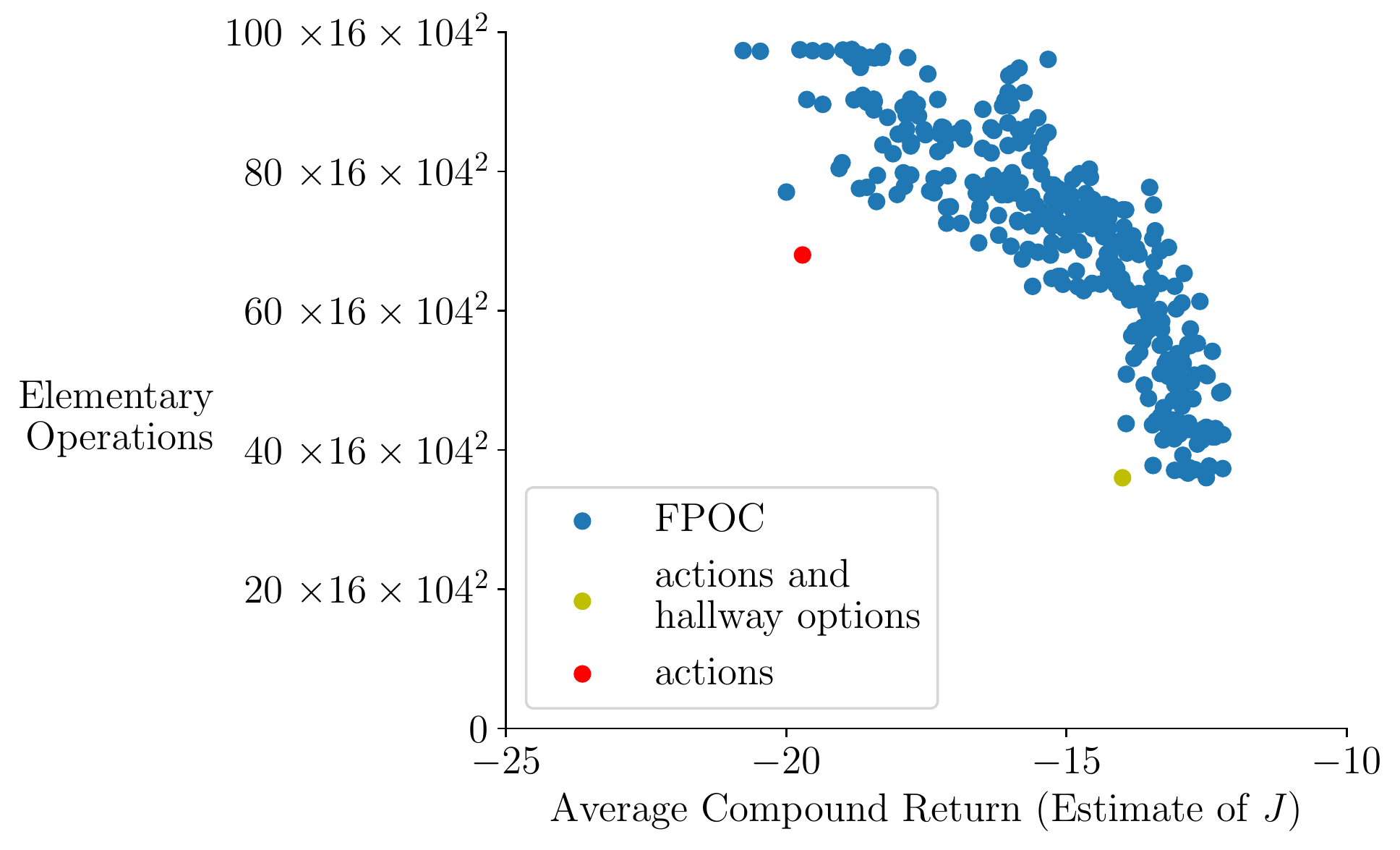}
    \caption{$k = 4$, Testing Tasks}
\end{subfigure}
\begin{subfigure}{0.5\textwidth}
    \centering
    \includegraphics[width=\textwidth]{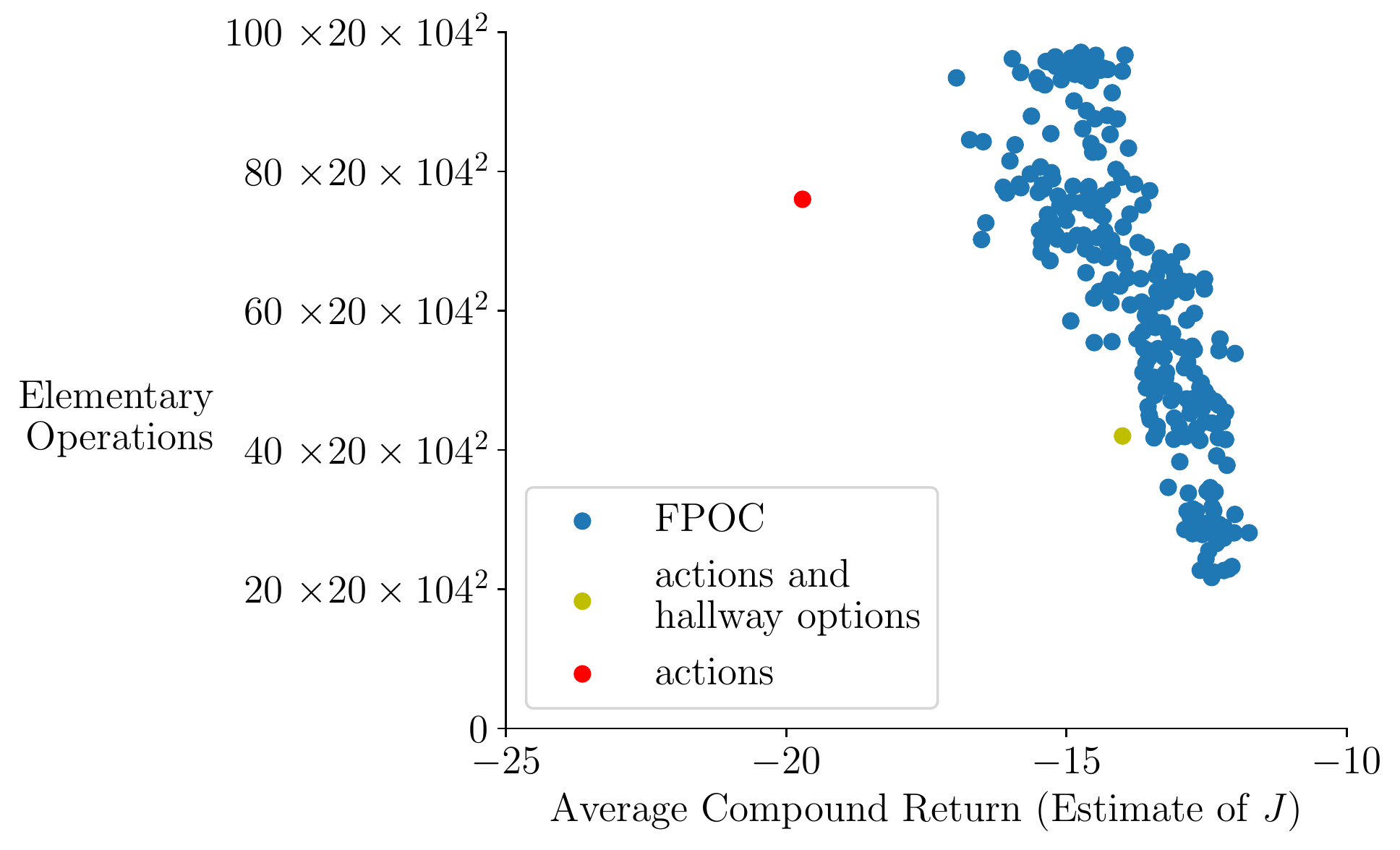}
    \caption{$k = 8$, Training Tasks}
\end{subfigure}%
\begin{subfigure}{0.5\textwidth}
    \centering
    \includegraphics[width=\textwidth]{figures/FPOC_planning_test_8_adj_options.pdf}
    \caption{$k = 8$, Testing Tasks}
\end{subfigure}%
\caption{The number of elementary operations used to achieve near-optimal performance in the set of \emph{training} and the set of \emph{testing} tasks by the FPOC algorithm. 
}
\label{fig: FPOC study planning iterations vs estimate of J}
\end{figure*}

\subsubsection{Learned Options with Different Choices of $c$}

Remember that $c$ is a problem parameter while $\bar c$ is a solution parameter. It is interesting to know what the discovered options are when we vary $c$. From \cref{fig: The learned options using MAOC with 2 adjustable options best c=0} to \cref{fig: The learned options using FPOC with 8 adjustable options best c=1}, we show the the discovered options when $c = 0, 0.4$ and $0.8$ for both MAOC and FPOC.

\begin{figure*}[h]
\begin{subfigure}{0.25\textwidth}
    \centering
    \includegraphics[width=\textwidth]{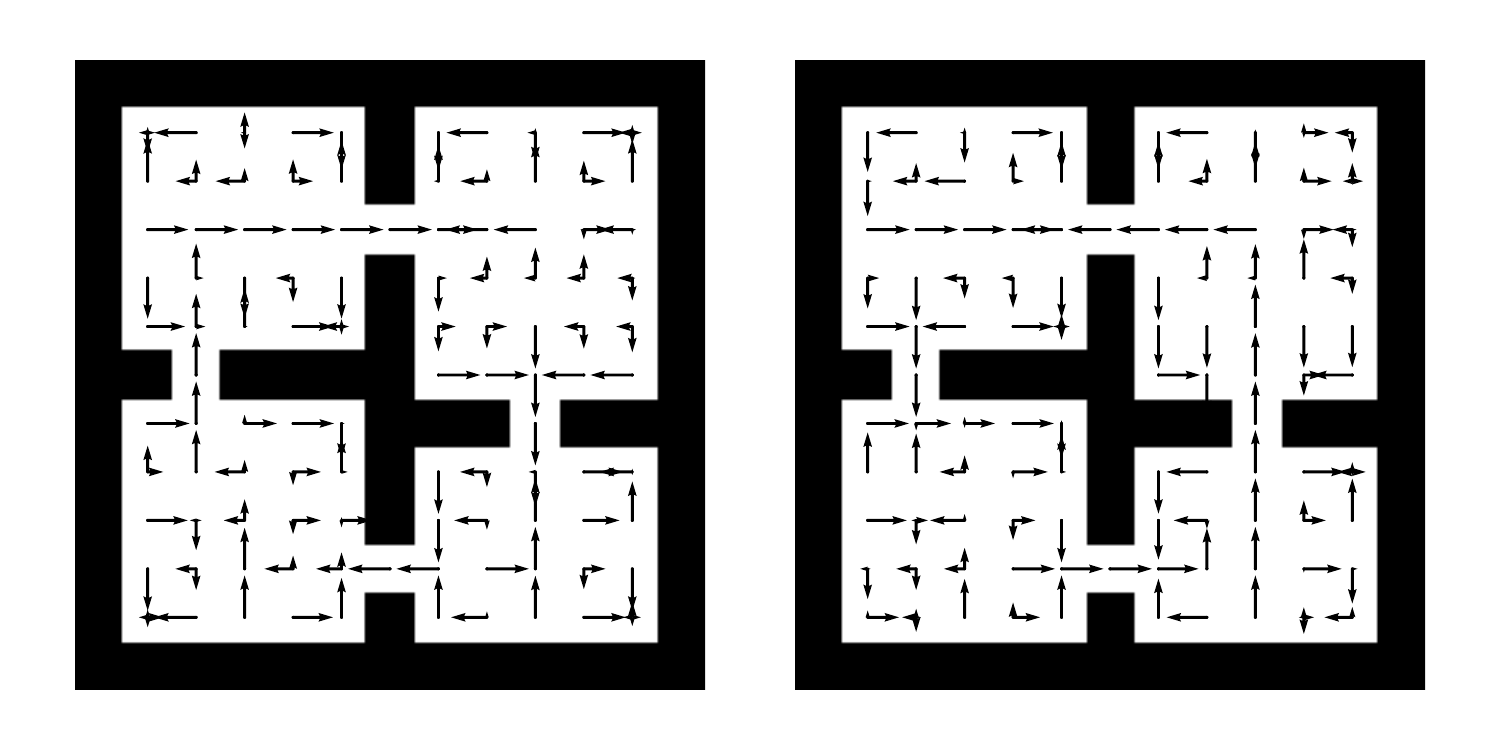}
    \caption{$\pi$}
\end{subfigure}%
\begin{subfigure}{0.25\textwidth}
    \centering
    \includegraphics[width=\textwidth]{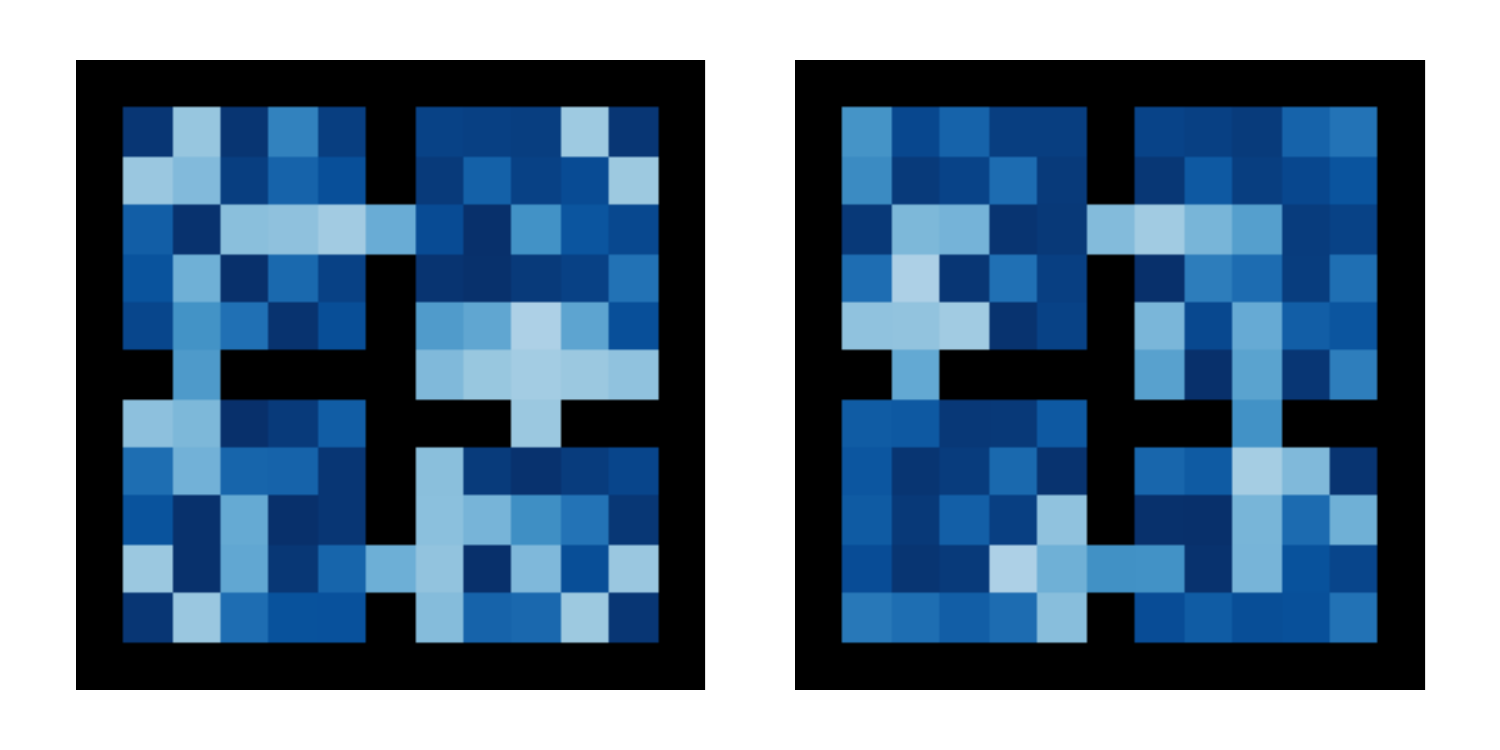}
    \caption{$\beta$}
\end{subfigure}%
\begin{subfigure}{0.25\textwidth}
    \centering
    \includegraphics[width=\textwidth]{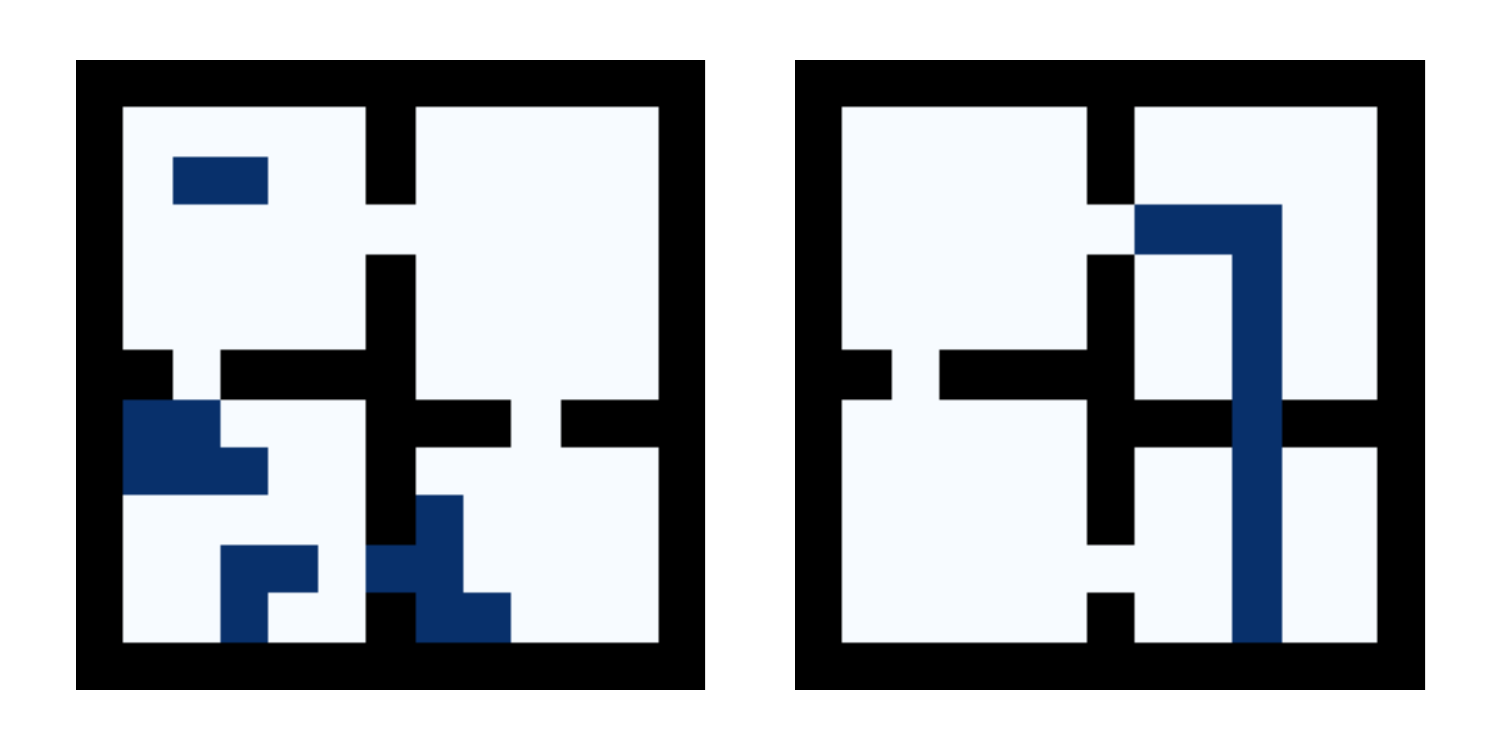}
    \caption{$\mu$ (task 1) }
\end{subfigure}%
\begin{subfigure}{0.25\textwidth}
    \centering
    \includegraphics[width=\textwidth]{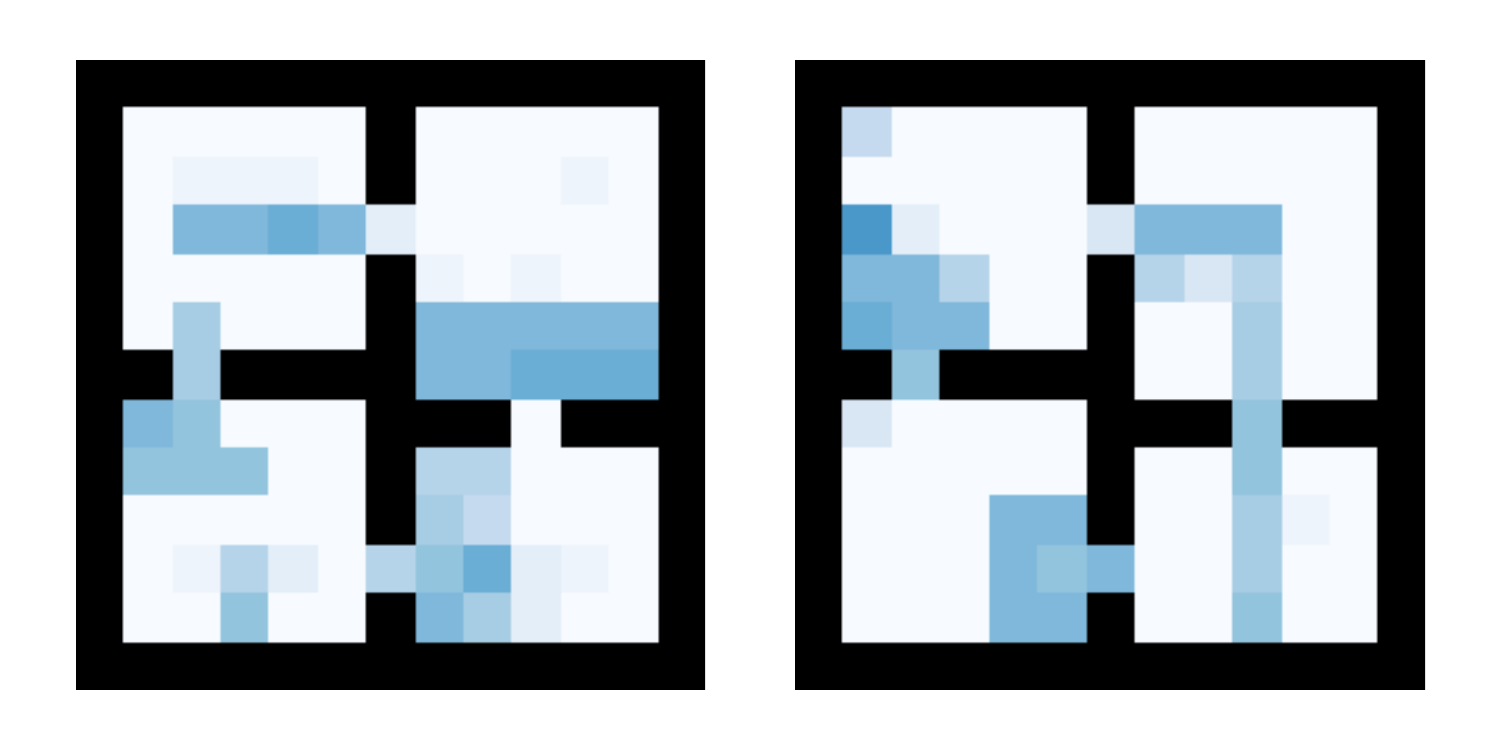}
    \caption{$\mu$ (all tasks)}
\end{subfigure}
\caption{The learned options using MAOC with 2 adjustable options ($c = 0$, best $\bar c$ and $\eta$). We plot the options' policies in (a) and the options' termination probabilities in (b). In (c), we plot the meta-policy for training task 1, which is to move to the cell in the upper left corner. In (d), we plot the meta-policy averaged over all training tasks. This plot shows whether an option is used frequently to solve training tasks.
}
\label{fig: The learned options using MAOC with 2 adjustable options best c=0}
\end{figure*}

\begin{figure*}[h]
\begin{subfigure}{0.25\textwidth}
    \centering
    \includegraphics[width=\textwidth]{figures/921_pi.pdf}
    \caption{$\pi$}
\end{subfigure}%
\begin{subfigure}{0.25\textwidth}
    \centering
    \includegraphics[width=\textwidth]{figures/921_beta.pdf}
    \caption{$\beta$}
\end{subfigure}%
\begin{subfigure}{0.25\textwidth}
    \centering
    \includegraphics[width=\textwidth]{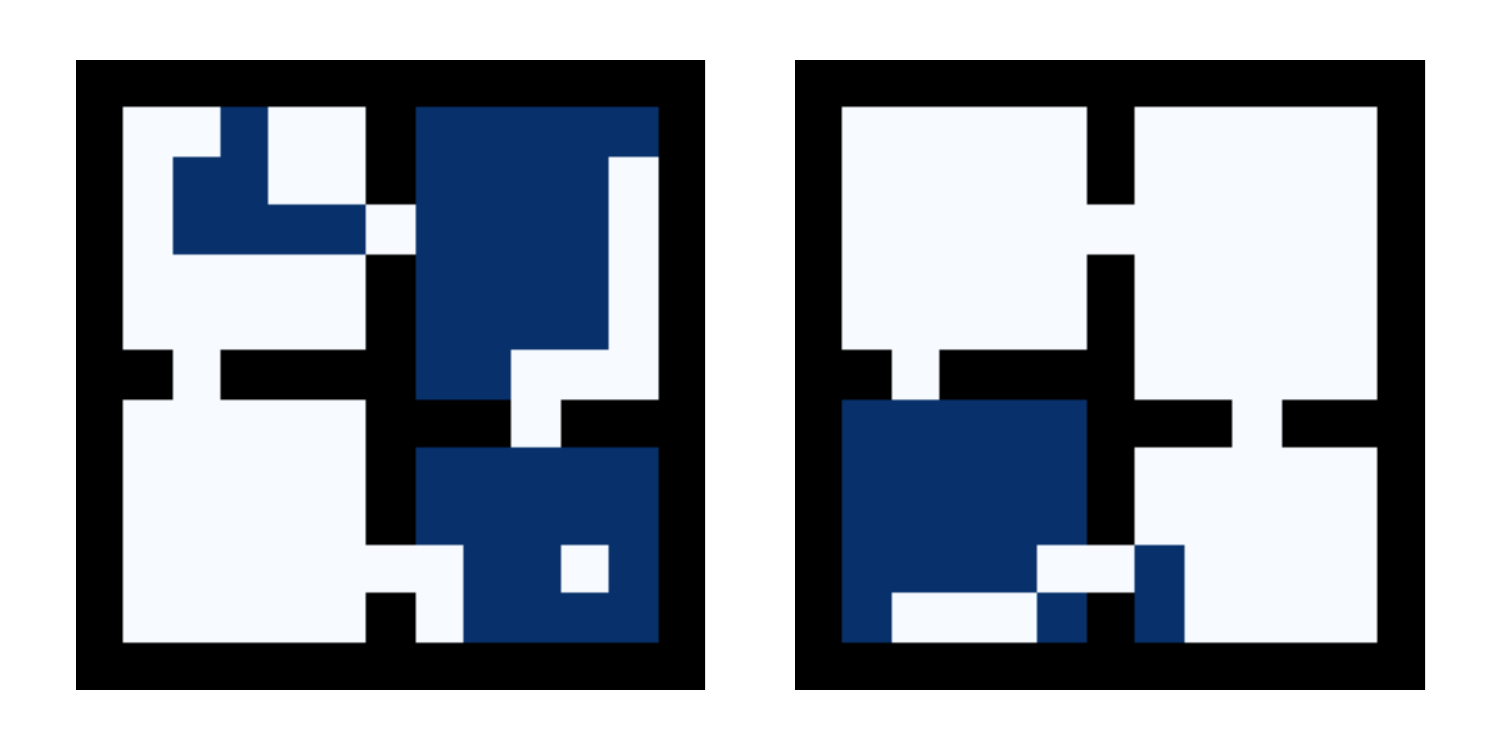}
    \caption{$\mu$ (task 1) }
\end{subfigure}%
\begin{subfigure}{0.25\textwidth}
    \centering
    \includegraphics[width=\textwidth]{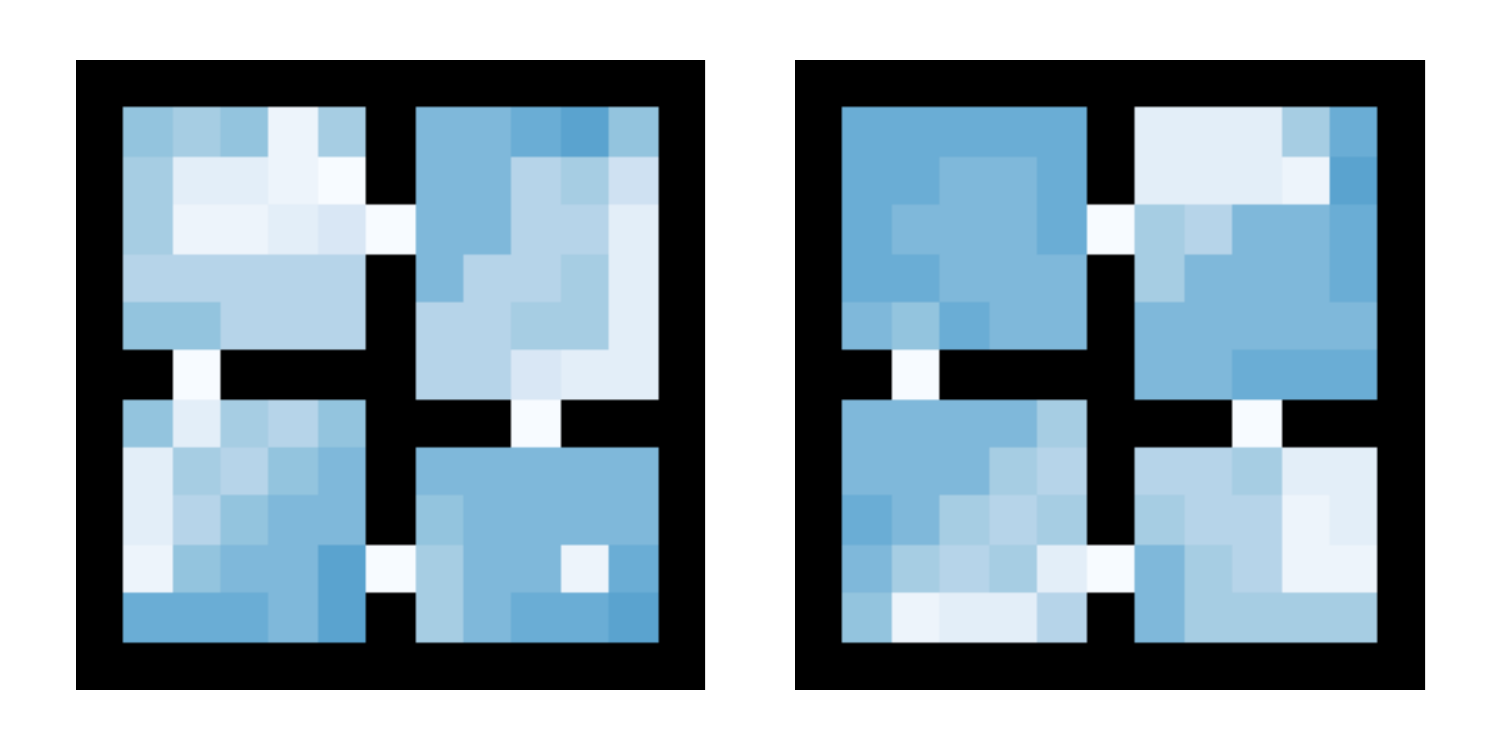}
    \caption{$\mu$ (all tasks)}
\end{subfigure}
\caption{The learned options using MAOC with 2 adjustable options ($c = 0.2$, best $\bar c$ and $\eta$).
}
\label{fig: The learned options using MAOC with 2 adjustable options best c=0.2}
\end{figure*}

\begin{figure*}[h]
\begin{subfigure}{0.25\textwidth}
    \centering
    \includegraphics[width=\textwidth]{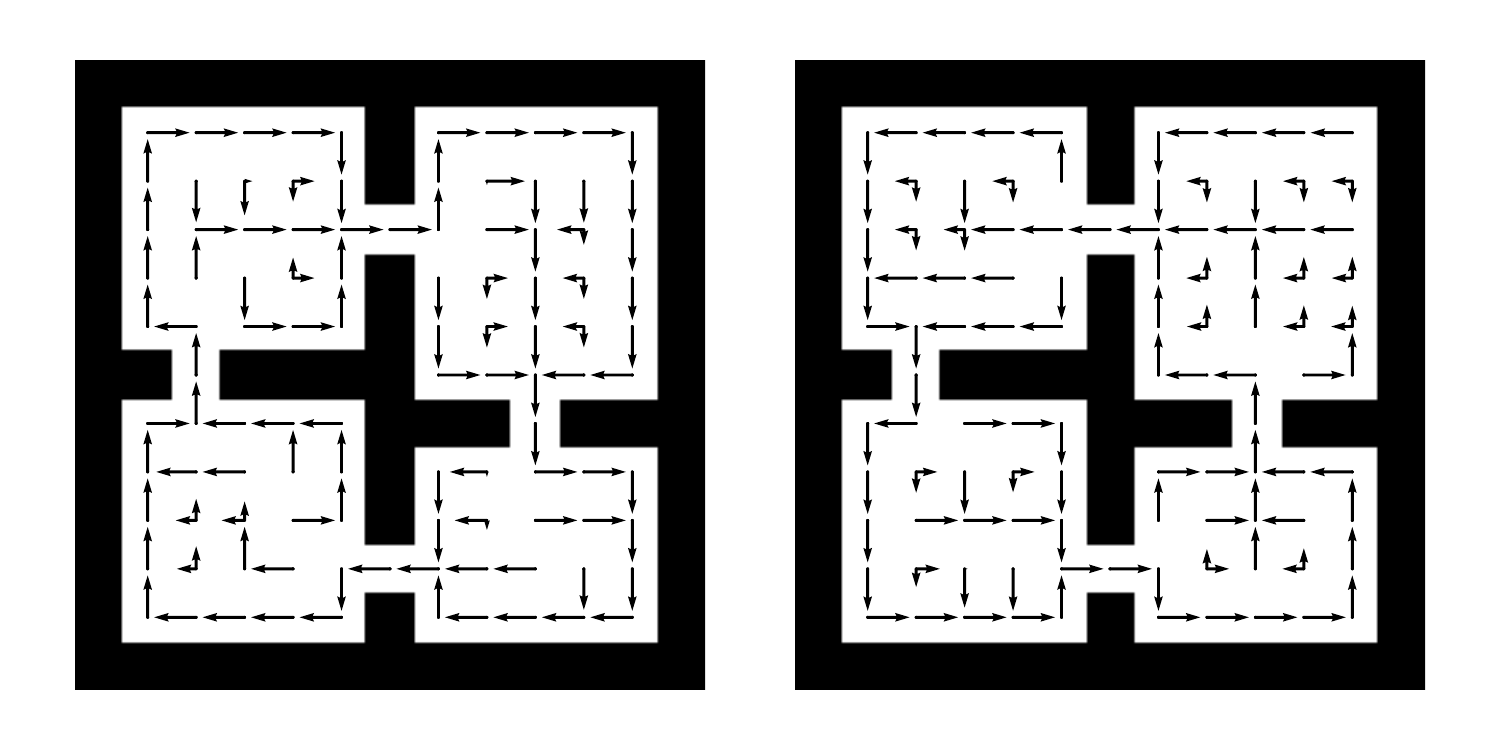}
    \caption{$\pi$}
\end{subfigure}%
\begin{subfigure}{0.25\textwidth}
    \centering
    \includegraphics[width=\textwidth]{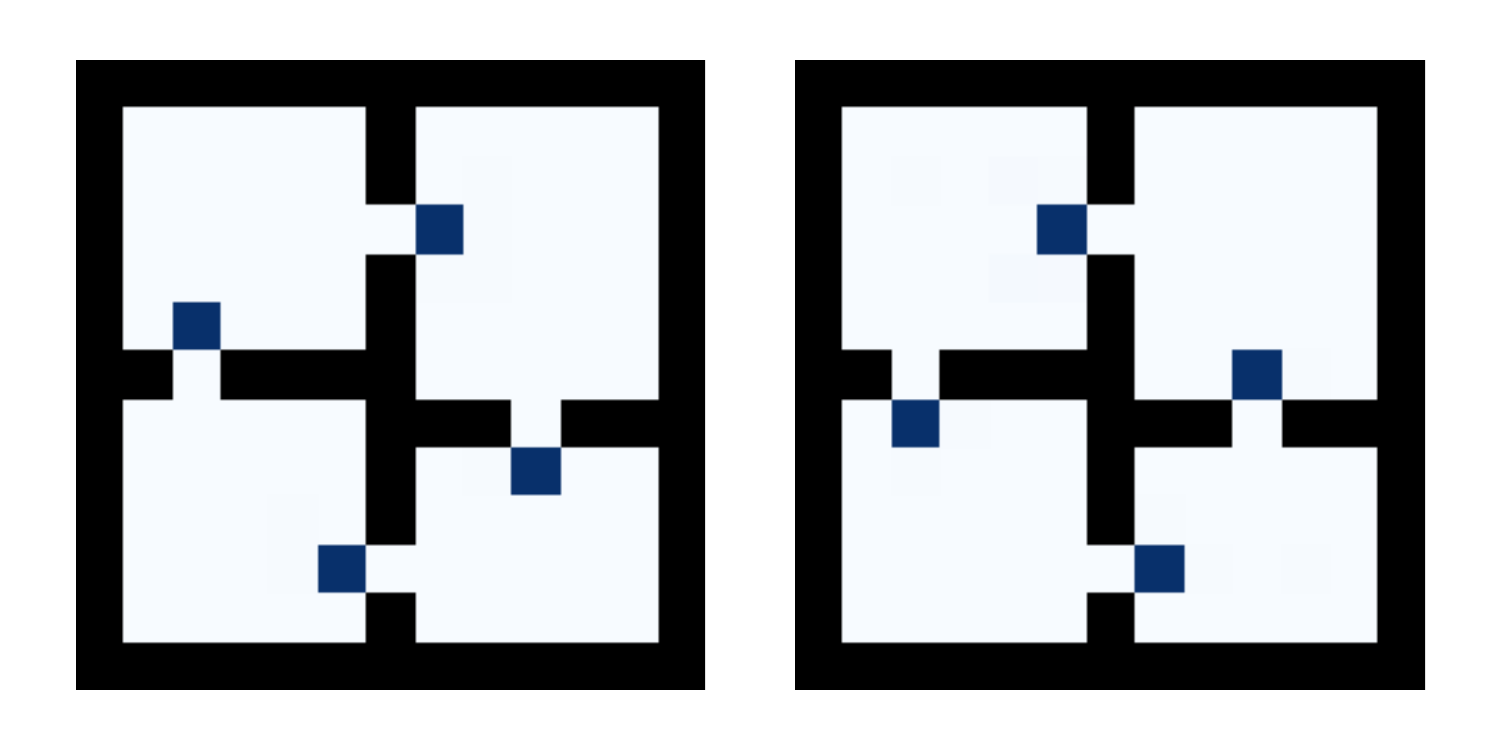}
    \caption{$\beta$}
\end{subfigure}%
\begin{subfigure}{0.25\textwidth}
    \centering
    \includegraphics[width=\textwidth]{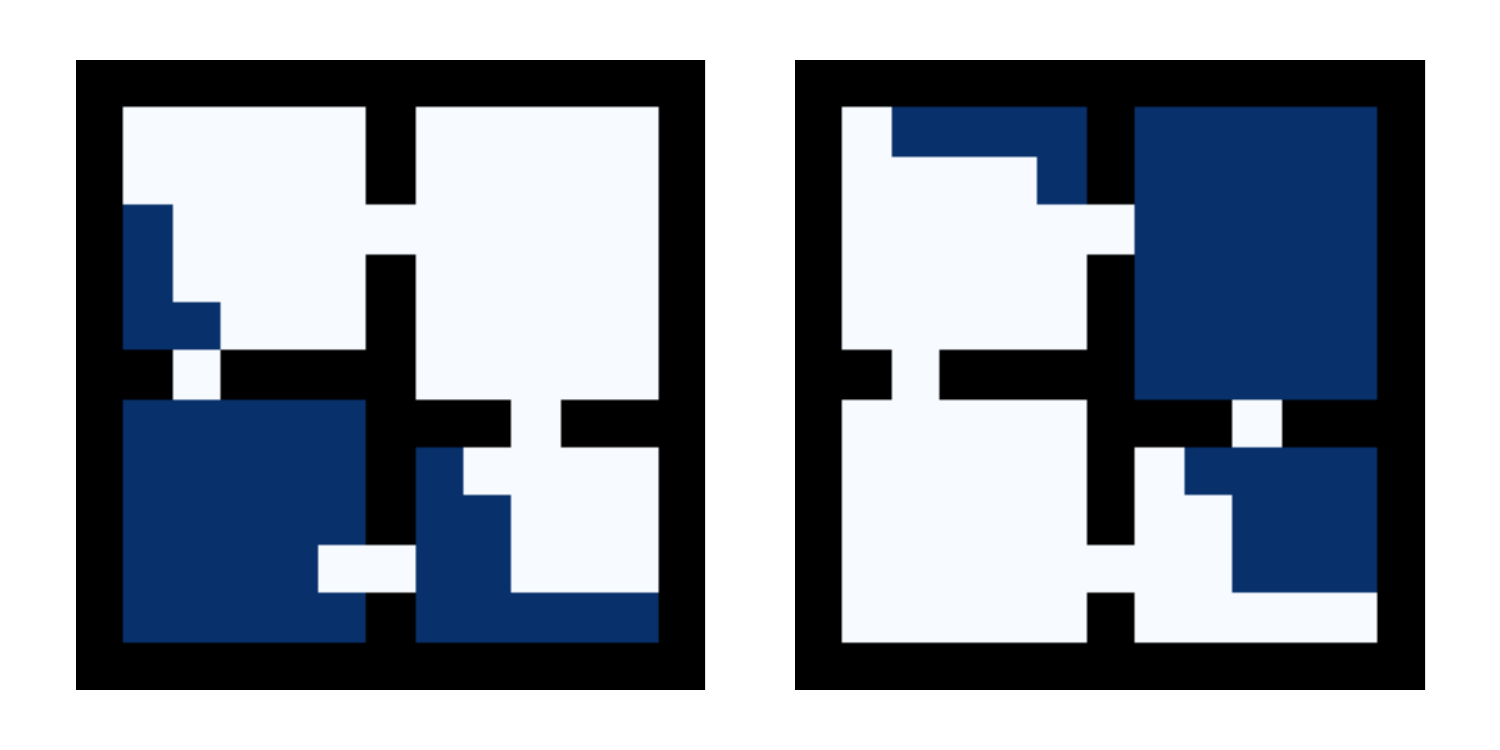}
    \caption{$\mu$ (task 1) }
\end{subfigure}%
\begin{subfigure}{0.25\textwidth}
    \centering
    \includegraphics[width=\textwidth]{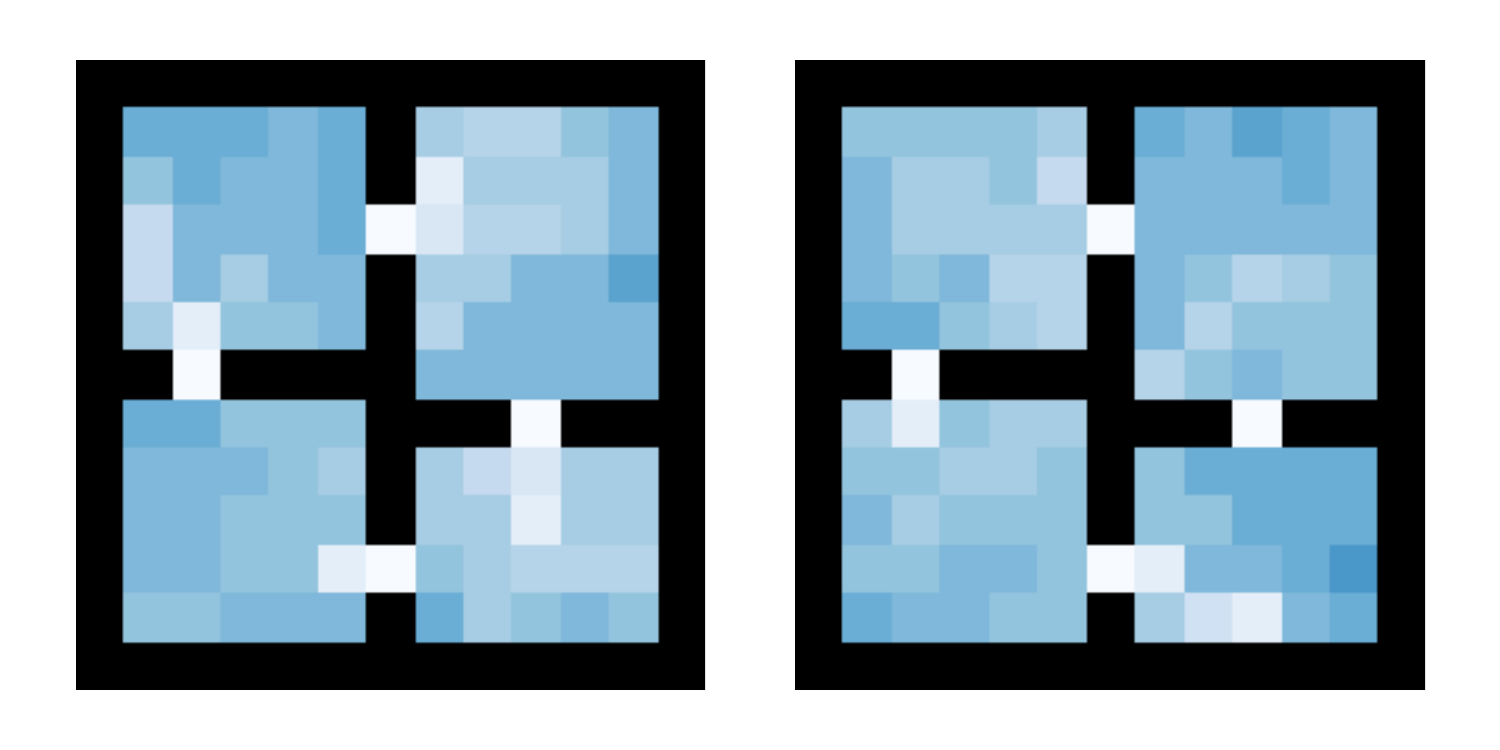}
    \caption{$\mu$ (all tasks)}
\end{subfigure}
\caption{The learned options using MAOC with 2 adjustable options ($c = 0.6$, best $\bar c$ and $\eta$).
}
\label{fig: The learned options using MAOC with 2 adjustable options best c=0.6}
\end{figure*}

\begin{figure*}[h]
\begin{subfigure}{0.25\textwidth}
    \centering
    \includegraphics[width=\textwidth]{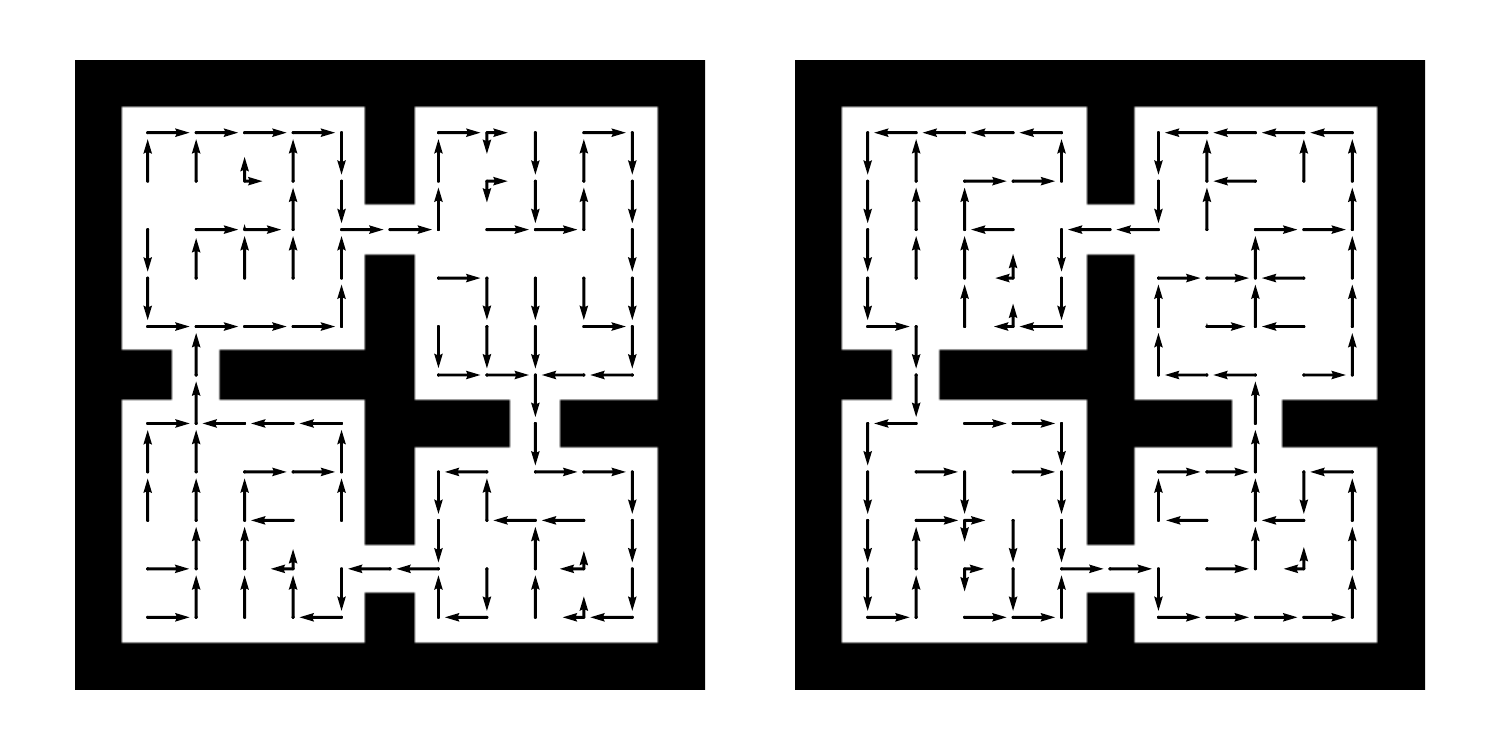}
    \caption{$\pi$}
\end{subfigure}%
\begin{subfigure}{0.25\textwidth}
    \centering
    \includegraphics[width=\textwidth]{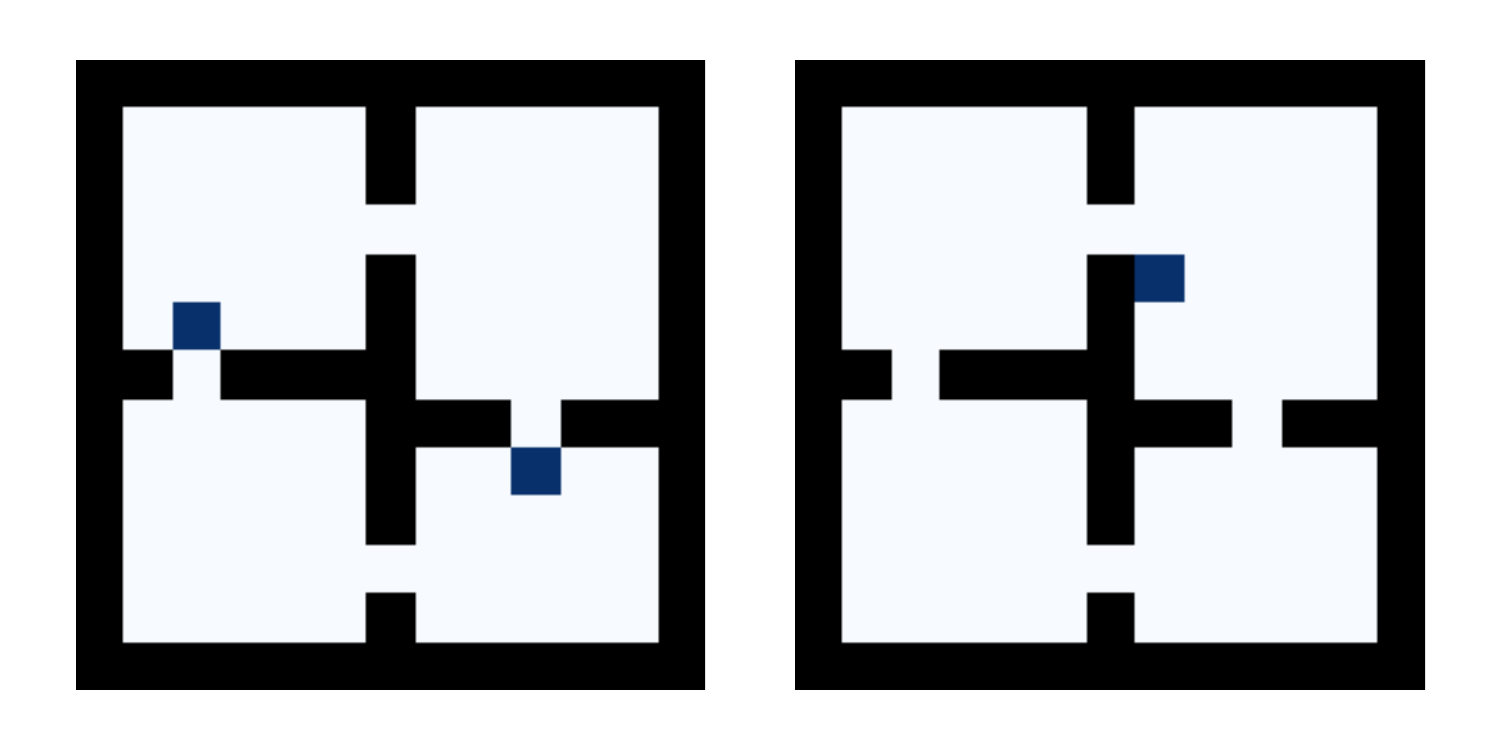}
    \caption{$\beta$}
\end{subfigure}%
\begin{subfigure}{0.25\textwidth}
    \centering
    \includegraphics[width=\textwidth]{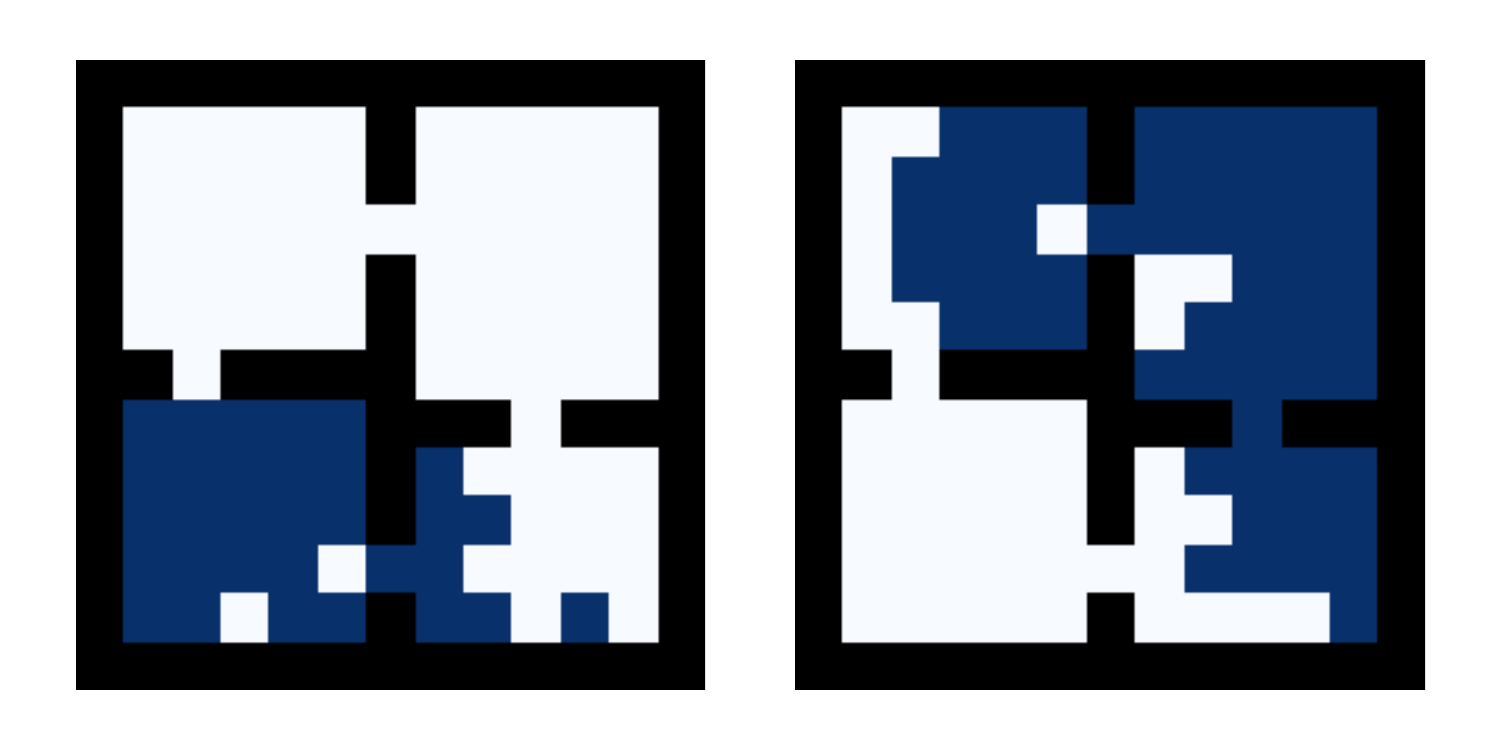}
    \caption{$\mu$ (task 1) }
\end{subfigure}%
\begin{subfigure}{0.25\textwidth}
    \centering
    \includegraphics[width=\textwidth]{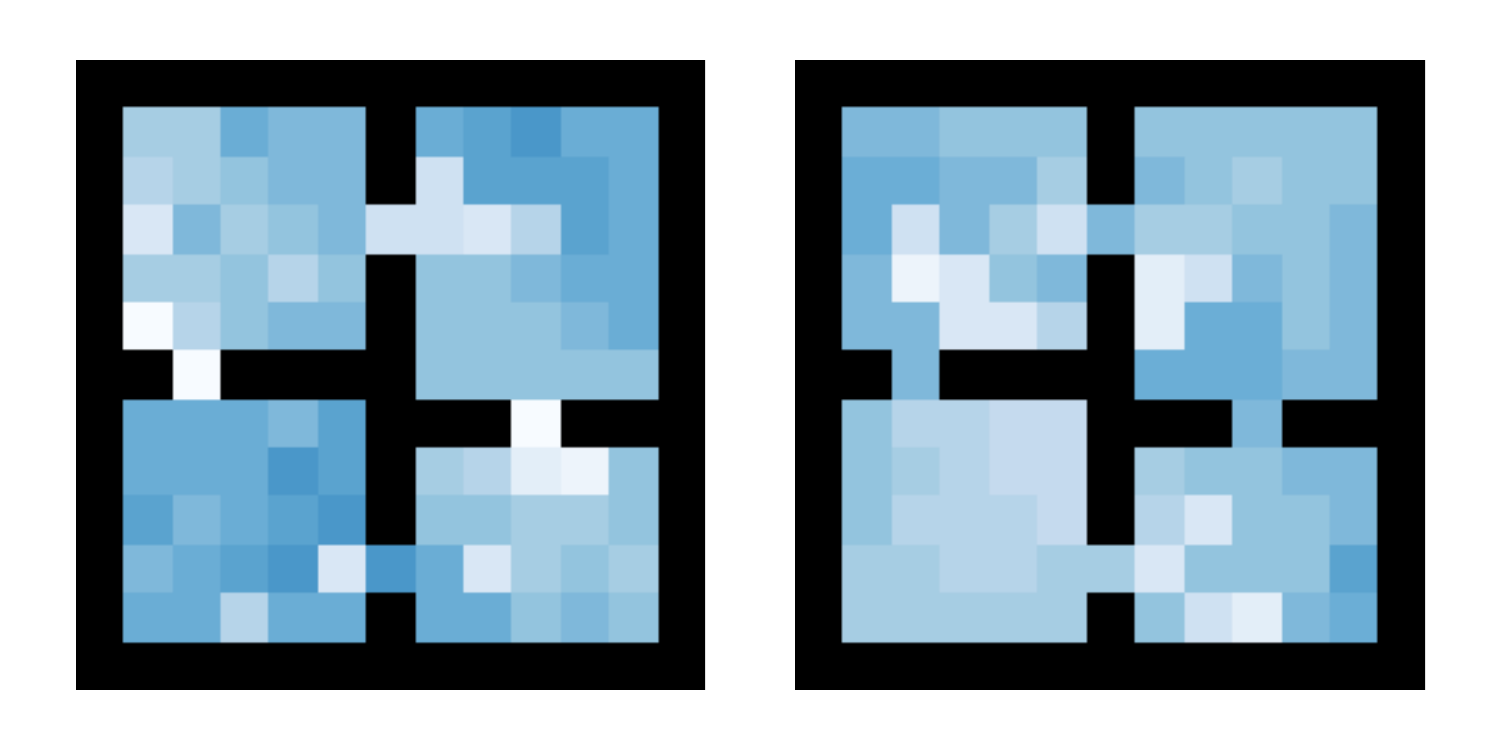}
    \caption{$\mu$ (all tasks)}
\end{subfigure}
\caption{The learned options using MAOC with 2 adjustable options ($c = 1$, best $\bar c$ and $\eta$).
}
\label{fig: The learned options using MAOC with 2 adjustable options best c=0.6}
\end{figure*}

\begin{figure*}[h]
\begin{subfigure}{0.5\textwidth}
    \centering
    \includegraphics[width=\textwidth]{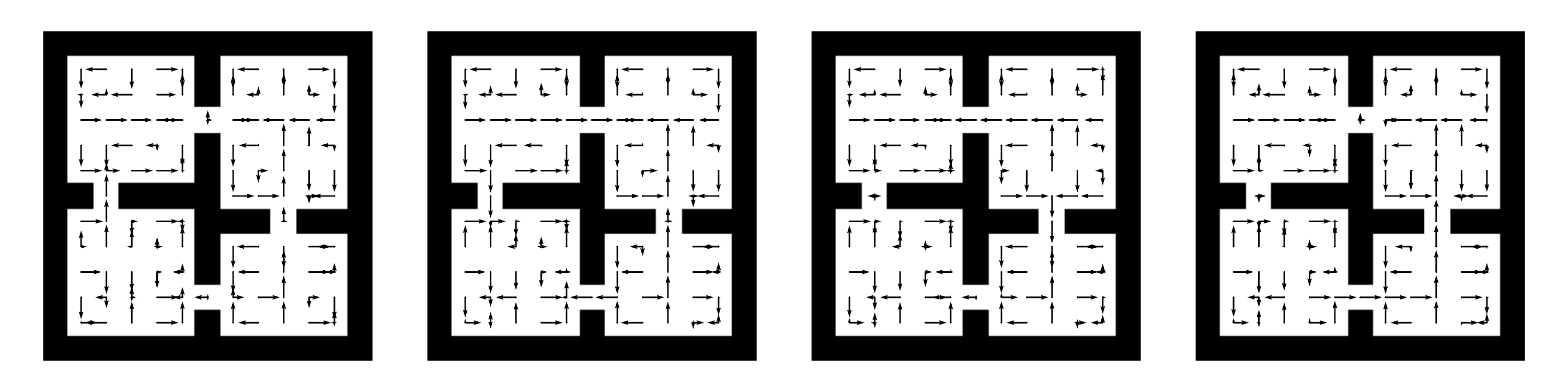}
    \caption{$\pi$}
\end{subfigure}%
\begin{subfigure}{0.5\textwidth}
    \centering
    \includegraphics[width=\textwidth]{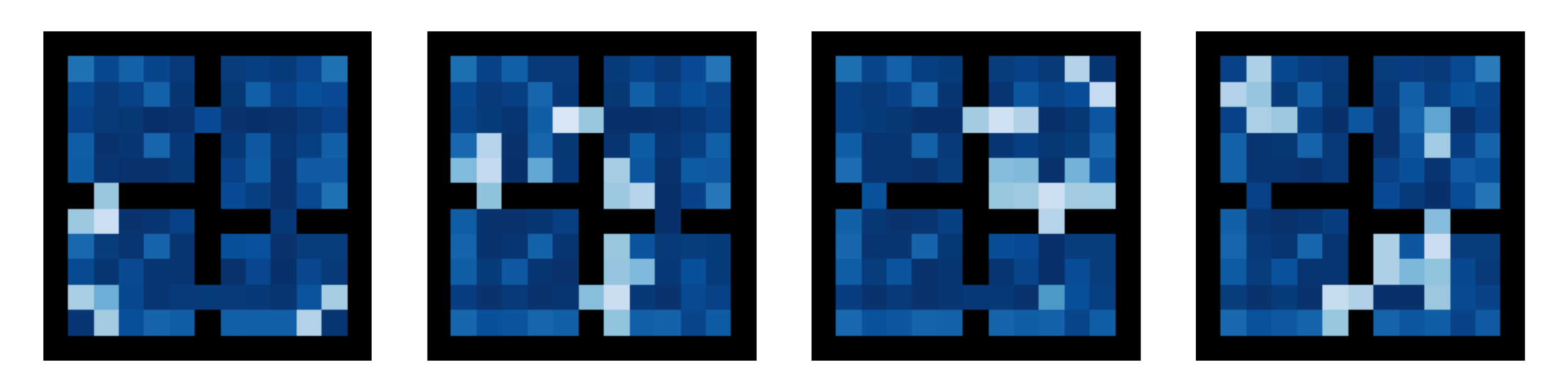}
    \caption{$\beta$}
\end{subfigure}
\begin{subfigure}{0.5\textwidth}
    \centering
    \includegraphics[width=\textwidth]{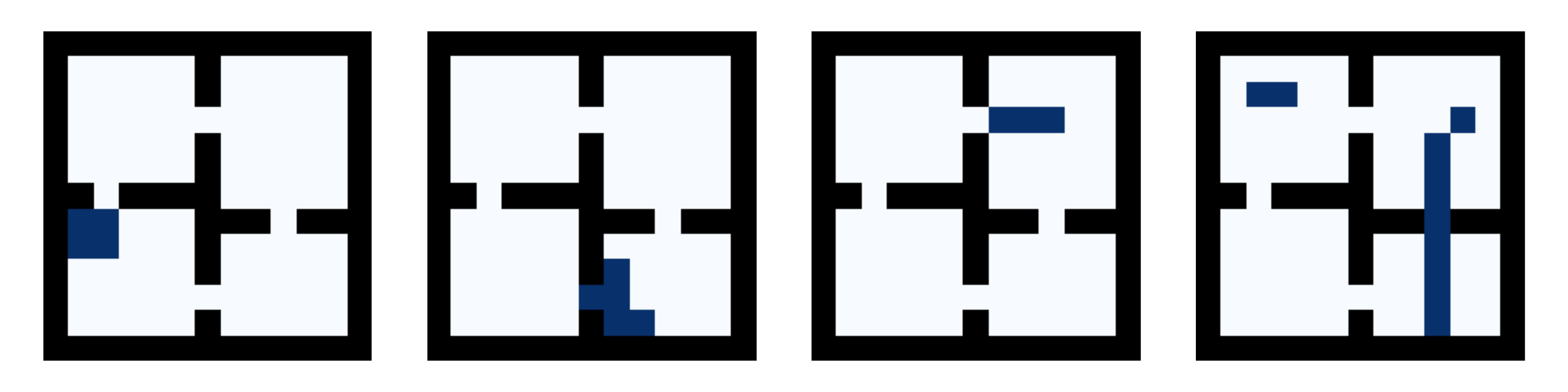}
    \caption{$\mu$ (task 1) }
\end{subfigure}%
\begin{subfigure}{0.5\textwidth}
    \centering
    \includegraphics[width=\textwidth]{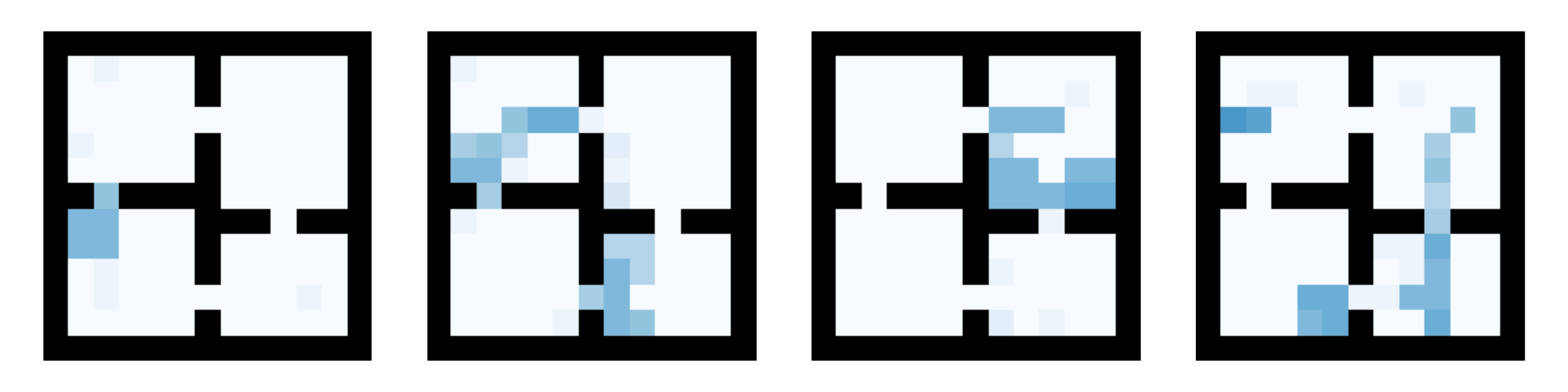}
    \caption{$\mu$ (all tasks) }
\end{subfigure}
\caption{The learned options using MAOC with 4 adjustable options ($c = 0$, best $\bar c$ and $\eta$).
}
\label{fig: The learned options using MAOC with 4 adjustable options best c=0}
\end{figure*}

\begin{figure*}[h]
\begin{subfigure}{0.5\textwidth}
    \centering
    \includegraphics[width=\textwidth]{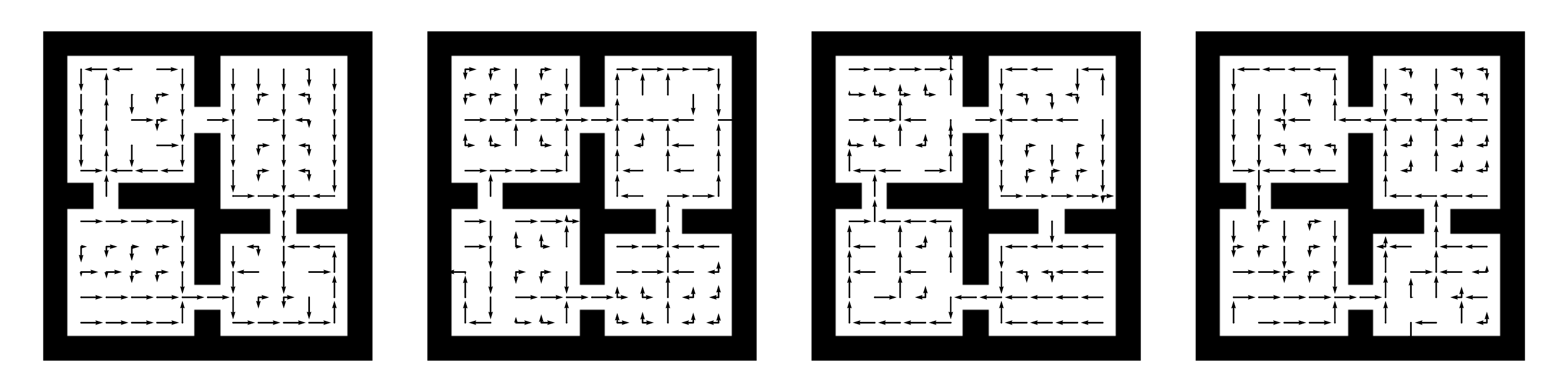}
    \caption{$\pi$}
\end{subfigure}%
\begin{subfigure}{0.5\textwidth}
    \centering
    \includegraphics[width=\textwidth]{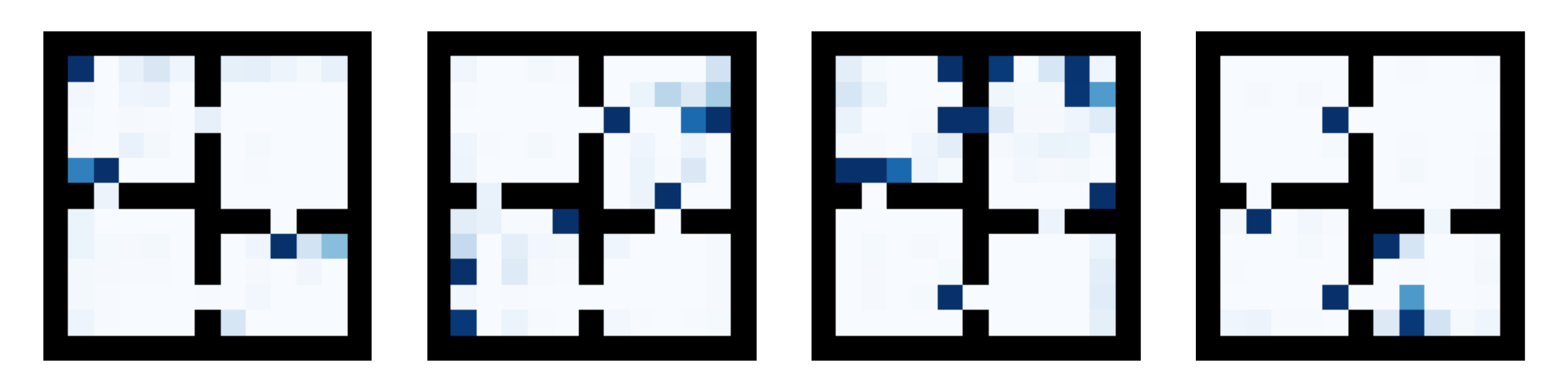}
    \caption{$\beta$}
\end{subfigure}
\begin{subfigure}{0.5\textwidth}
    \centering
    \includegraphics[width=\textwidth]{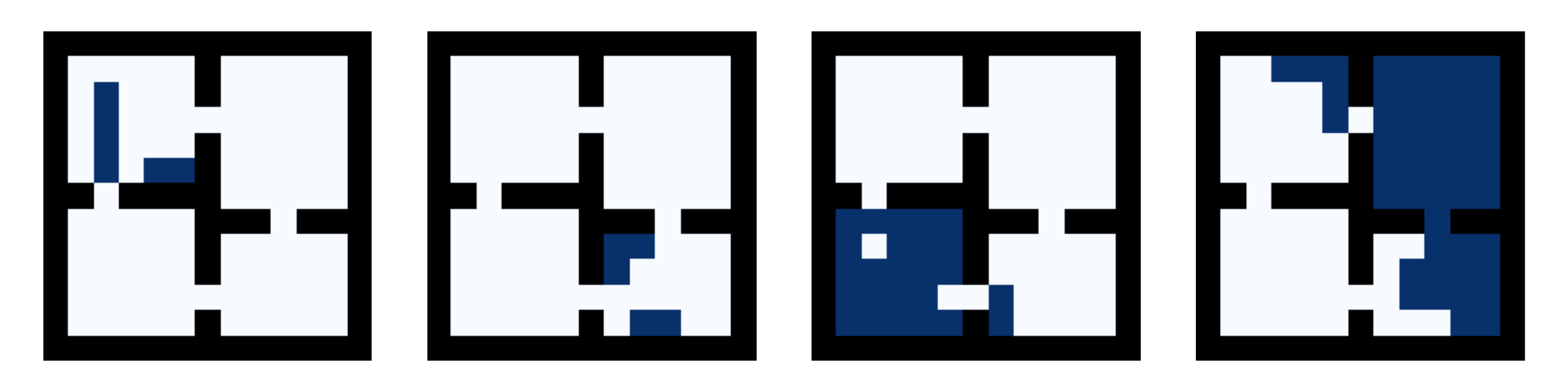}
    \caption{$\mu$ (task 1) }
\end{subfigure}%
\begin{subfigure}{0.5\textwidth}
    \centering
    \includegraphics[width=\textwidth]{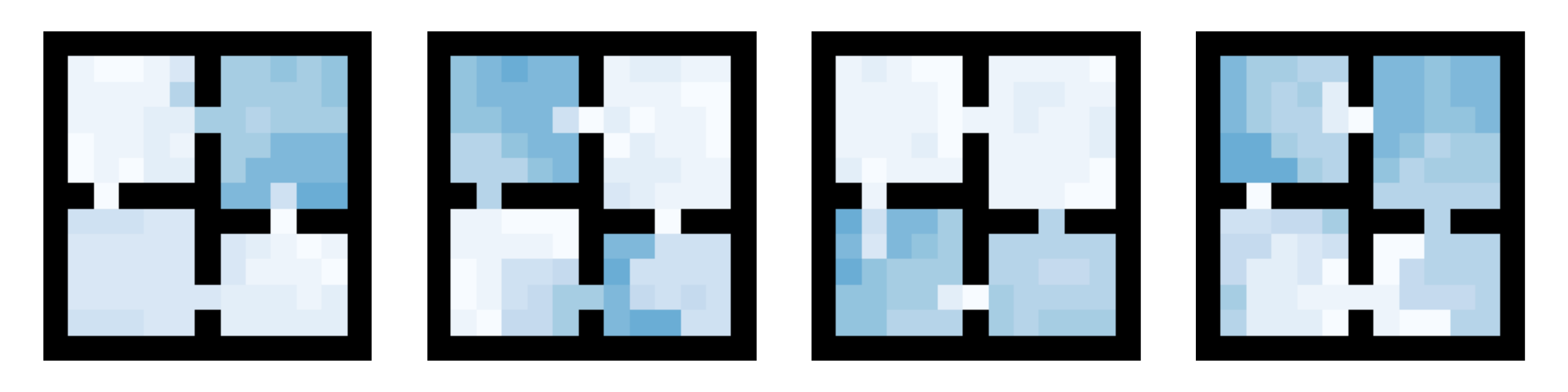}
    \caption{$\mu$ (all tasks) }
\end{subfigure}
\caption{The learned options using MAOC with 4 adjustable options ($c = 0.2$, best $\bar c$ and $\eta$).
}
\label{fig: The learned options using MAOC with 4 adjustable options best c=0.2}
\end{figure*}

\begin{figure*}[h]
\begin{subfigure}{0.5\textwidth}
    \centering
    \includegraphics[width=\textwidth]{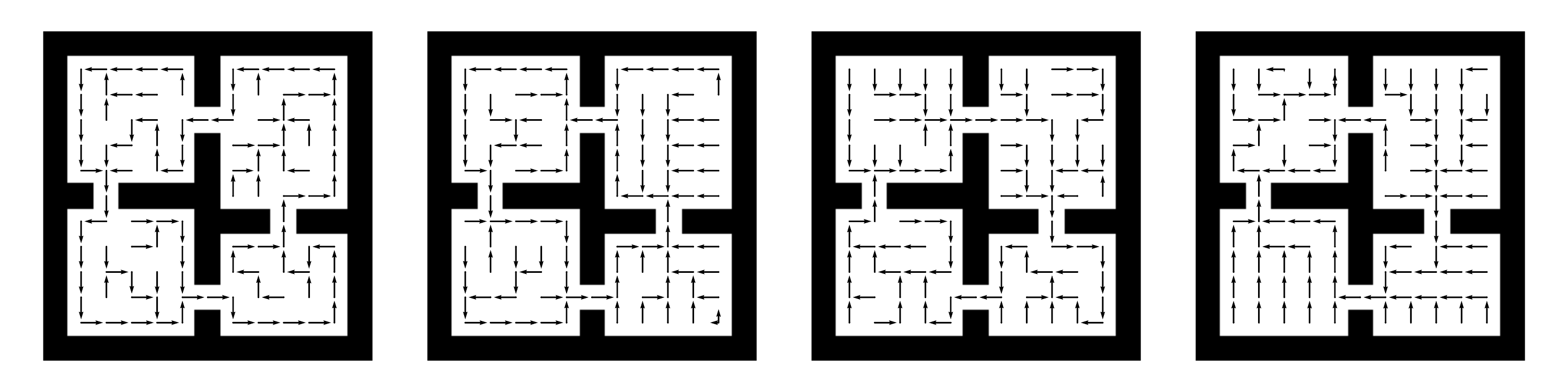}
    \caption{$\pi$}
\end{subfigure}%
\begin{subfigure}{0.5\textwidth}
    \centering
    \includegraphics[width=\textwidth]{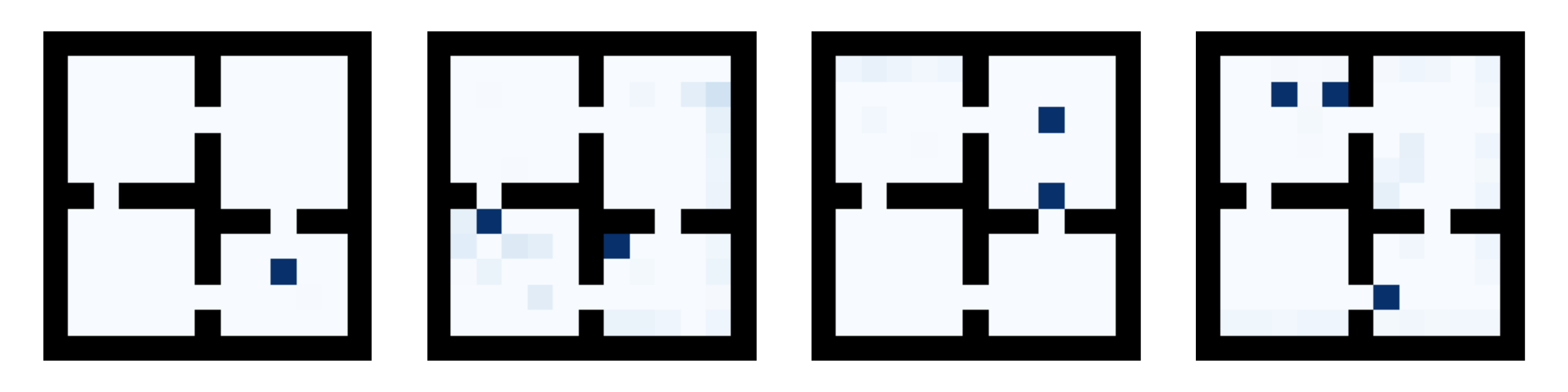}
    \caption{$\beta$}
\end{subfigure}
\begin{subfigure}{0.5\textwidth}
    \centering
    \includegraphics[width=\textwidth]{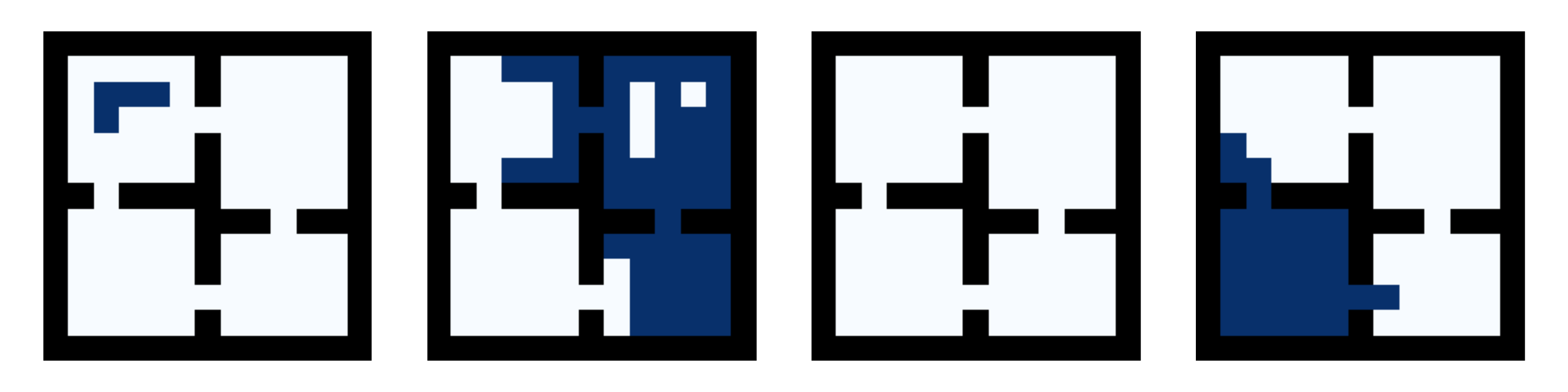}
    \caption{$\mu$ (task 1) }
\end{subfigure}%
\begin{subfigure}{0.5\textwidth}
    \centering
    \includegraphics[width=\textwidth]{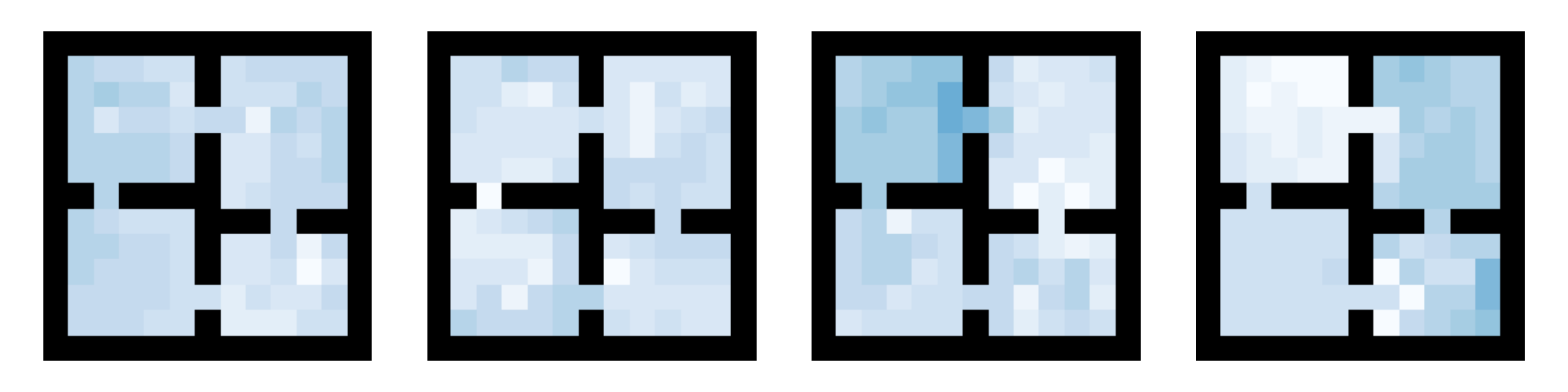}
    \caption{$\mu$ (all tasks) }
\end{subfigure}
\caption{The learned options using MAOC with 4 adjustable options ($c = 0.6$, best $\bar c$ and $\eta$).
}
\label{fig: The learned options using MAOC with 4 adjustable options best c=0.6}
\end{figure*}

\begin{figure*}[h]
\begin{subfigure}{0.5\textwidth}
    \centering
    \includegraphics[width=\textwidth]{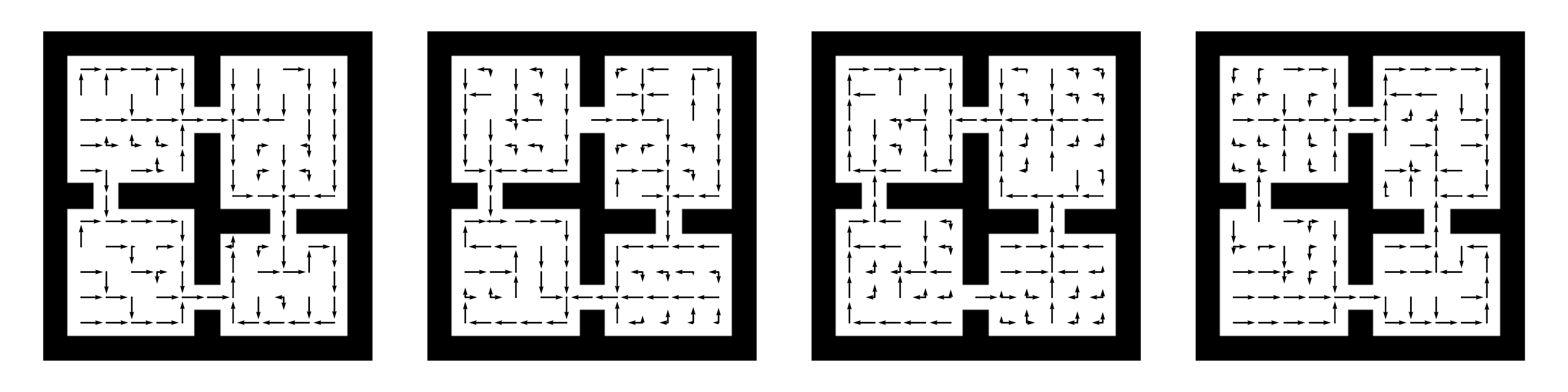}
    \caption{$\pi$}
\end{subfigure}%
\begin{subfigure}{0.5\textwidth}
    \centering
    \includegraphics[width=\textwidth]{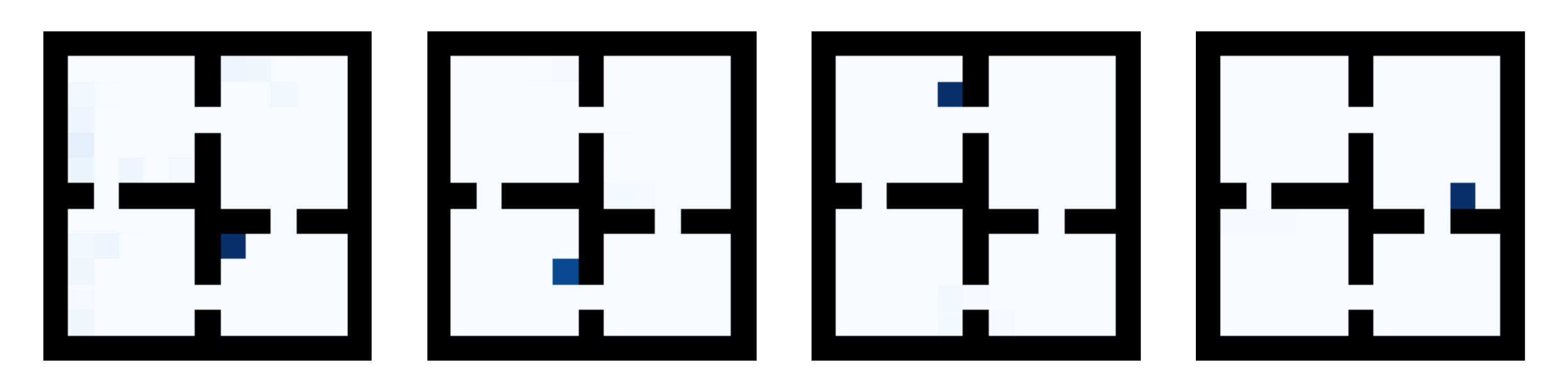}
    \caption{$\beta$}
\end{subfigure}
\begin{subfigure}{0.5\textwidth}
    \centering
    \includegraphics[width=\textwidth]{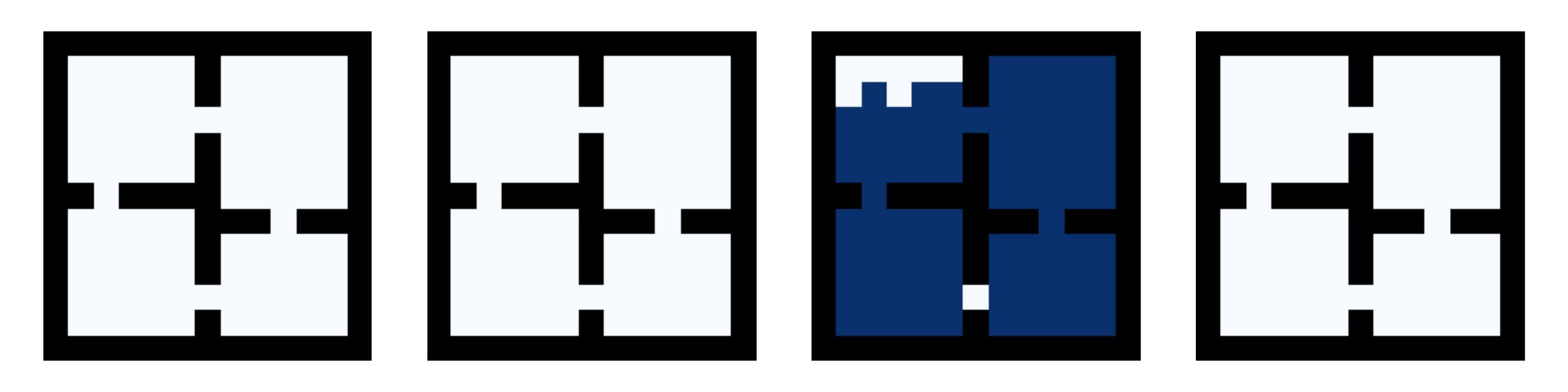}
    \caption{$\mu$ (task 1) }
\end{subfigure}%
\begin{subfigure}{0.5\textwidth}
    \centering
    \includegraphics[width=\textwidth]{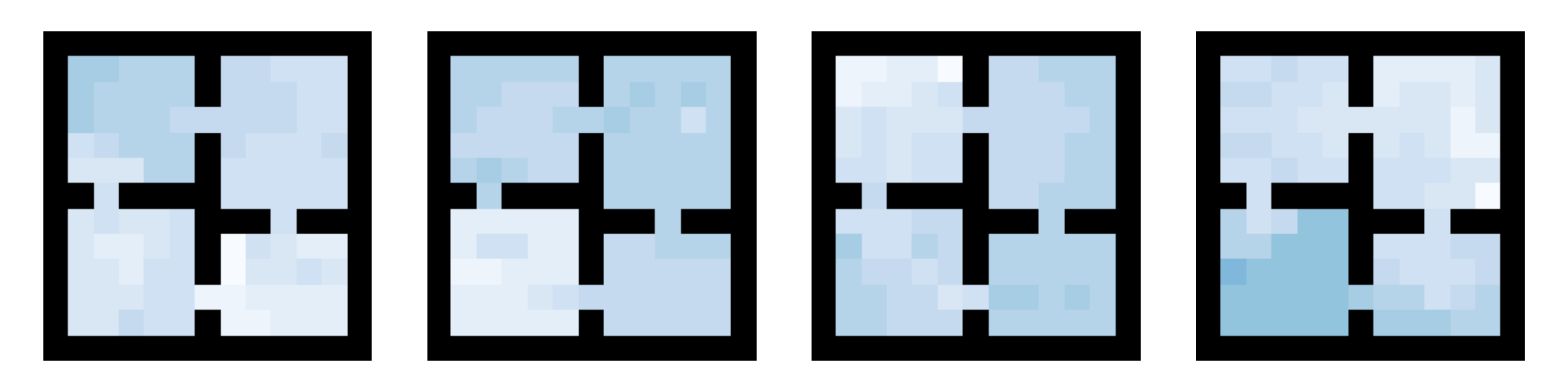}
    \caption{$\mu$ (all tasks) }
\end{subfigure}
\caption{The learned options using MAOC with 4 adjustable options ($c = 1$, best $\bar c$ and $\eta$).
}
\label{fig: The learned options using MAOC with 4 adjustable options best c=1}
\end{figure*}

\begin{figure*}[h]
\begin{subfigure}{\textwidth}
    \centering
    \includegraphics[width=\textwidth]{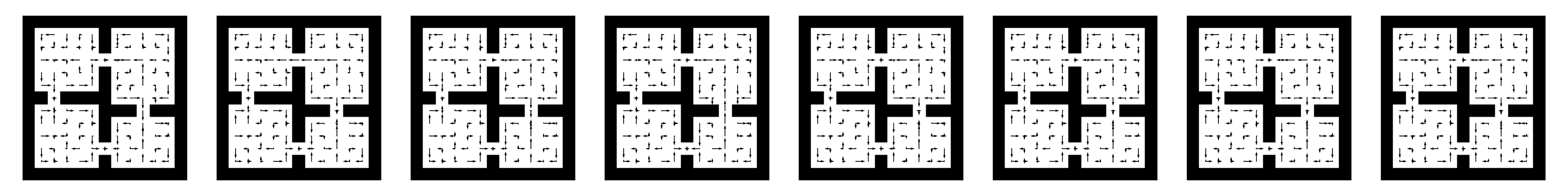}
    \caption{$\pi$}
\end{subfigure}
\begin{subfigure}{\textwidth}
    \centering
    \includegraphics[width=\textwidth]{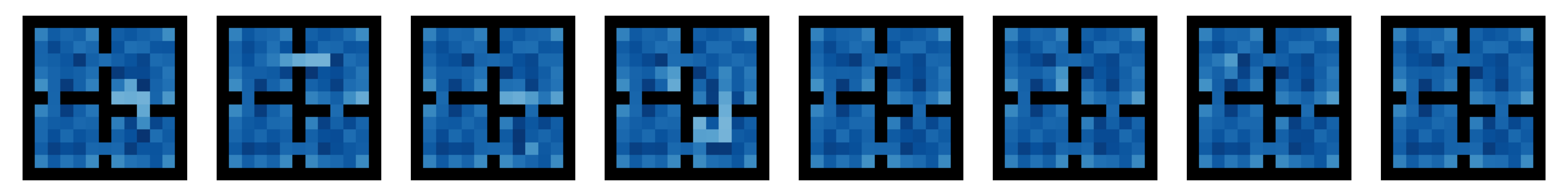}
    \caption{$\beta$}
\end{subfigure}
\begin{subfigure}{\textwidth}
    \centering
    \includegraphics[width=\textwidth]{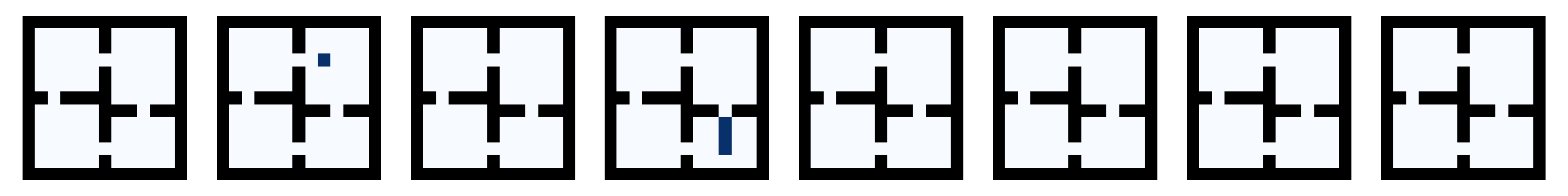}
    \caption{$\mu$ (task 1) }
\end{subfigure}
\begin{subfigure}{\textwidth}
    \centering
    \includegraphics[width=\textwidth]{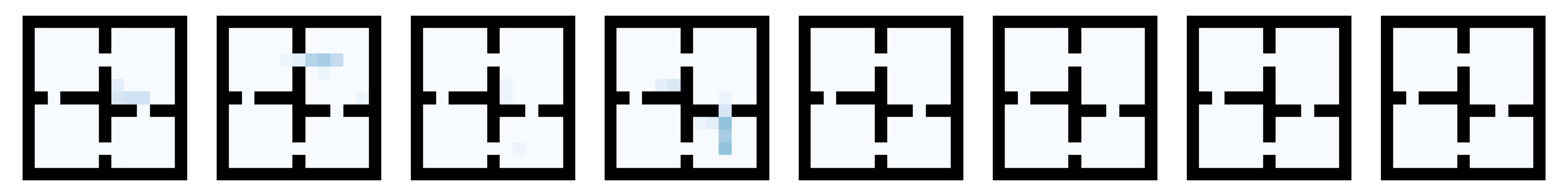}
    \caption{$\mu$ (all tasks) }
\end{subfigure}
\caption{The learned options using MAOC with 8 adjustable options ($c = 0$, best $\bar c$ and $\eta$).
}
\label{fig: The learned options using MAOC with 8 adjustable options best c=0}
\end{figure*}

\begin{figure*}[h]
\begin{subfigure}{\textwidth}
    \centering
    \includegraphics[width=\textwidth]{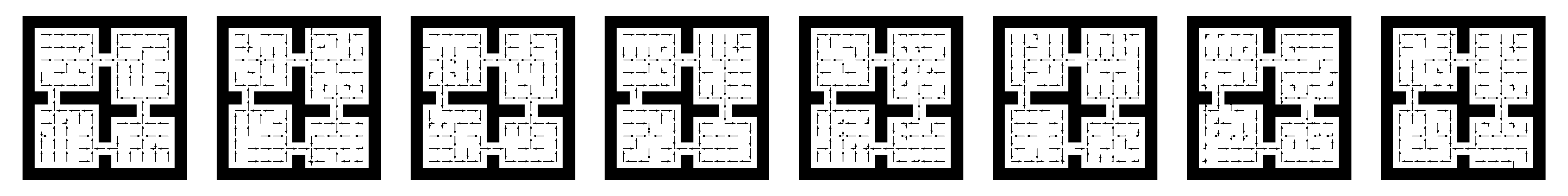}
    \caption{$\pi$}
\end{subfigure}
\begin{subfigure}{\textwidth}
    \centering
    \includegraphics[width=\textwidth]{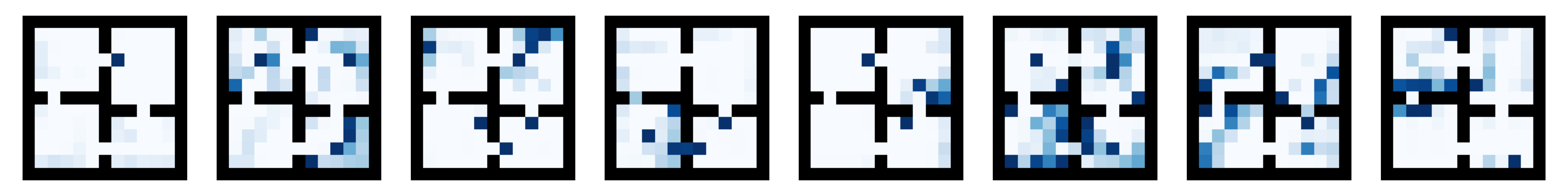}
    \caption{$\beta$}
\end{subfigure}
\begin{subfigure}{\textwidth}
    \centering
    \includegraphics[width=\textwidth]{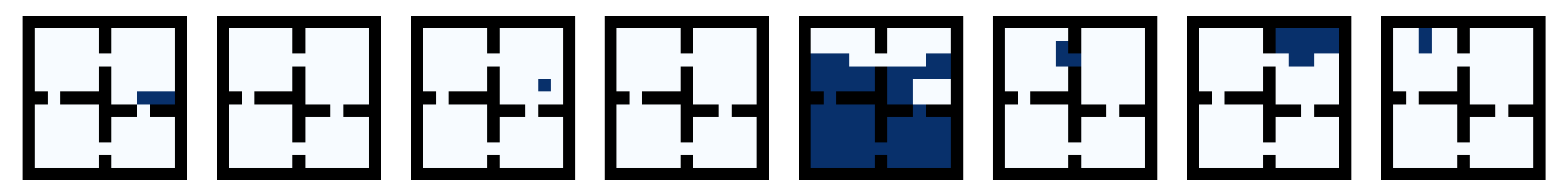}
    \caption{$\mu$ (task 1) }
\end{subfigure}
\begin{subfigure}{\textwidth}
    \centering
    \includegraphics[width=\textwidth]{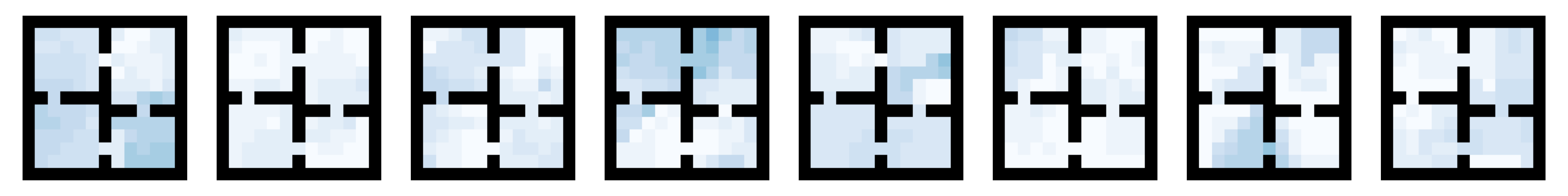}
    \caption{$\mu$ (all tasks) }
\end{subfigure}
\caption{The learned options using MAOC with 8 adjustable options ($c = 0.2$, best $\bar c$ and $\eta$).
}
\label{fig: The learned options using MAOC with 8 adjustable options best c=0.2}
\end{figure*}

\begin{figure*}[h]
\begin{subfigure}{\textwidth}
    \centering
    \includegraphics[width=\textwidth]{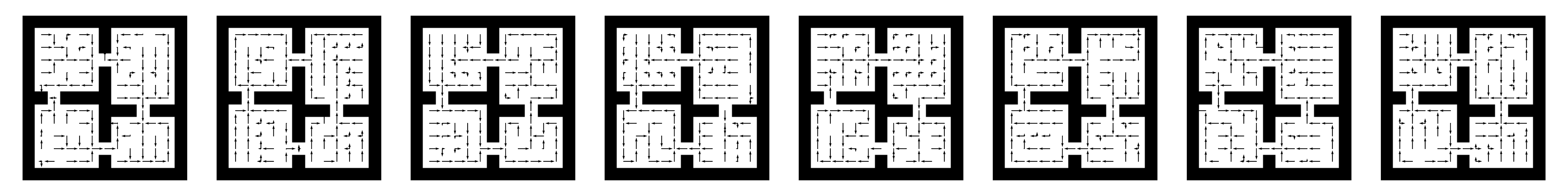}
    \caption{$\pi$}
\end{subfigure}
\begin{subfigure}{\textwidth}
    \centering
    \includegraphics[width=\textwidth]{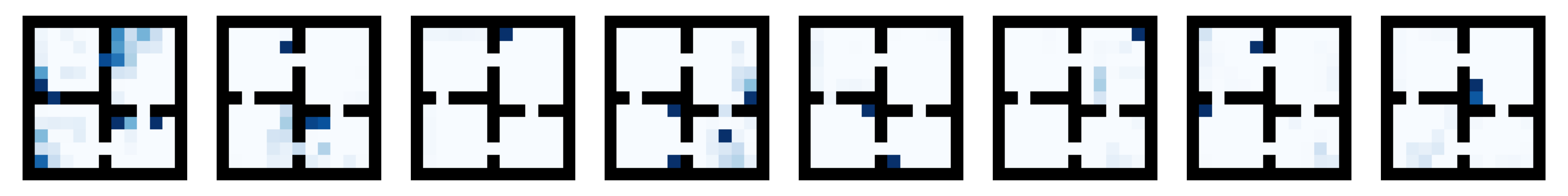}
    \caption{$\beta$}
\end{subfigure}
\begin{subfigure}{\textwidth}
    \centering
    \includegraphics[width=\textwidth]{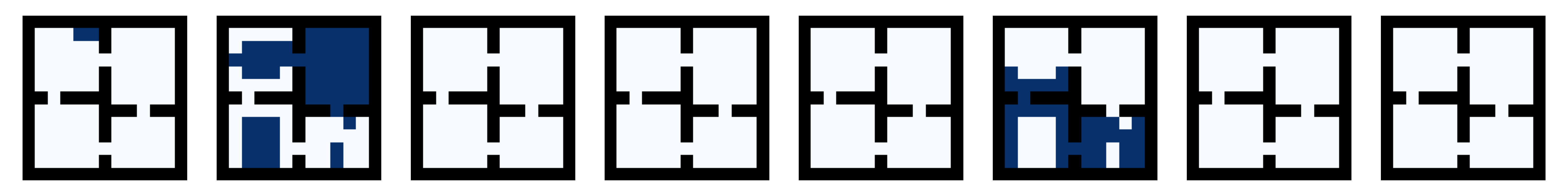}
    \caption{$\mu$ (task 1) }
\end{subfigure}
\begin{subfigure}{\textwidth}
    \centering
    \includegraphics[width=\textwidth]{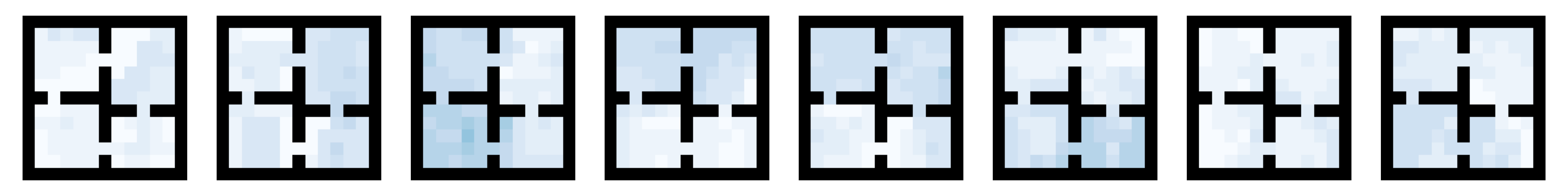}
    \caption{$\mu$ (all tasks) }
\end{subfigure}
\caption{The learned options using MAOC with 8 adjustable options ($c = 0.6$, best $\bar c$ and $\eta$).
}
\label{fig: The learned options using MAOC with 8 adjustable options best c=0.6}
\end{figure*}

\begin{figure*}[h]
\begin{subfigure}{\textwidth}
    \centering
    \includegraphics[width=\textwidth]{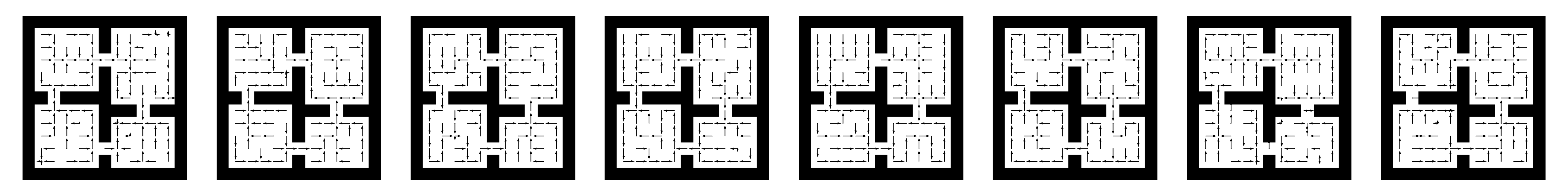}
    \caption{$\pi$}
\end{subfigure}
\begin{subfigure}{\textwidth}
    \centering
    \includegraphics[width=\textwidth]{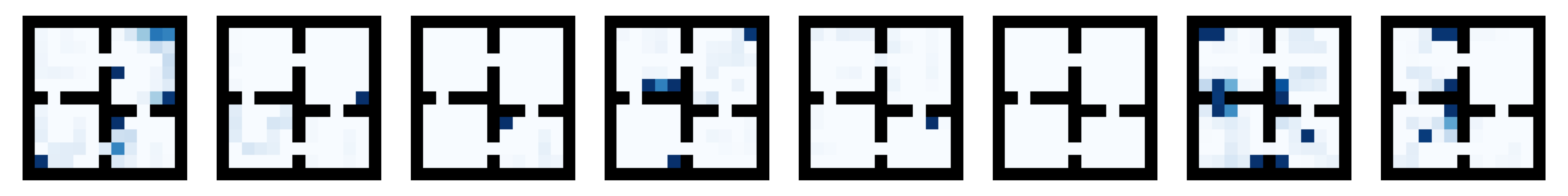}
    \caption{$\beta$}
\end{subfigure}
\begin{subfigure}{\textwidth}
    \centering
    \includegraphics[width=\textwidth]{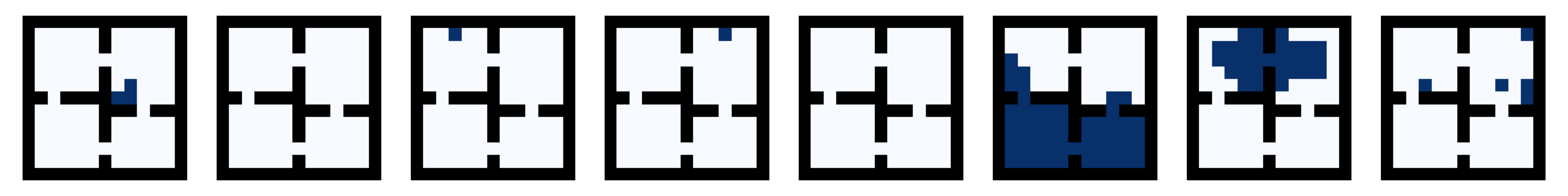}
    \caption{$\mu$ (task 1) }
\end{subfigure}
\begin{subfigure}{\textwidth}
    \centering
    \includegraphics[width=\textwidth]{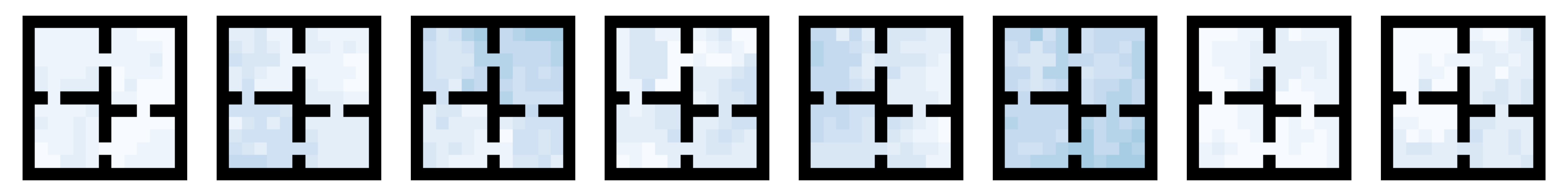}
    \caption{$\mu$ (all tasks) }
\end{subfigure}
\caption{The learned options using MAOC with 8 adjustable options ($c = 1$, best $\bar c$ and $\eta$).
}
\label{fig: The learned options using MAOC with 8 adjustable options best c=1}
\end{figure*}

\begin{figure*}[h]
\begin{subfigure}{0.25\textwidth}
    \centering
    \includegraphics[width=\textwidth]{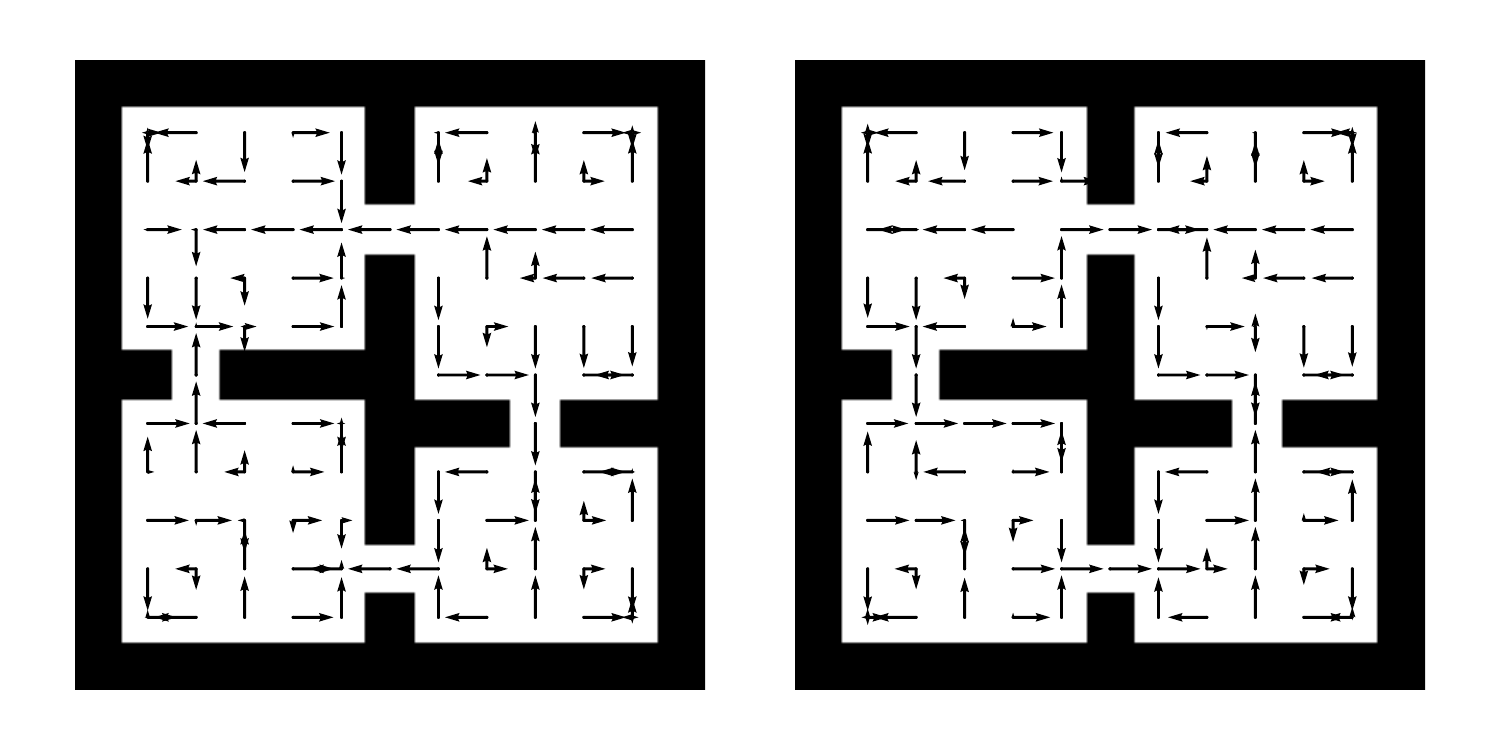}
    \caption{$\pi$}
\end{subfigure}%
\begin{subfigure}{0.25\textwidth}
    \centering
    \includegraphics[width=\textwidth]{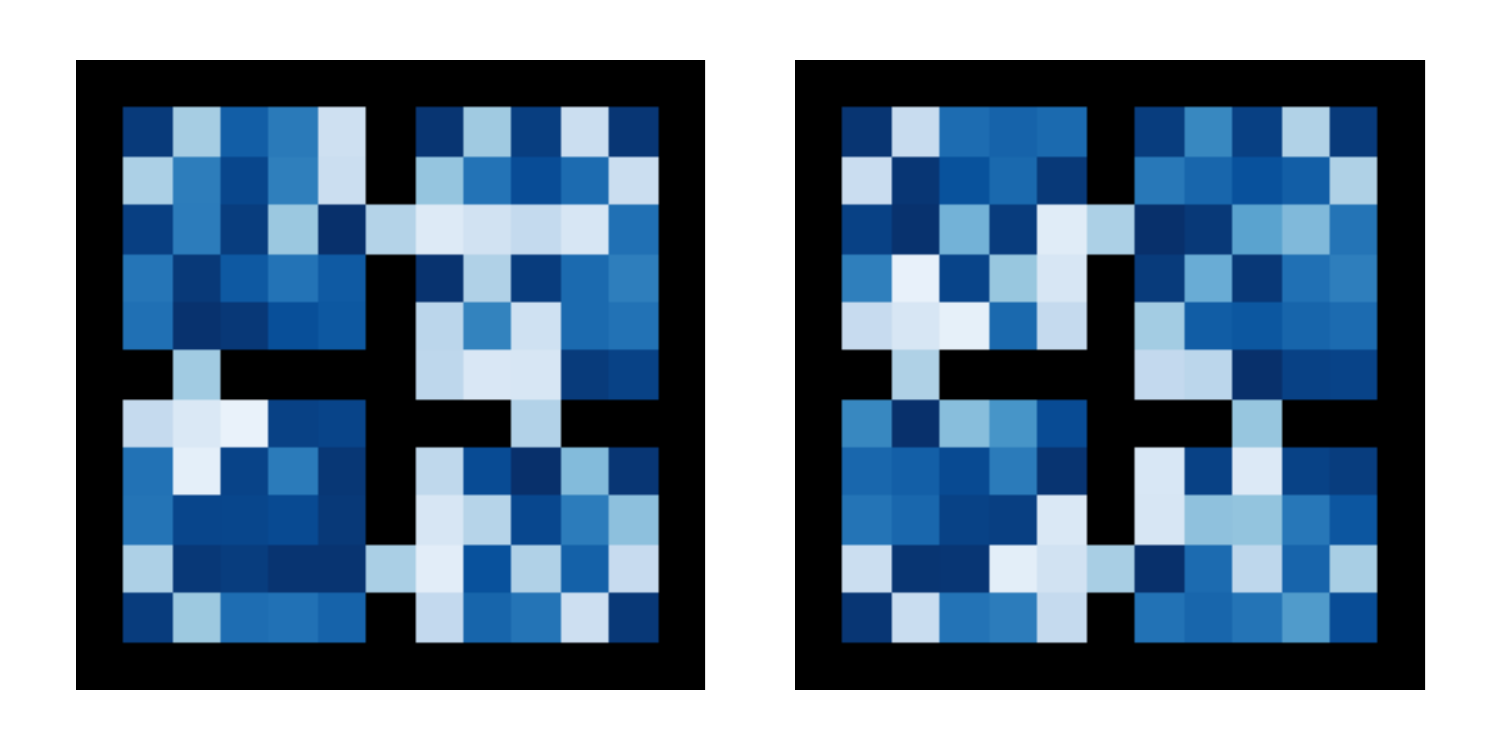}
    \caption{$\beta$}
\end{subfigure}%
\begin{subfigure}{0.25\textwidth}
    \centering
    \includegraphics[width=\textwidth]{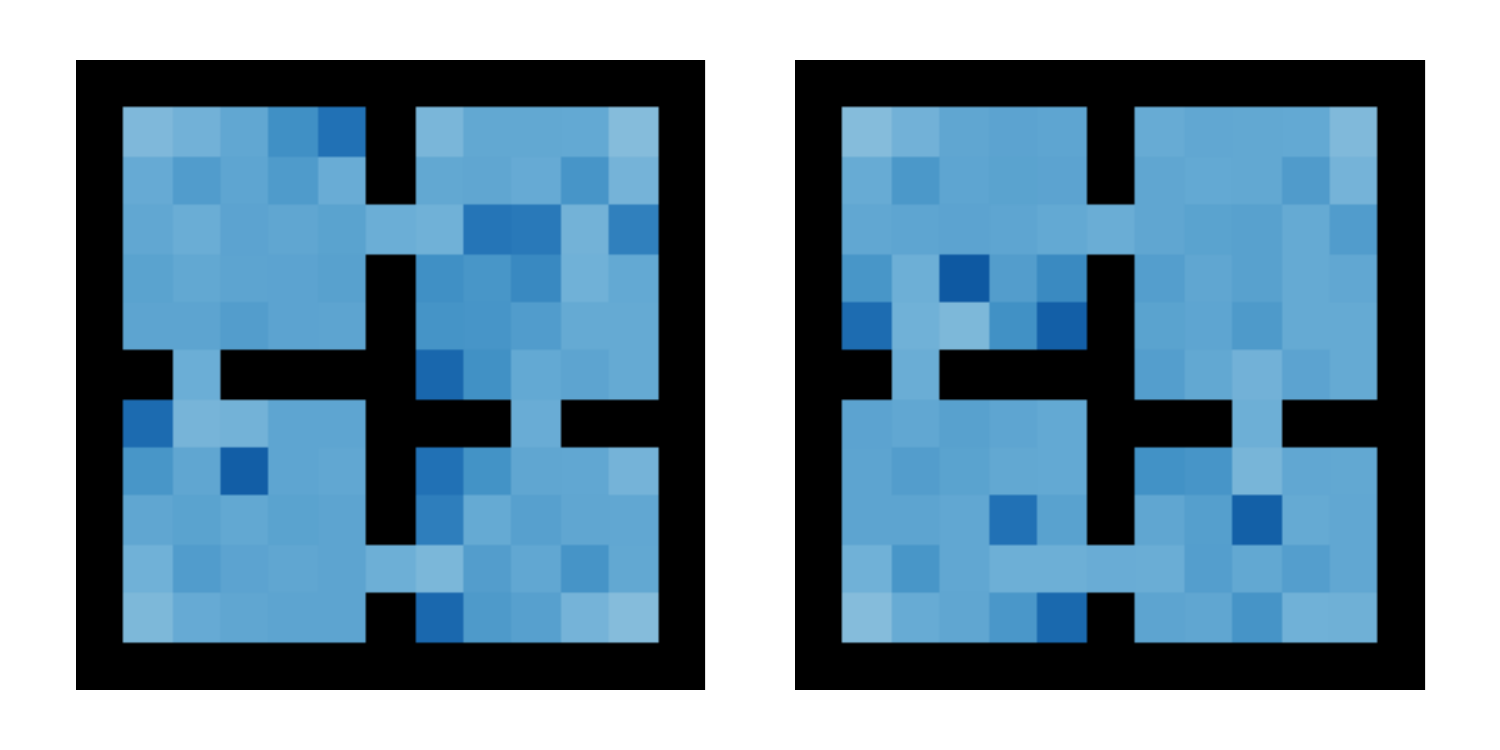}
    \caption{$i$}
\end{subfigure}%
\begin{subfigure}{0.25\textwidth}
    \centering
    \includegraphics[width=\textwidth]{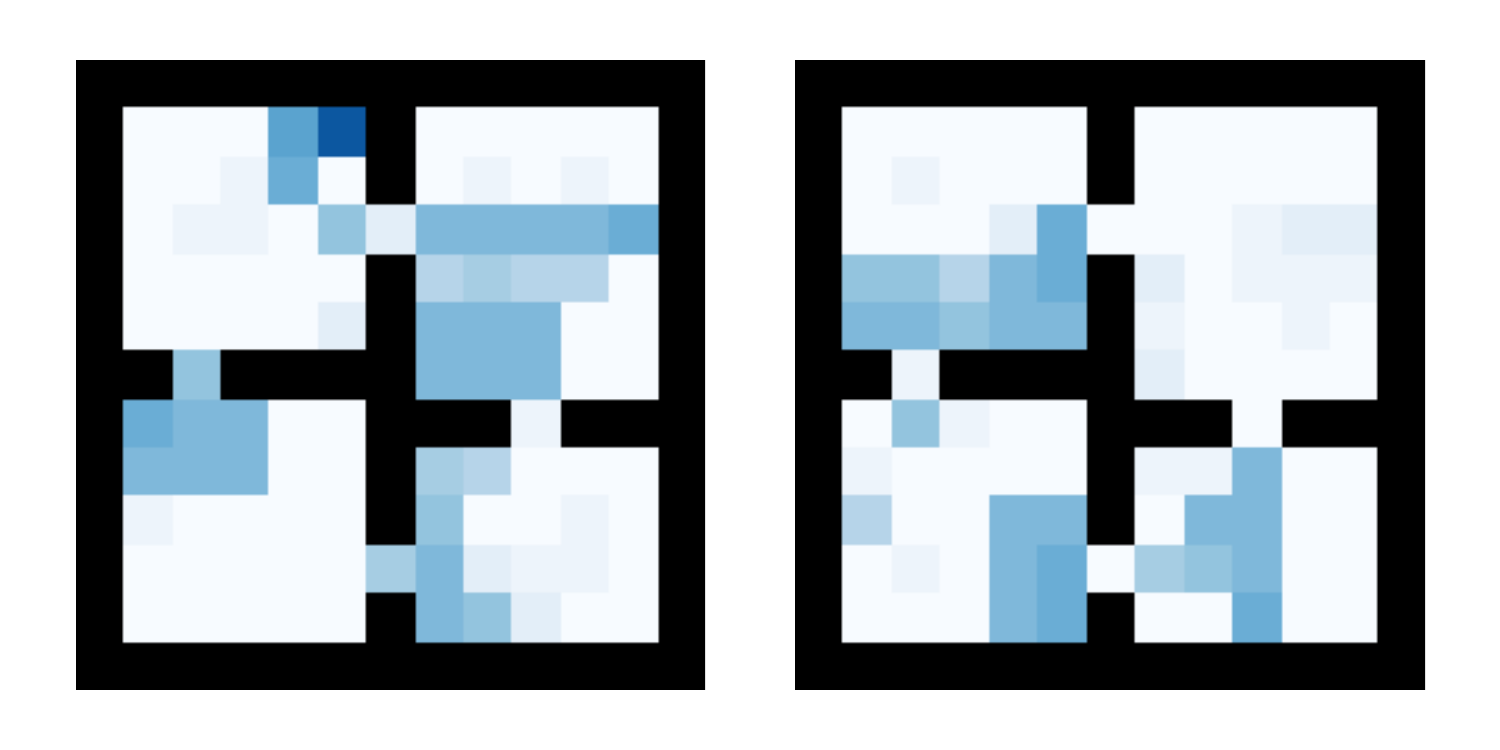}
    \caption{$\mu$ (task 1)}
\end{subfigure}
\begin{subfigure}{0.25\textwidth}
    \centering
    \includegraphics[width=\textwidth]{figures/684_mu_all_tasks.pdf}
    \caption{$\mu$ (all tasks)}
\end{subfigure}
\caption{The learned options using FPOC with 2 adjustable options ($c = 0$, best $\bar c$ and $\eta$). For FPOC, the interest function is learned. We show the learned interest function in (c).
}
\label{fig: The learned options using FPOC with 2 adjustable options best c=0}
\end{figure*}

\begin{figure*}[h]
\begin{subfigure}{0.25\textwidth}
    \centering
    \includegraphics[width=\textwidth]{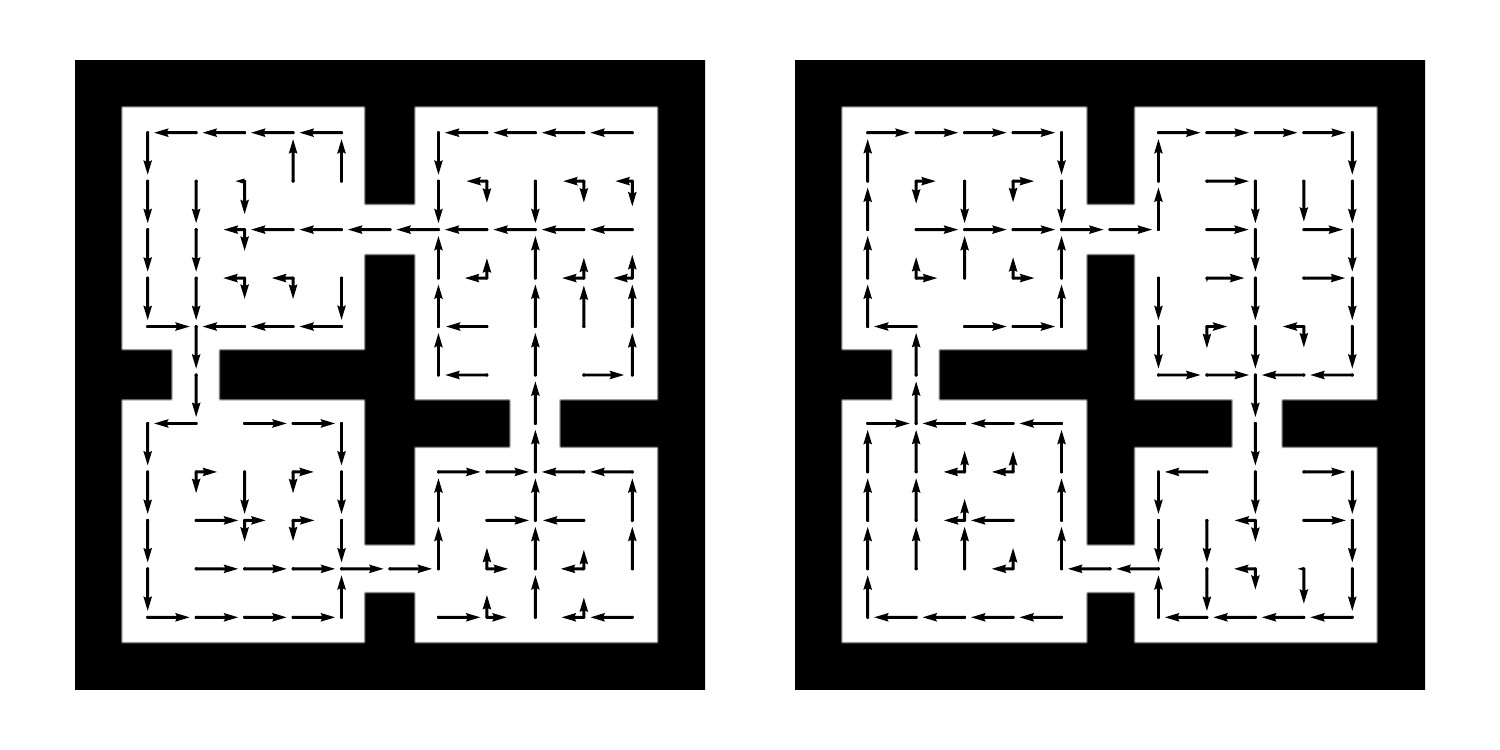}
    \caption{$\pi$}
\end{subfigure}%
\begin{subfigure}{0.25\textwidth}
    \centering
    \includegraphics[width=\textwidth]{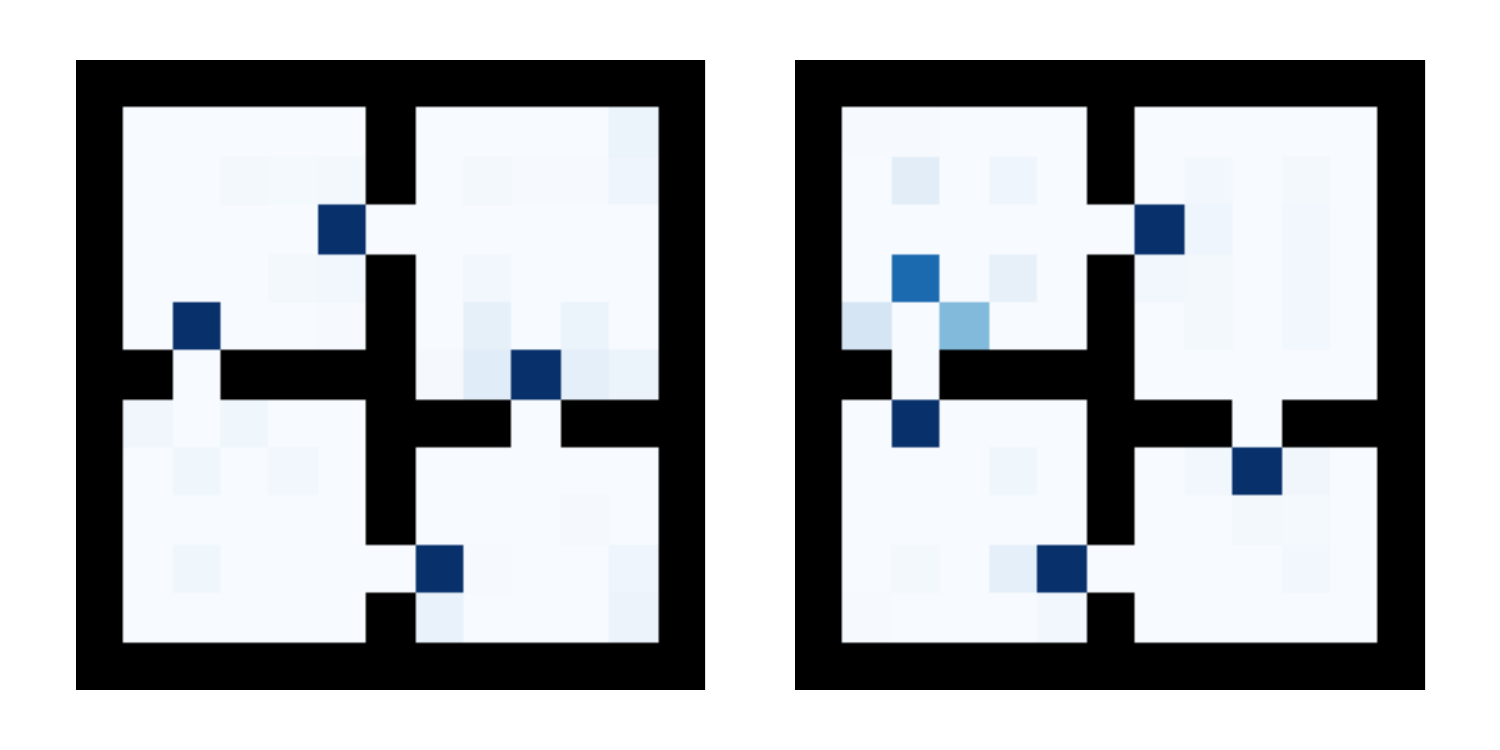}
    \caption{$\beta$}
\end{subfigure}%
\begin{subfigure}{0.25\textwidth}
    \centering
    \includegraphics[width=\textwidth]{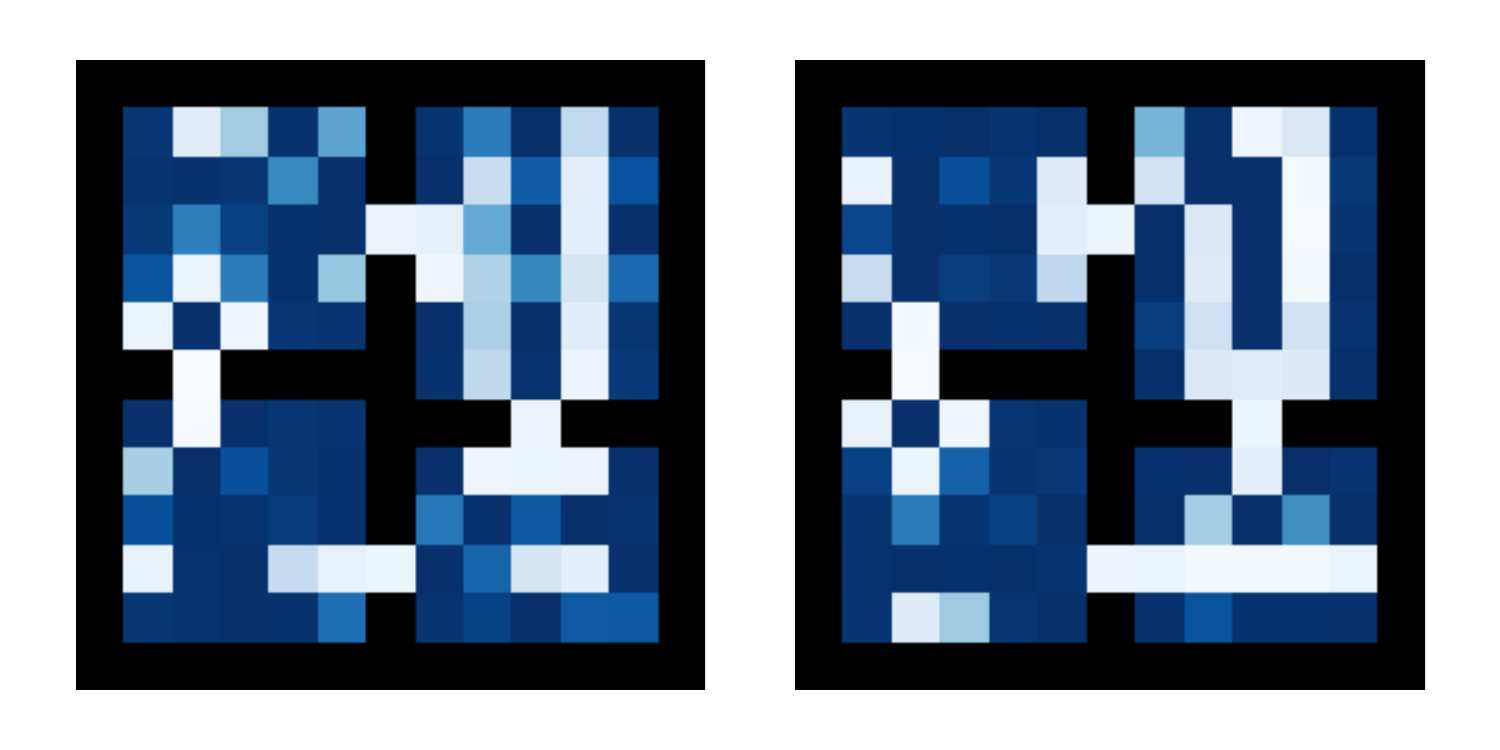}
    \caption{$i$}
\end{subfigure}%
\begin{subfigure}{0.25\textwidth}
    \centering
    \includegraphics[width=\textwidth]{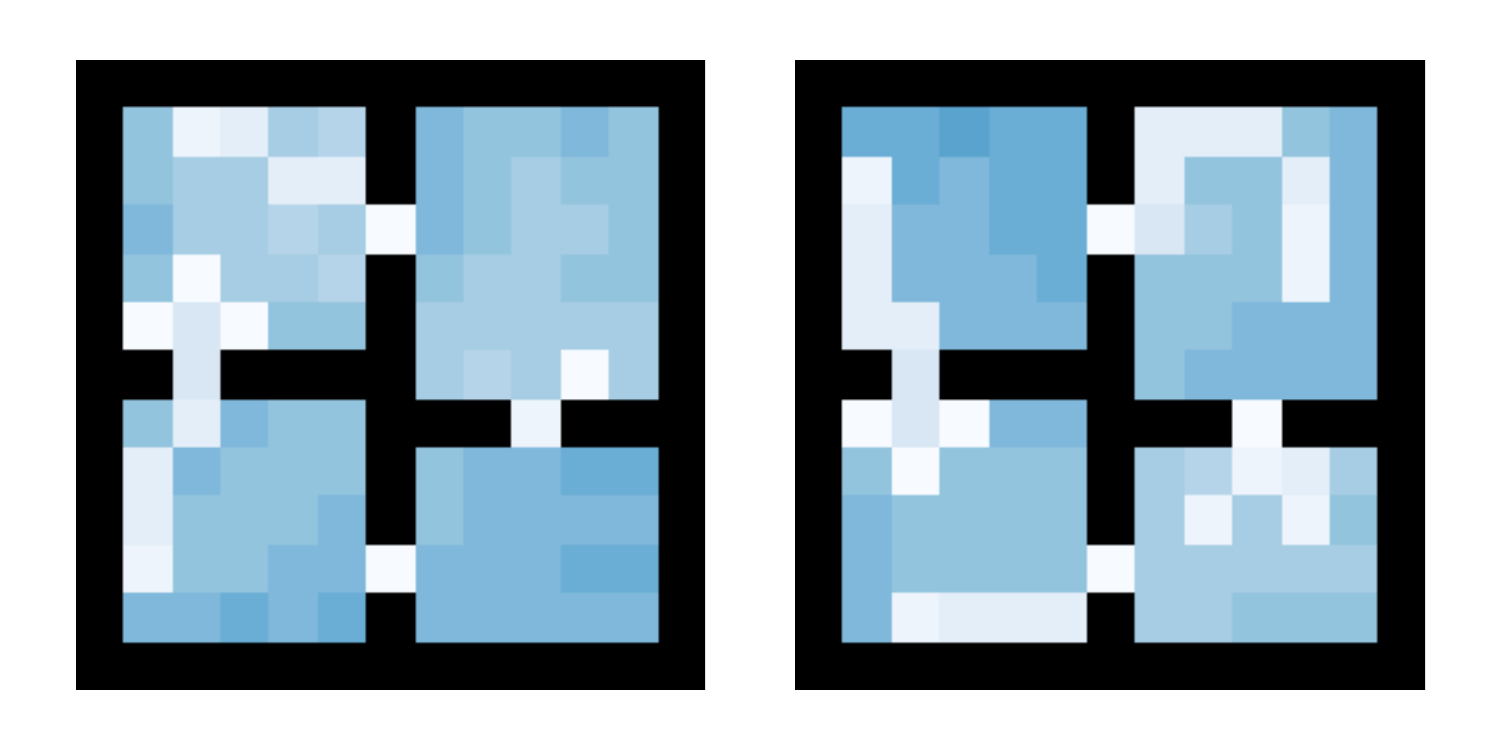}
    \caption{$\mu$ (task 1)}
\end{subfigure}
\begin{subfigure}{0.25\textwidth}
    \centering
    \includegraphics[width=\textwidth]{figures/39_mu_all_tasks.pdf}
    \caption{$\mu$ (all tasks)}
\end{subfigure}
\caption{The learned options using FPOC with 2 adjustable options ($c = 0.2$, best $\bar c$ and $\eta$).
}
\label{fig: The learned options using FPOC with 2 adjustable options best c=0.2}
\end{figure*}

\begin{figure*}[h]
\begin{subfigure}{0.25\textwidth}
    \centering
    \includegraphics[width=\textwidth]{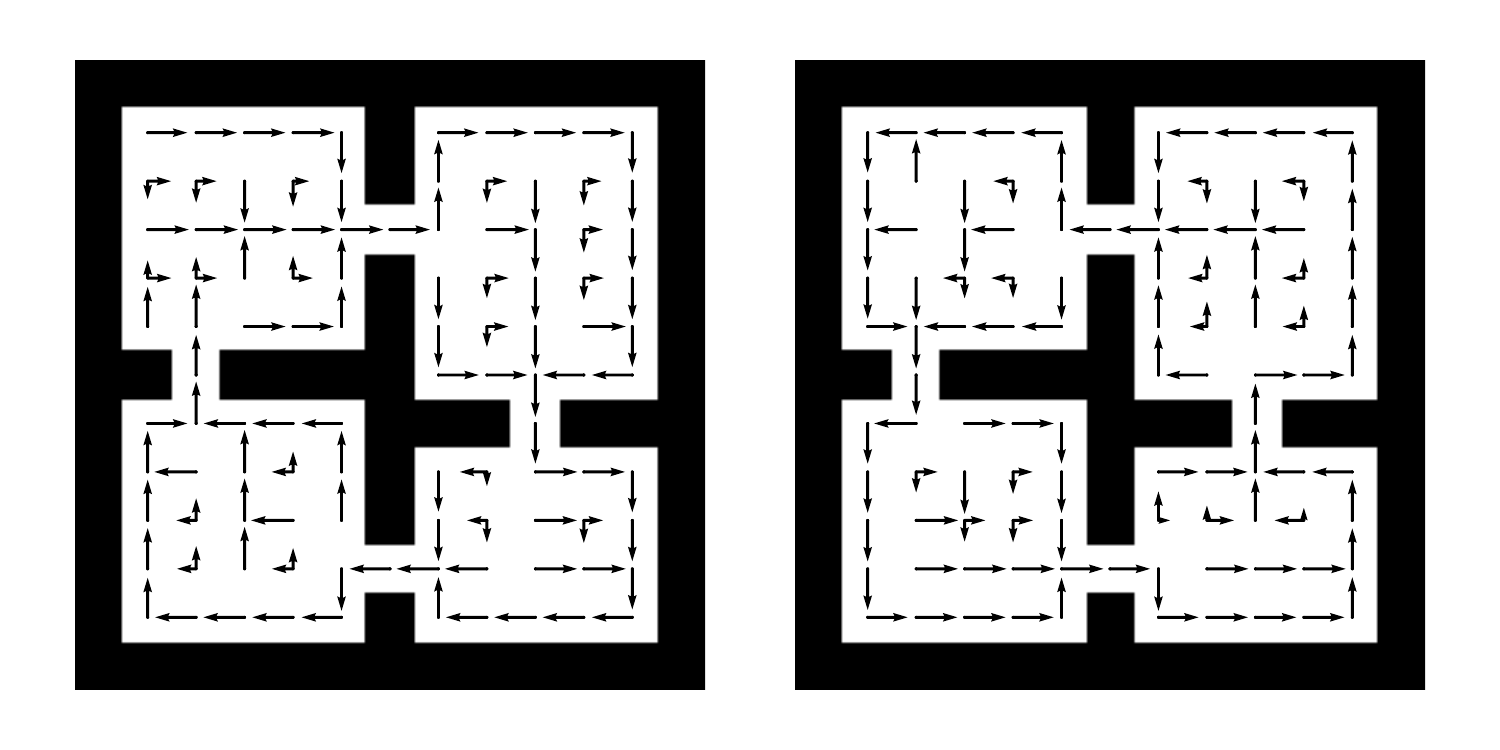}
    \caption{$\pi$}
\end{subfigure}%
\begin{subfigure}{0.25\textwidth}
    \centering
    \includegraphics[width=\textwidth]{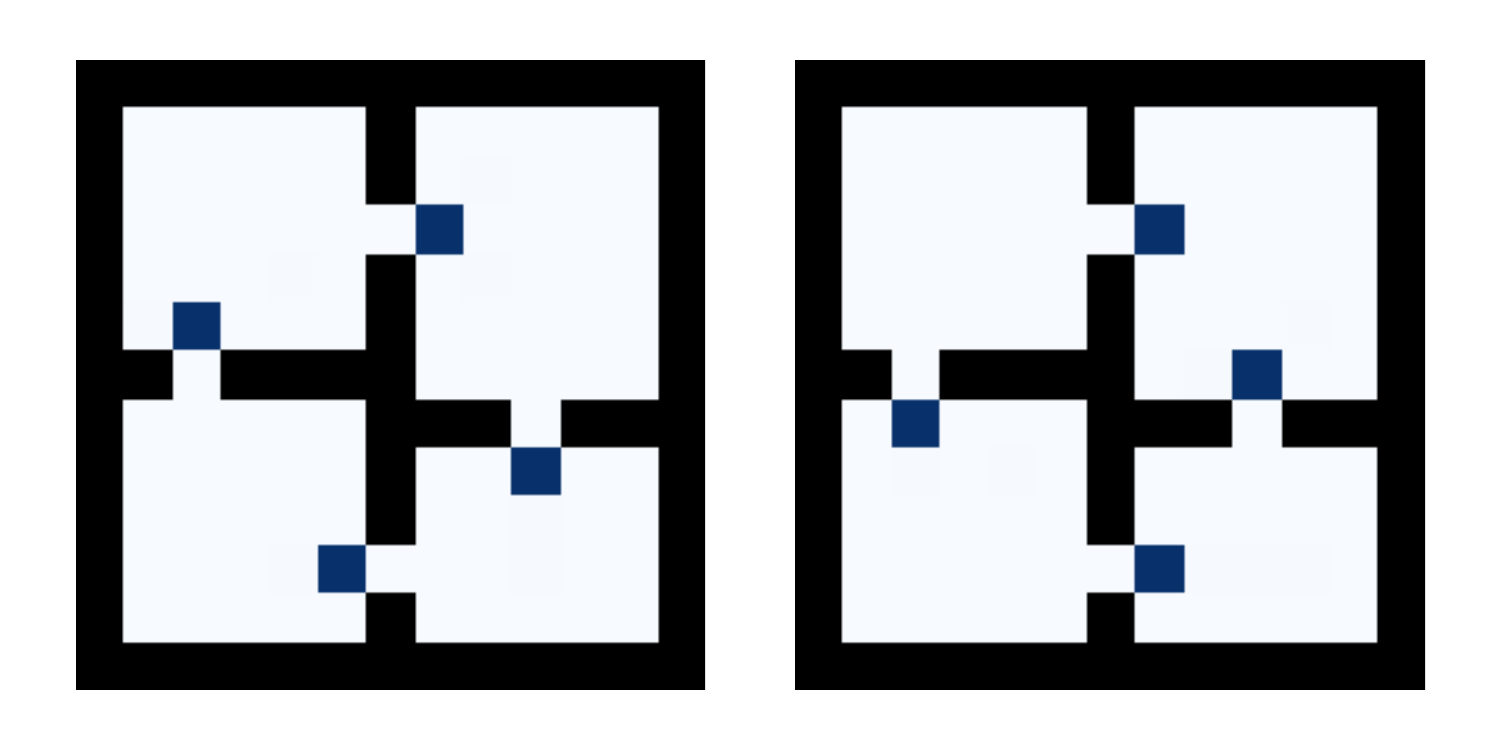}
    \caption{$\beta$}
\end{subfigure}%
\begin{subfigure}{0.25\textwidth}
    \centering
    \includegraphics[width=\textwidth]{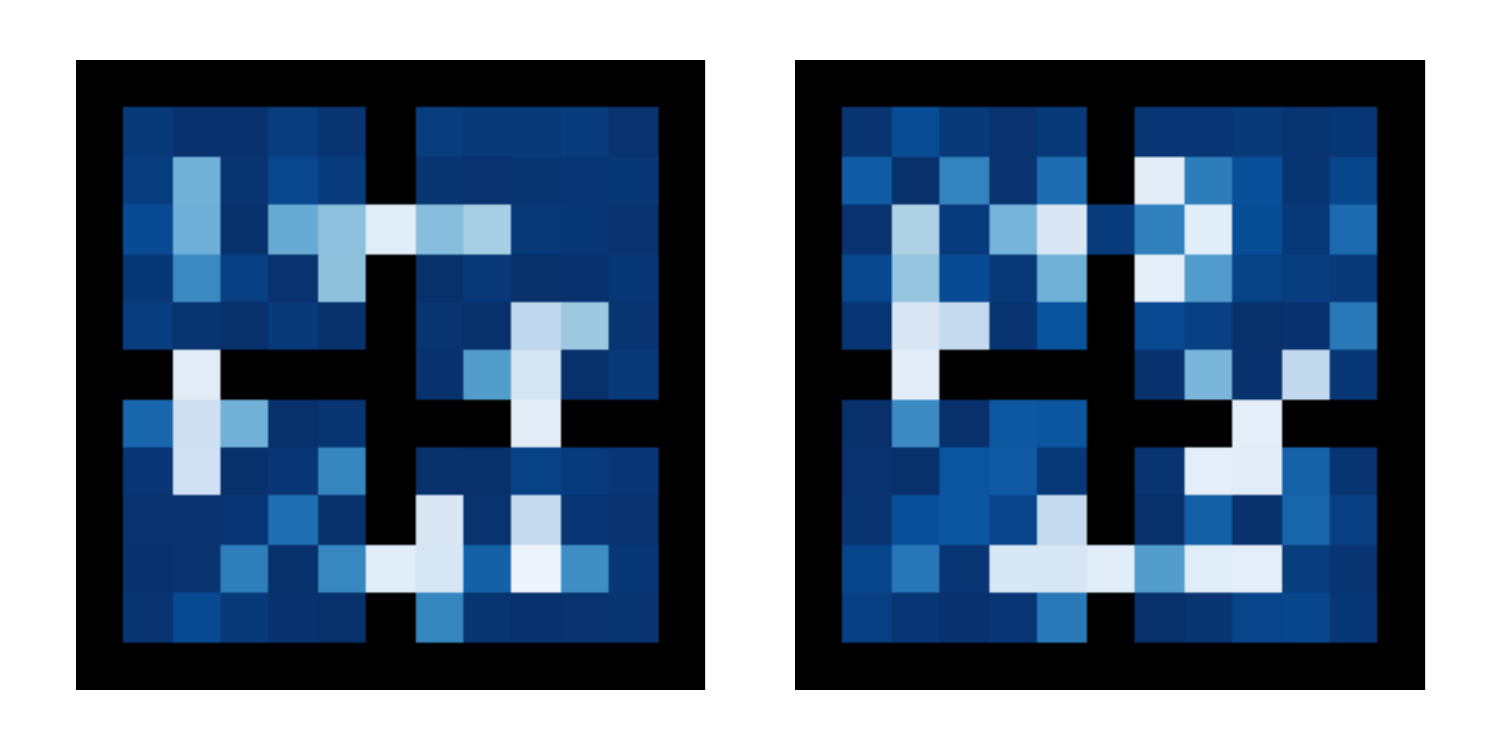}
    \caption{$i$}
\end{subfigure}%
\begin{subfigure}{0.25\textwidth}
    \centering
    \includegraphics[width=\textwidth]{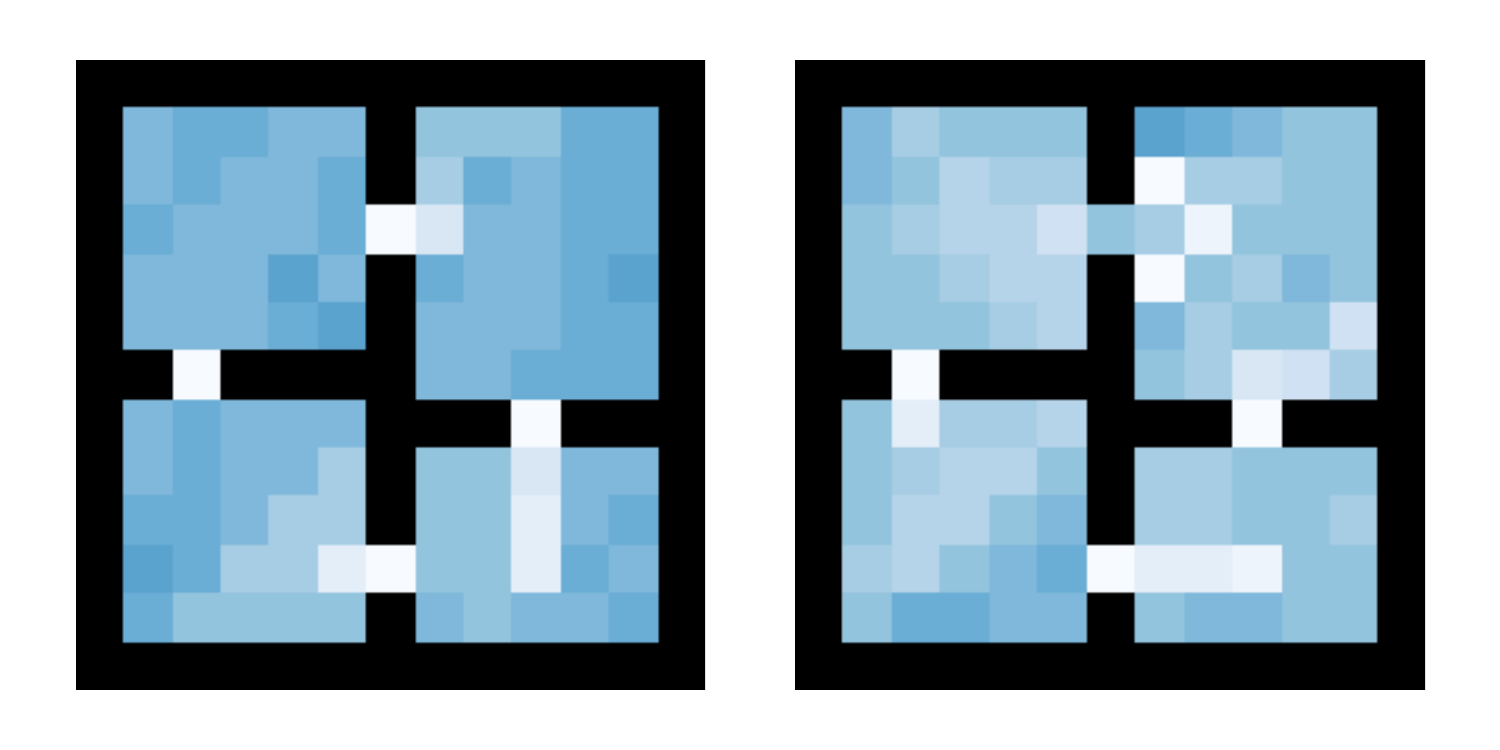}
    \caption{$\mu$ (task 1)}
\end{subfigure}
\begin{subfigure}{0.25\textwidth}
    \centering
    \includegraphics[width=\textwidth]{figures/297_mu_all_tasks.pdf}
    \caption{$\mu$ (all tasks)}
\end{subfigure}
\caption{The learned options using FPOC with 2 adjustable options ($c = 0.6$, best $\bar c$ and $\eta$).
}
\label{fig: The learned options using FPOC with 2 adjustable options best c=0.6}
\end{figure*}

\begin{figure*}[h]
\begin{subfigure}{0.25\textwidth}
    \centering
    \includegraphics[width=\textwidth]{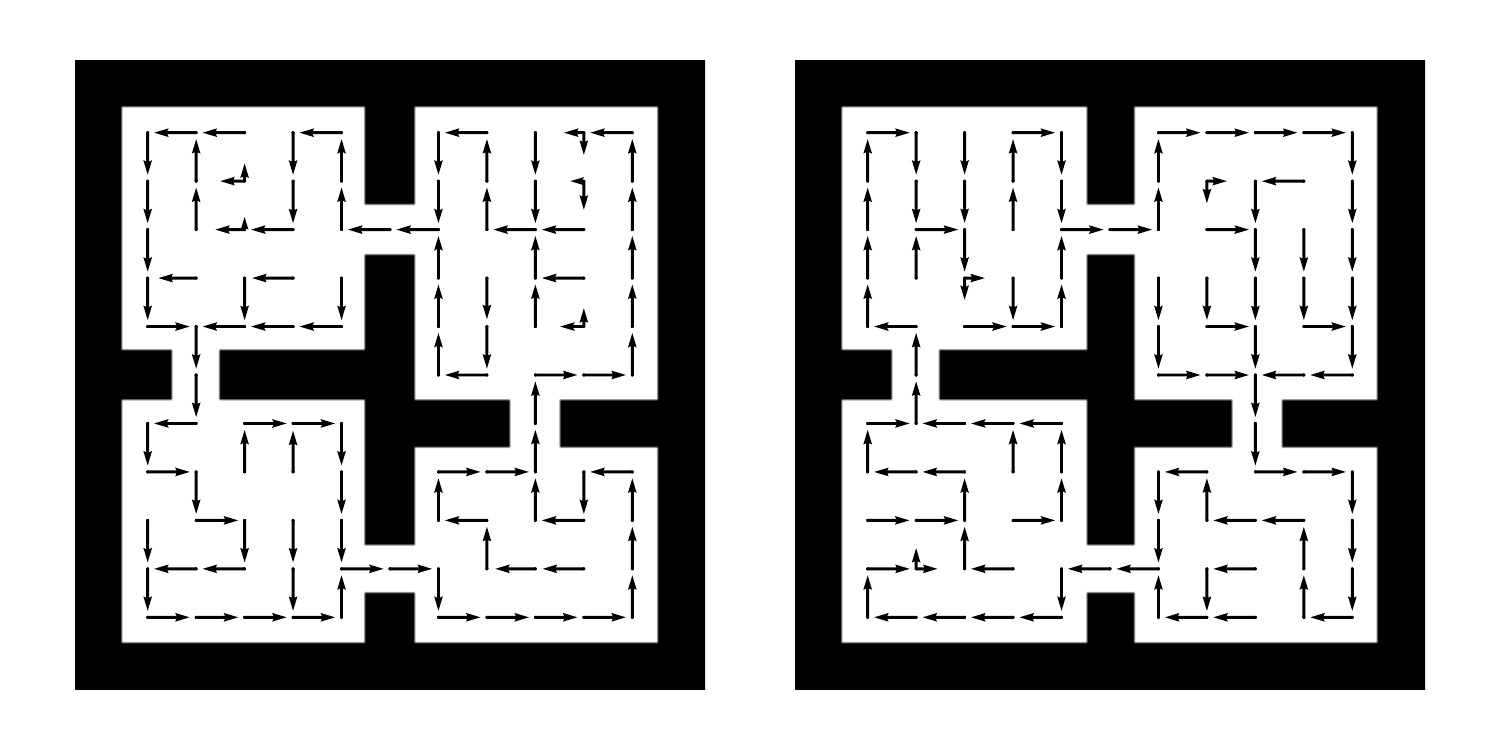}
    \caption{$\pi$}
\end{subfigure}%
\begin{subfigure}{0.25\textwidth}
    \centering
    \includegraphics[width=\textwidth]{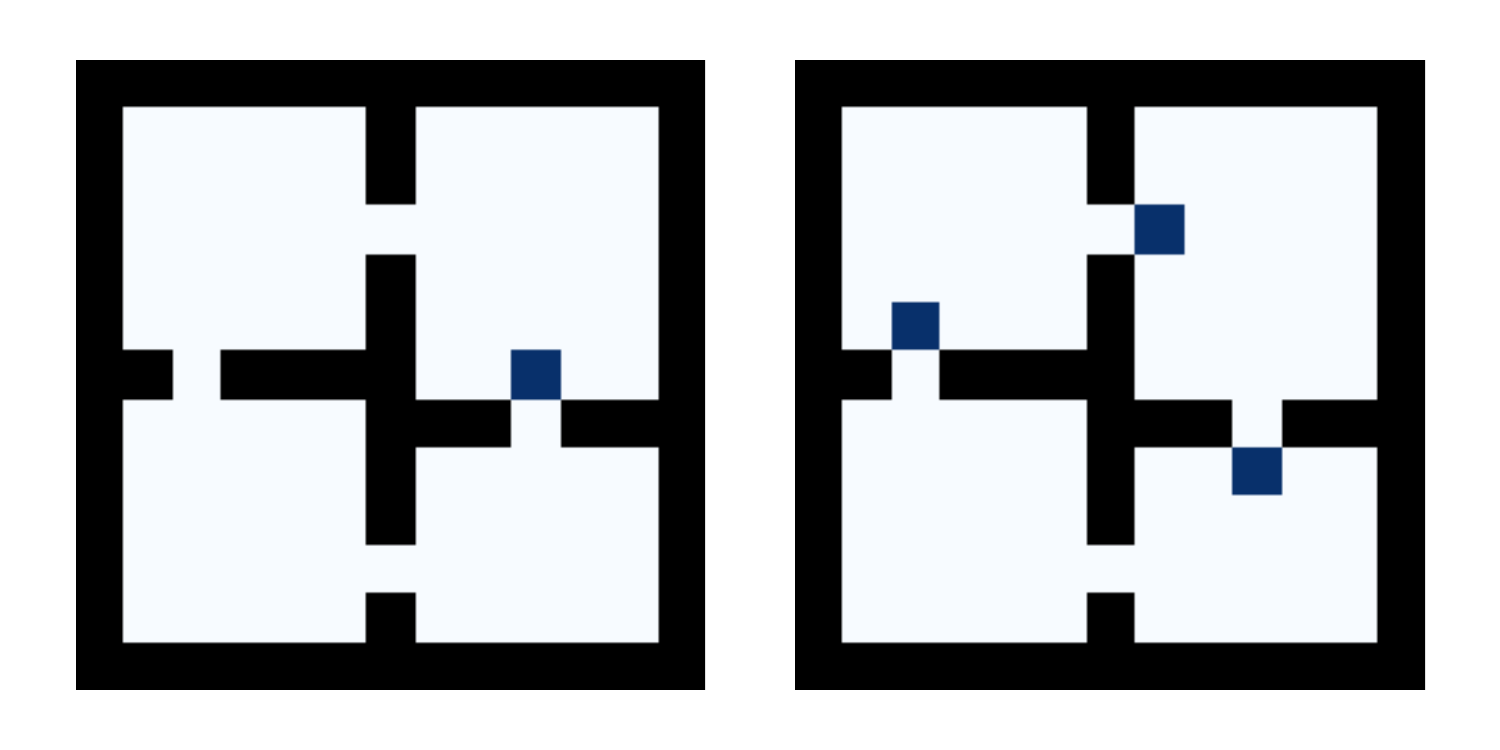}
    \caption{$\beta$}
\end{subfigure}%
\begin{subfigure}{0.25\textwidth}
    \centering
    \includegraphics[width=\textwidth]{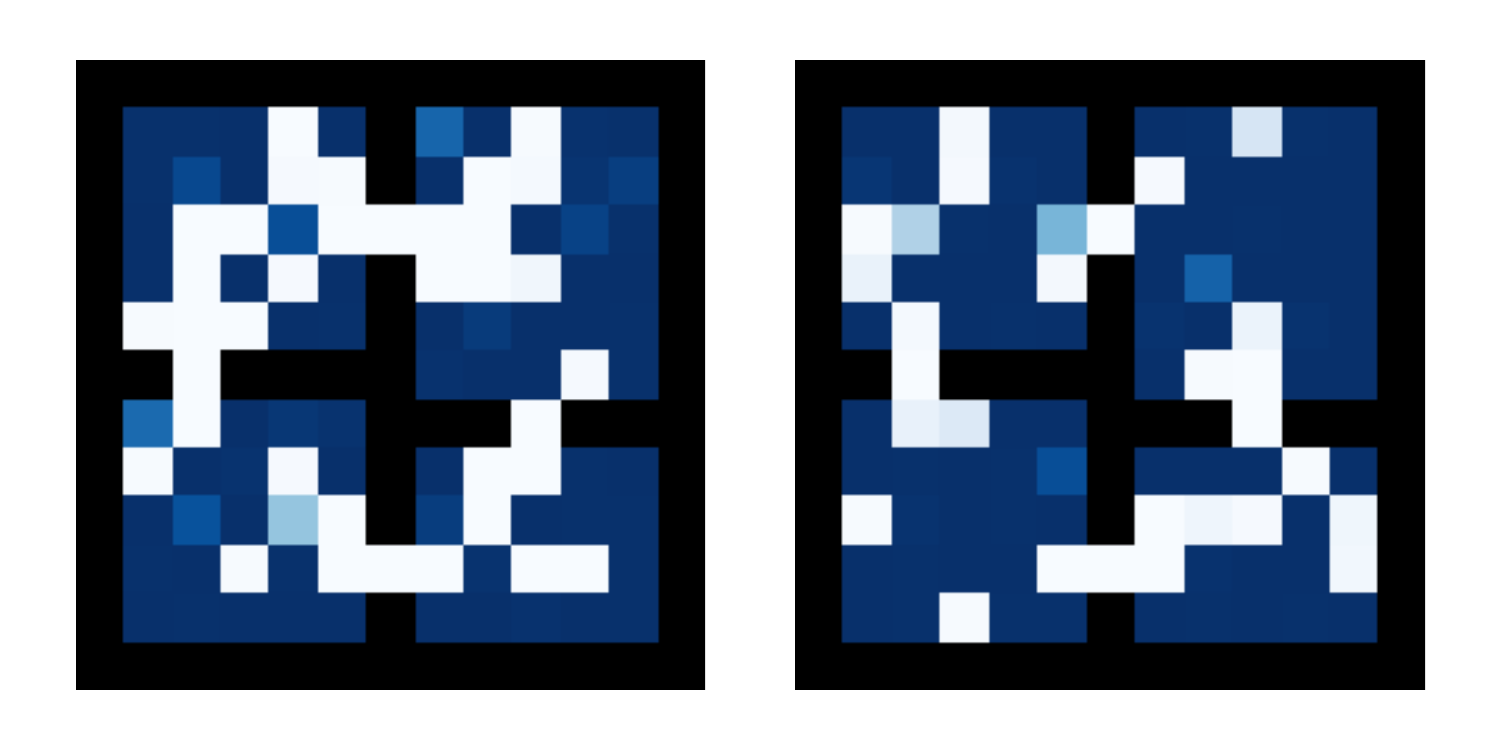}
    \caption{$i$}
\end{subfigure}%
\begin{subfigure}{0.25\textwidth}
    \centering
    \includegraphics[width=\textwidth]{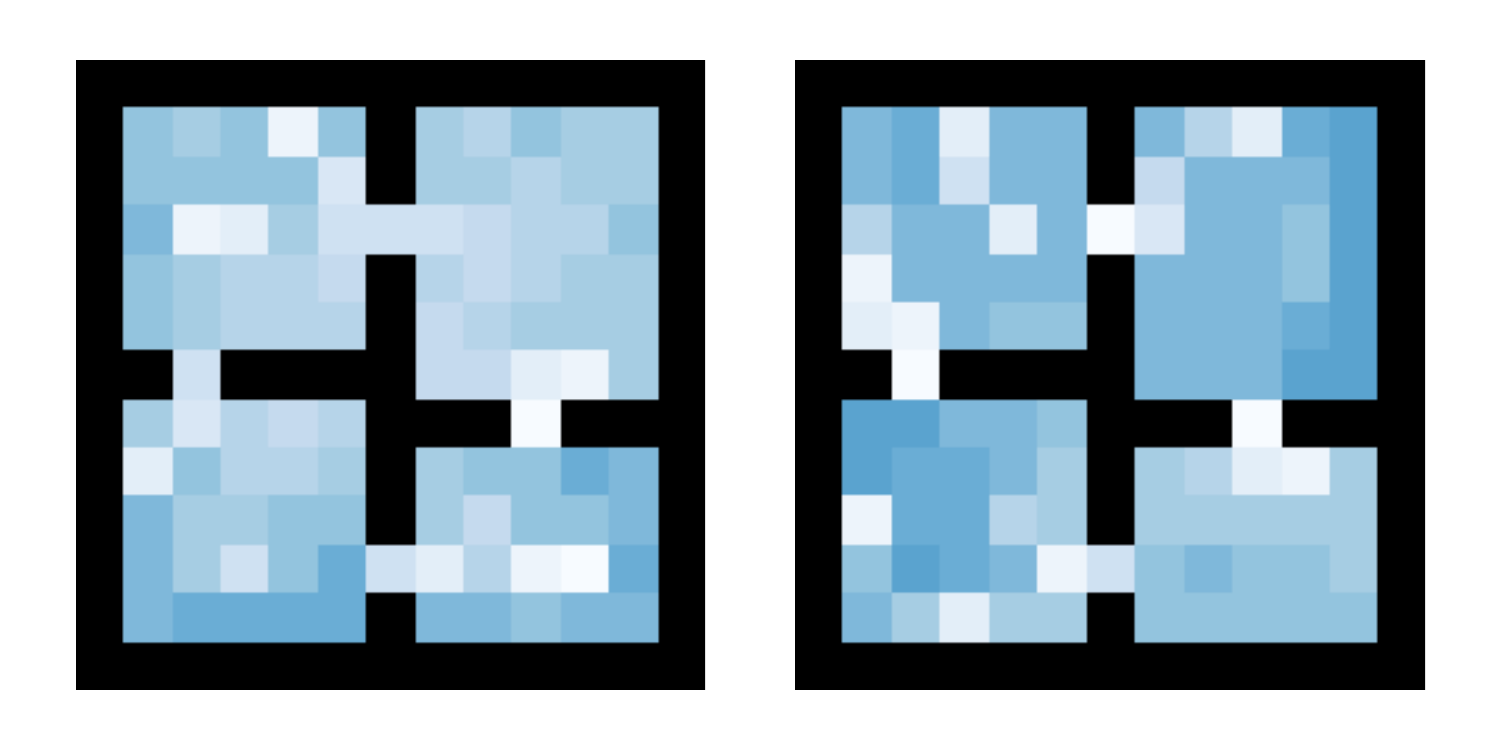}
    \caption{$\mu$ (task 1)}
\end{subfigure}
\begin{subfigure}{0.25\textwidth}
    \centering
    \includegraphics[width=\textwidth]{figures/801_mu_all_tasks.pdf}
    \caption{$\mu$ (all tasks)}
\end{subfigure}
\caption{The learned options using FPOC with 2 adjustable options ($c = 1$, best $\bar c$ and $\eta$).
}
\label{fig: The learned options using FPOC with 2 adjustable options best c=1}
\end{figure*}

\begin{figure*}[h]
\begin{subfigure}{0.5\textwidth}
    \centering
    \includegraphics[width=\textwidth]{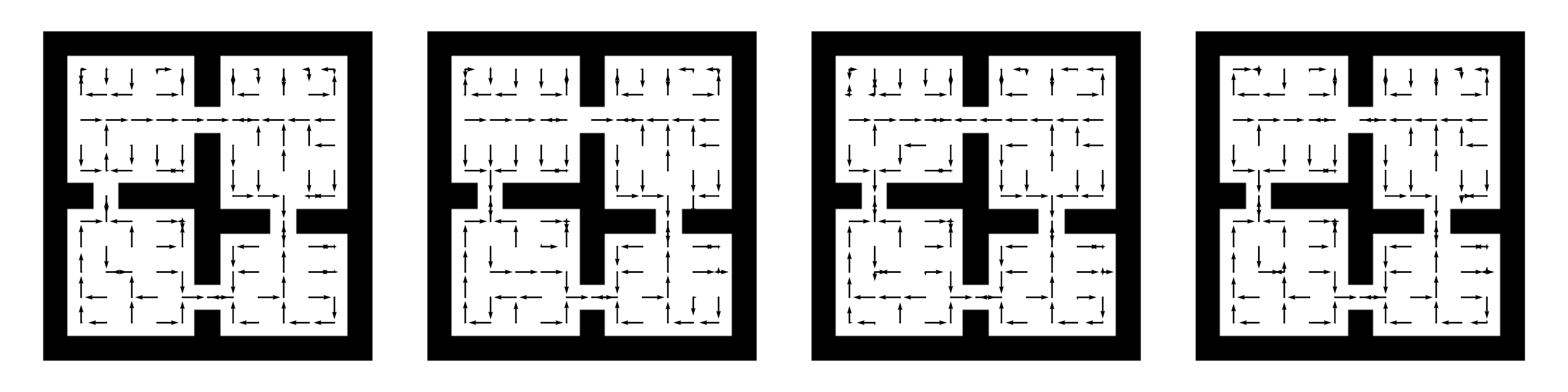}
    \caption{$\pi$}
\end{subfigure}%
\begin{subfigure}{0.5\textwidth}
    \centering
    \includegraphics[width=\textwidth]{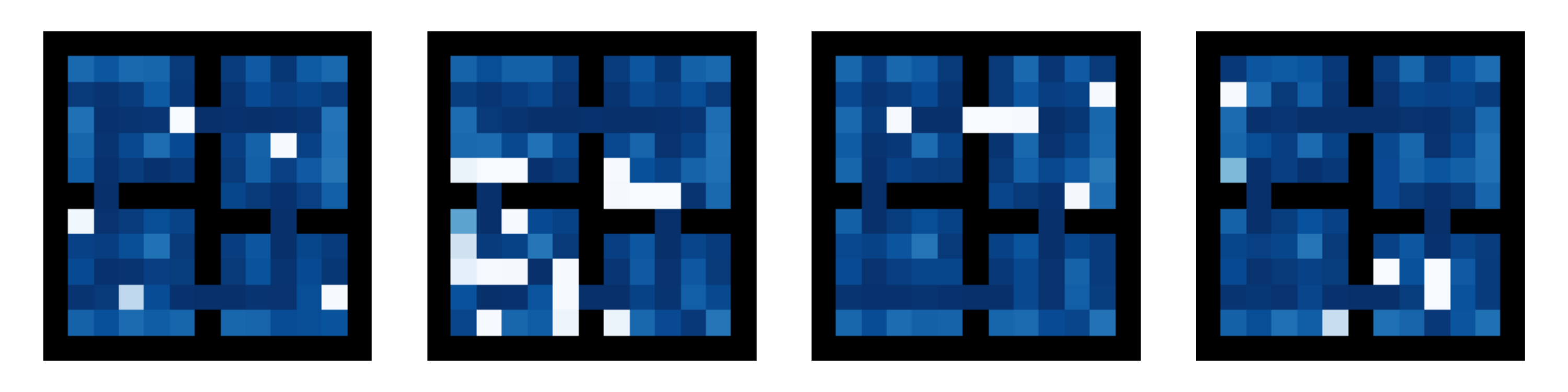}
    \caption{$\beta$}
\end{subfigure}
\begin{subfigure}{0.5\textwidth}
    \centering
    \includegraphics[width=\textwidth]{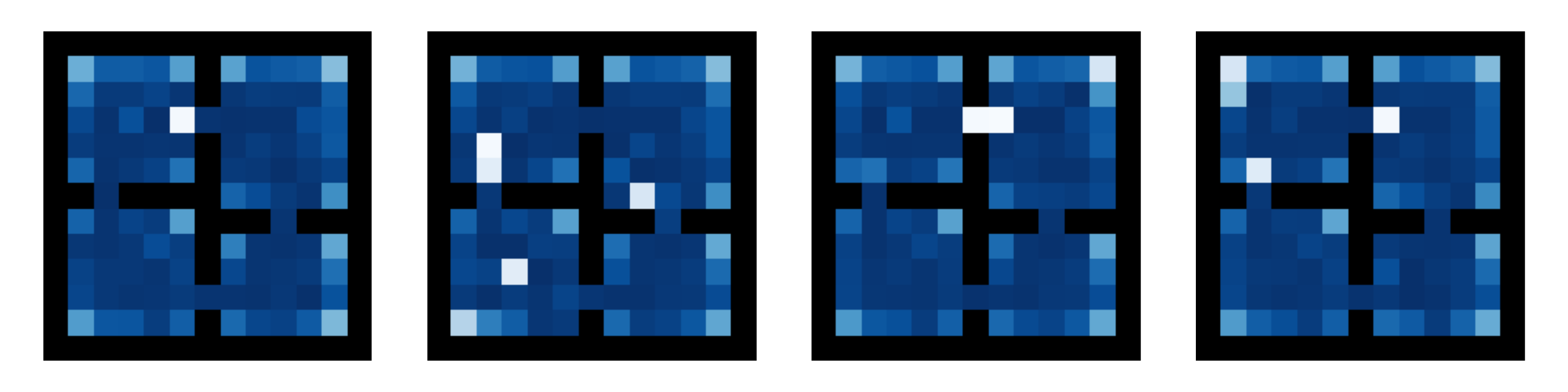}
    \caption{$i$}
\end{subfigure}%
\begin{subfigure}{0.5\textwidth}
    \centering
    \includegraphics[width=\textwidth]{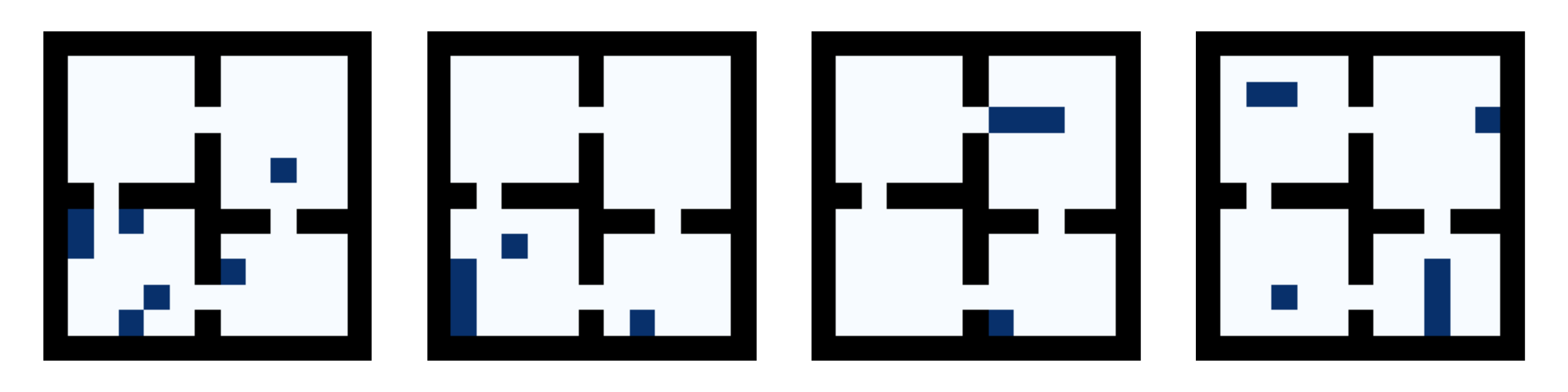}
    \caption{$\mu$ (task 1) }
\end{subfigure} 
\begin{subfigure}{0.5\textwidth}
    \centering
    \includegraphics[width=\textwidth]{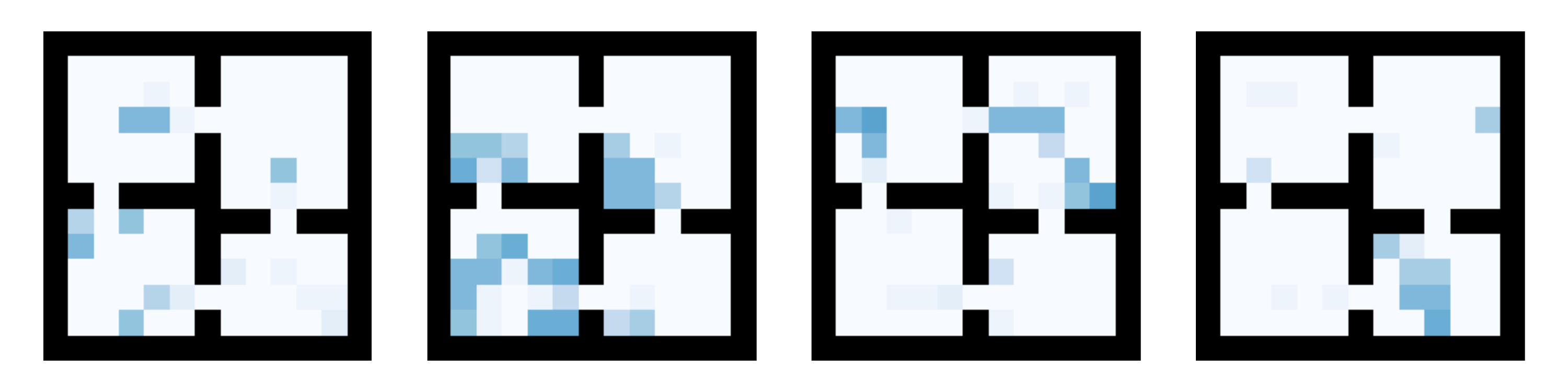}
    \caption{$\mu$ (all tasks) }
\end{subfigure} 
\caption{The learned options using FPOC with 4 adjustable options ($c = 0$, best $\bar c$ and $\eta$).
}
\label{fig: The learned options using FPOC with 4 adjustable options best c=0}
\end{figure*}

\begin{figure*}[h]
\begin{subfigure}{0.5\textwidth}
    \centering
    \includegraphics[width=\textwidth]{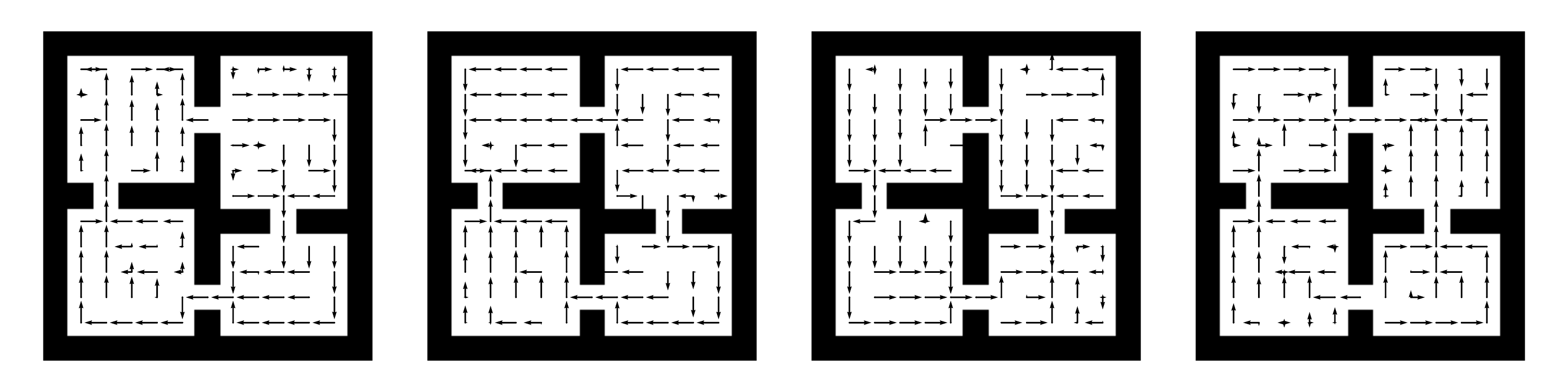}
    \caption{$\pi$}
\end{subfigure}%
\begin{subfigure}{0.5\textwidth}
    \centering
    \includegraphics[width=\textwidth]{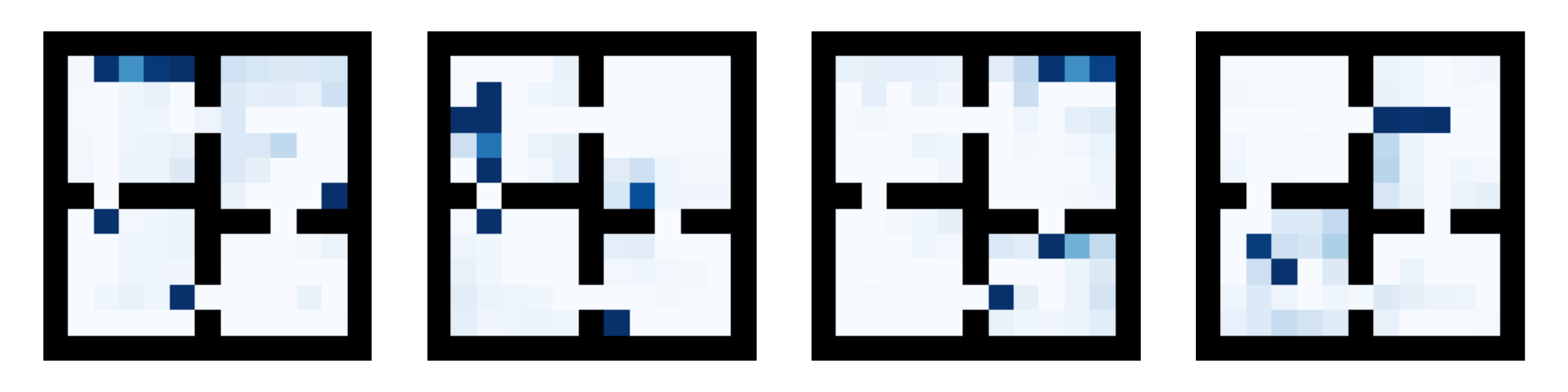}
    \caption{$\beta$}
\end{subfigure}
\begin{subfigure}{0.5\textwidth}
    \centering
    \includegraphics[width=\textwidth]{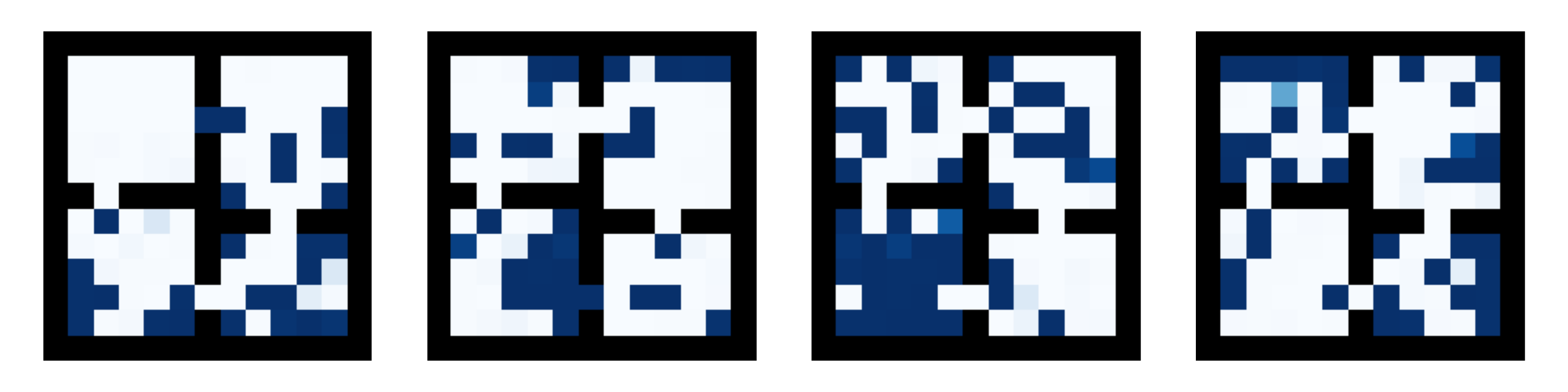}
    \caption{$i$}
\end{subfigure}%
\begin{subfigure}{0.5\textwidth}
    \centering
    \includegraphics[width=\textwidth]{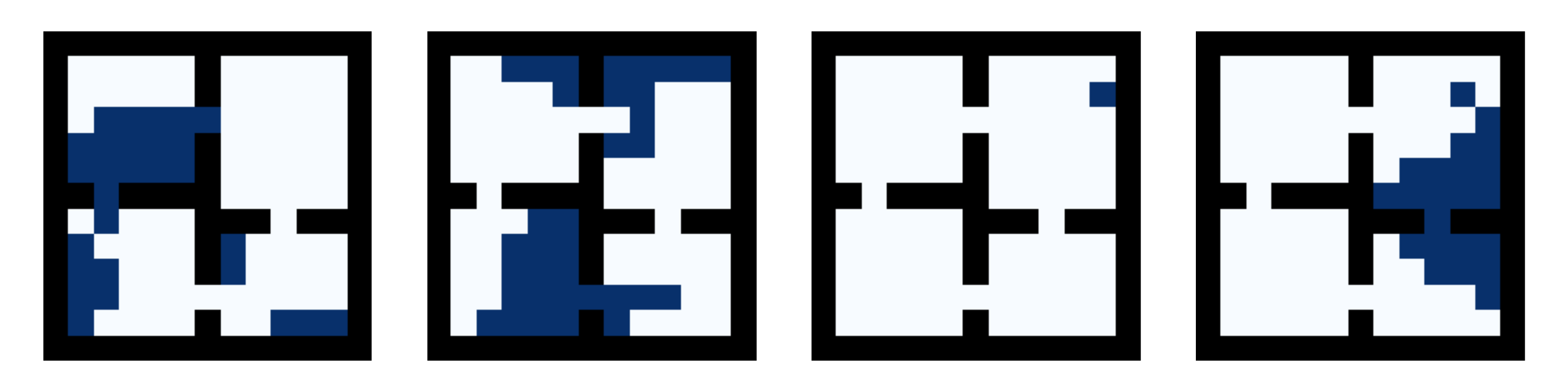}
    \caption{$\mu$ (task 1) }
\end{subfigure} 
\begin{subfigure}{0.5\textwidth}
    \centering
    \includegraphics[width=\textwidth]{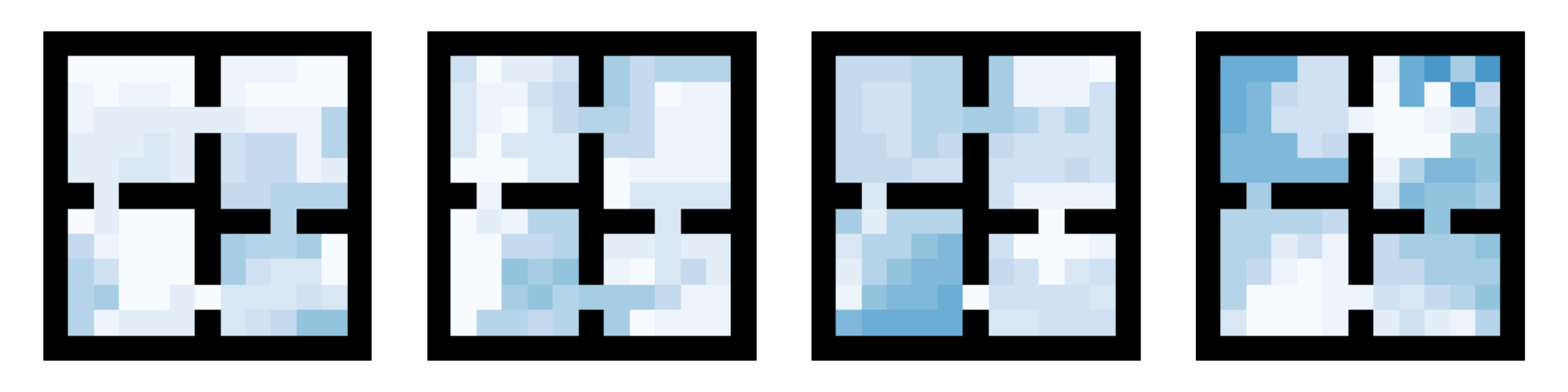}
    \caption{$\mu$ (all tasks) }
\end{subfigure} 
\caption{The learned options using FPOC with 4 adjustable options ($c = 0.2$, best $\bar c$ and $\eta$).
}
\label{fig: The learned options using FPOC with 4 adjustable options best c=0.2}
\end{figure*}

\begin{figure*}[h]
\begin{subfigure}{0.5\textwidth}
    \centering
    \includegraphics[width=\textwidth]{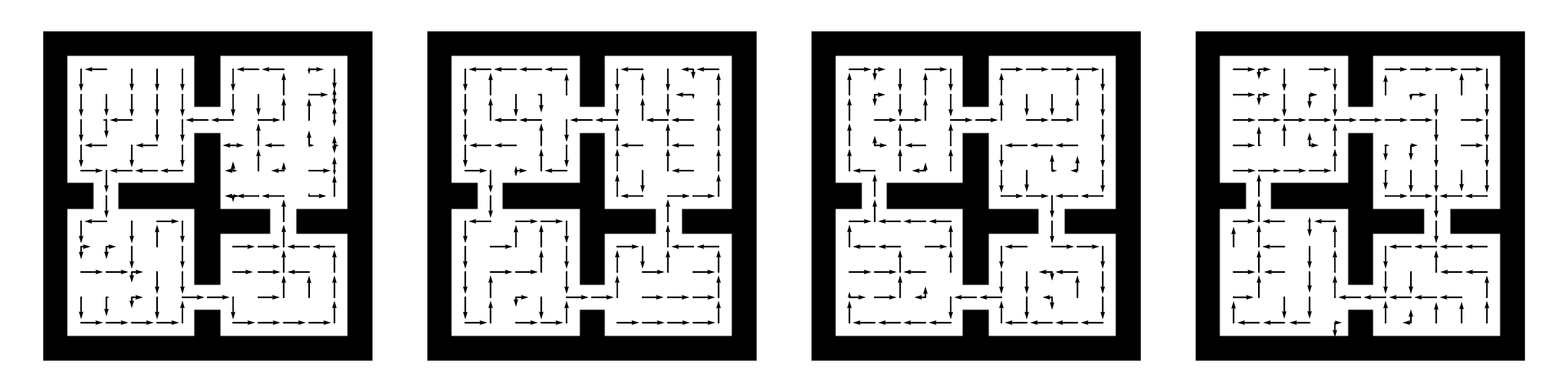}
    \caption{$\pi$}
\end{subfigure}%
\begin{subfigure}{0.5\textwidth}
    \centering
    \includegraphics[width=\textwidth]{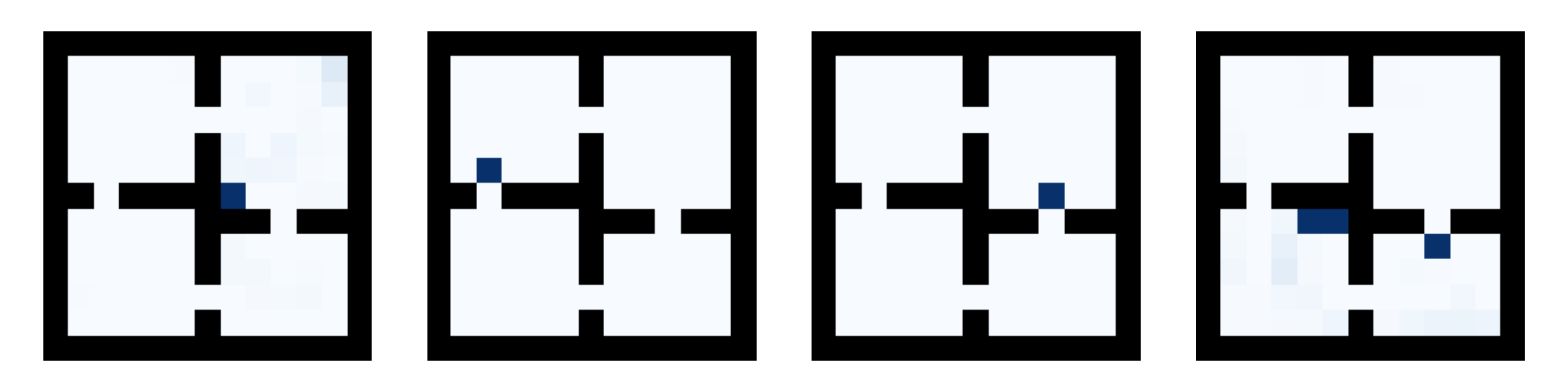}
    \caption{$\beta$}
\end{subfigure}
\begin{subfigure}{0.5\textwidth}
    \centering
    \includegraphics[width=\textwidth]{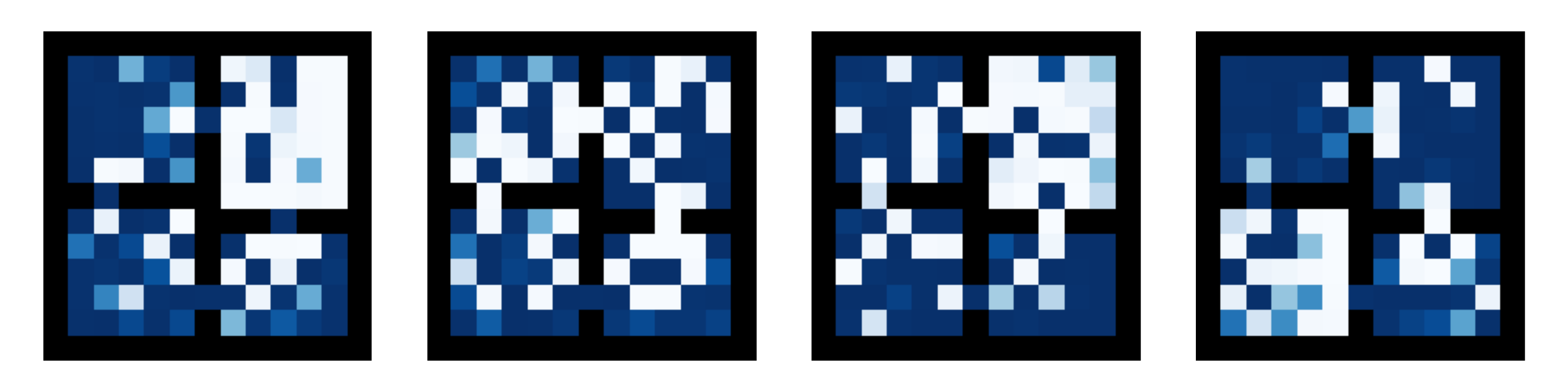}
    \caption{$i$}
\end{subfigure}%
\begin{subfigure}{0.5\textwidth}
    \centering
    \includegraphics[width=\textwidth]{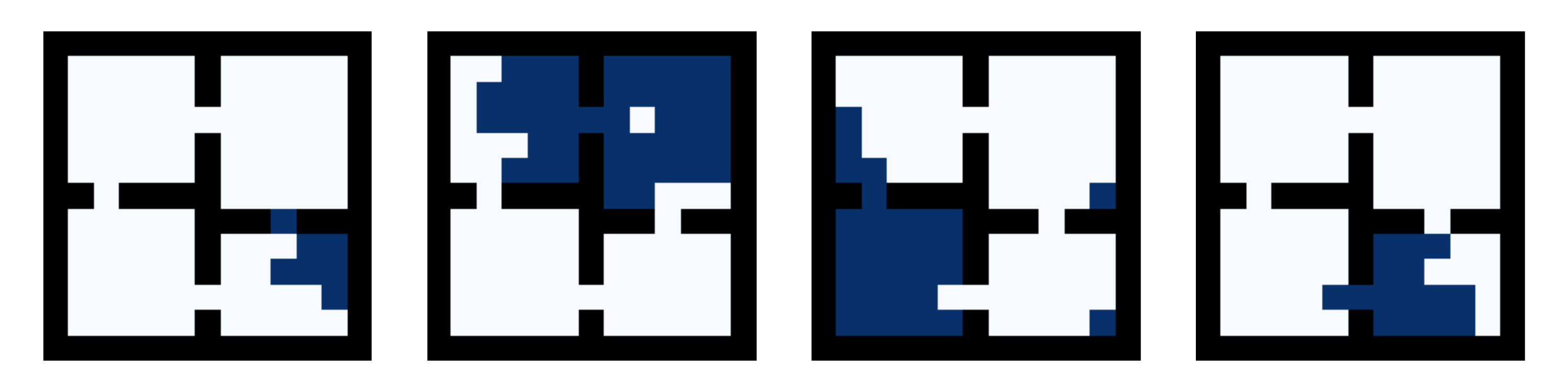}
    \caption{$\mu$ (task 1) }
\end{subfigure} 
\begin{subfigure}{0.5\textwidth}
    \centering
    \includegraphics[width=\textwidth]{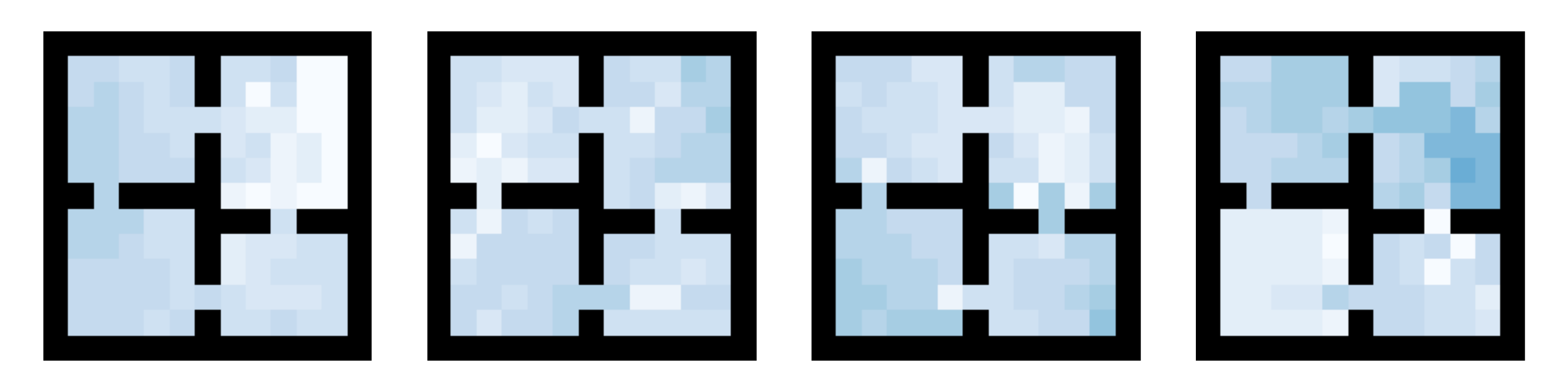}
    \caption{$\mu$ (all tasks) }
\end{subfigure} 
\caption{The learned options using FPOC with 4 adjustable options ($c = 0.6$, best $\bar c$ and $\eta$).
}
\label{fig: The learned options using FPOC with 4 adjustable options best c=0.6}
\end{figure*}

\begin{figure*}[h]
\begin{subfigure}{0.5\textwidth}
    \centering
    \includegraphics[width=\textwidth]{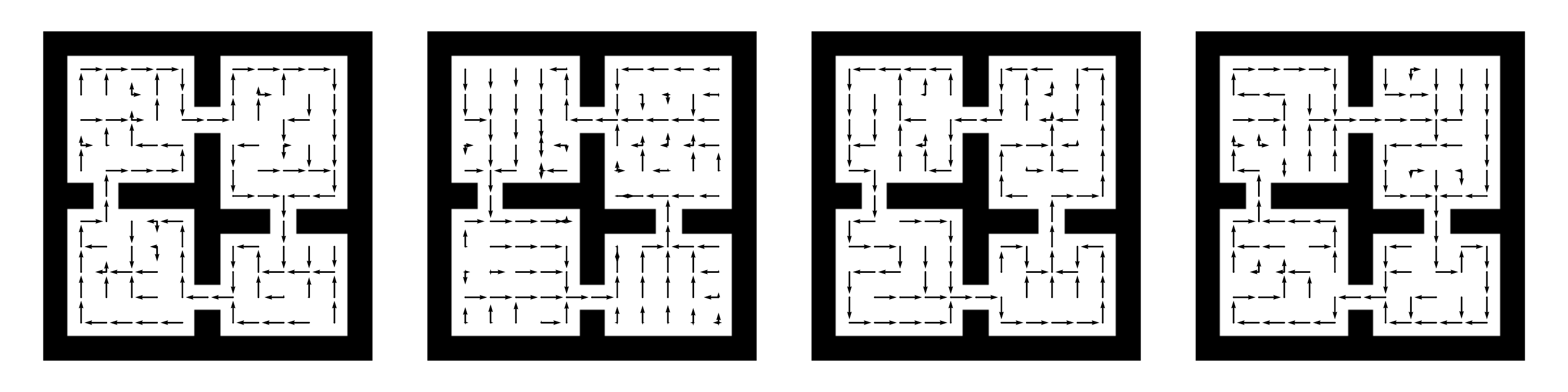}
    \caption{$\pi$}
\end{subfigure}%
\begin{subfigure}{0.5\textwidth}
    \centering
    \includegraphics[width=\textwidth]{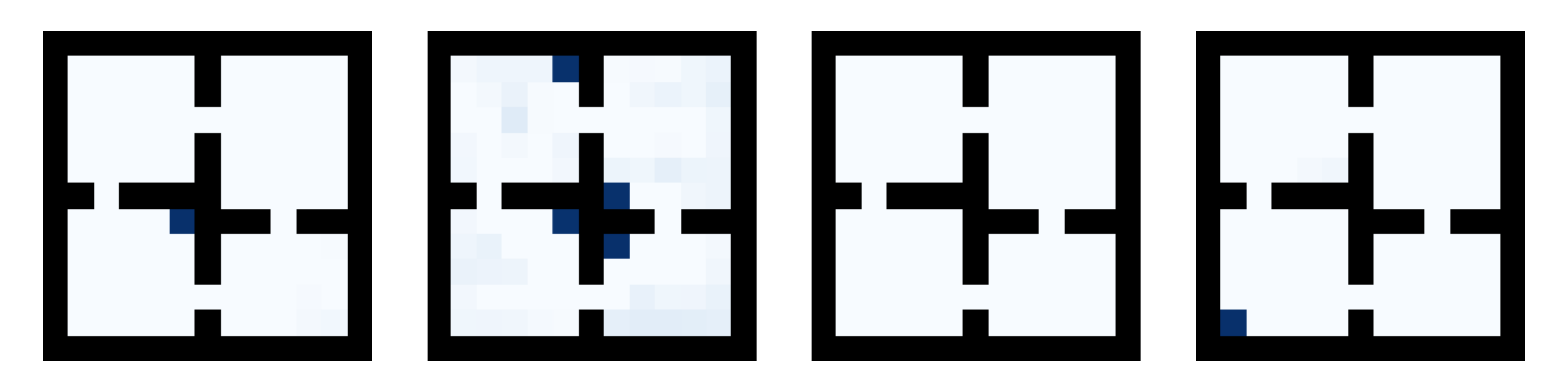}
    \caption{$\beta$}
\end{subfigure}
\begin{subfigure}{0.5\textwidth}
    \centering
    \includegraphics[width=\textwidth]{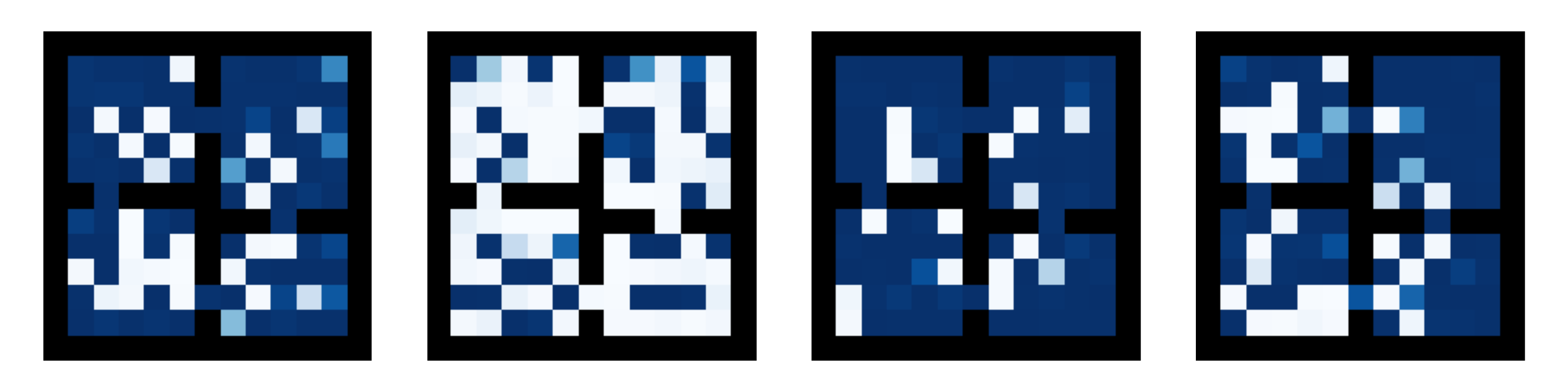}
    \caption{$i$}
\end{subfigure}%
\begin{subfigure}{0.5\textwidth}
    \centering
    \includegraphics[width=\textwidth]{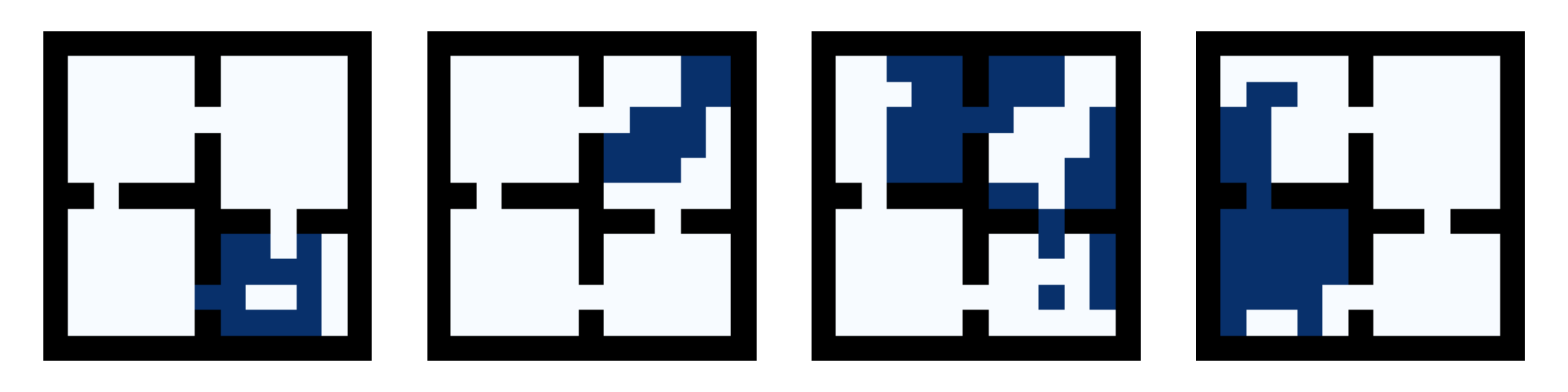}
    \caption{$\mu$ (task 1) }
\end{subfigure} 
\begin{subfigure}{0.5\textwidth}
    \centering
    \includegraphics[width=\textwidth]{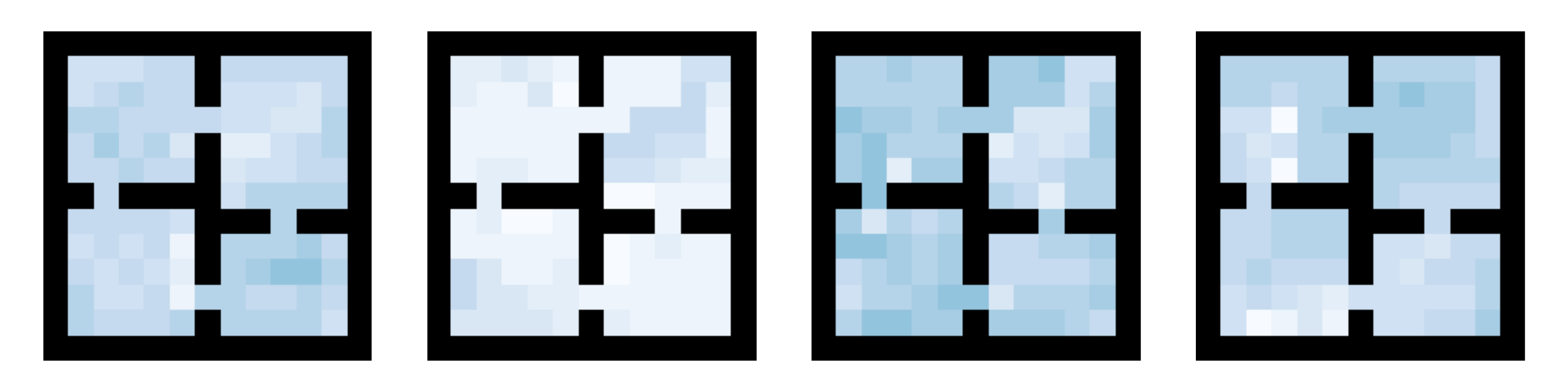}
    \caption{$\mu$ (all tasks) }
\end{subfigure} 
\caption{The learned options using FPOC with 4 adjustable options ($c = 1$, best $\bar c$ and $\eta$).
}
\label{fig: The learned options using FPOC with 4 adjustable options best c=1}
\end{figure*}

\begin{figure*}[h]
\begin{subfigure}{\textwidth}
    \centering
    \includegraphics[width=\textwidth]{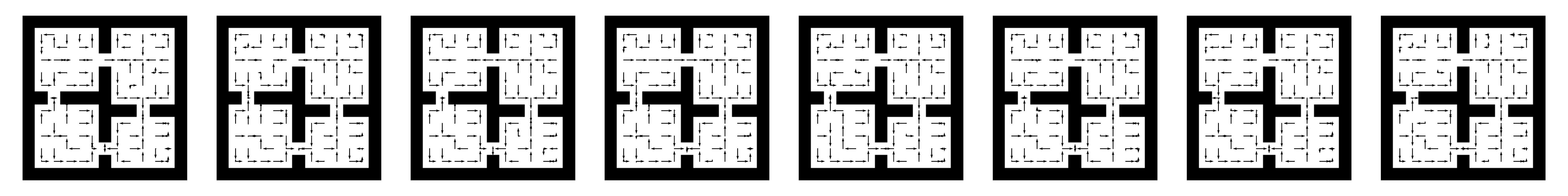}
    \caption{$\pi$}
\end{subfigure}
\begin{subfigure}{\textwidth}
    \centering
    \includegraphics[width=\textwidth]{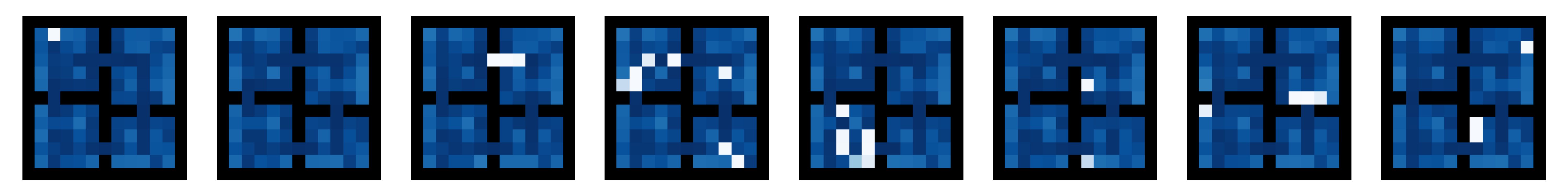}
    \caption{$\beta$}
\end{subfigure}
\begin{subfigure}{\textwidth}
    \centering
    \includegraphics[width=\textwidth]{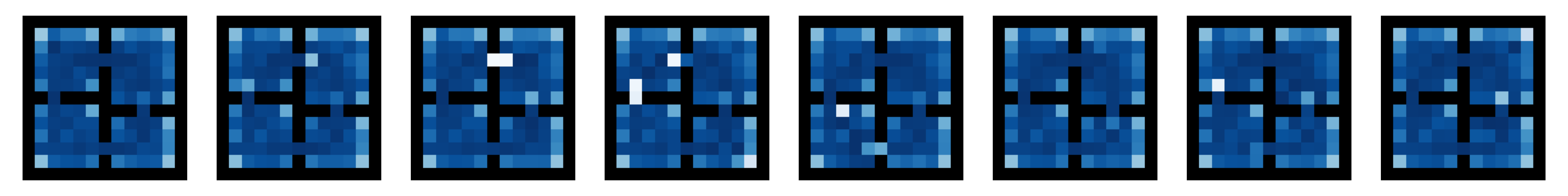}
    \caption{$i$}
\end{subfigure}
\begin{subfigure}{\textwidth}
    \centering
    \includegraphics[width=\textwidth]{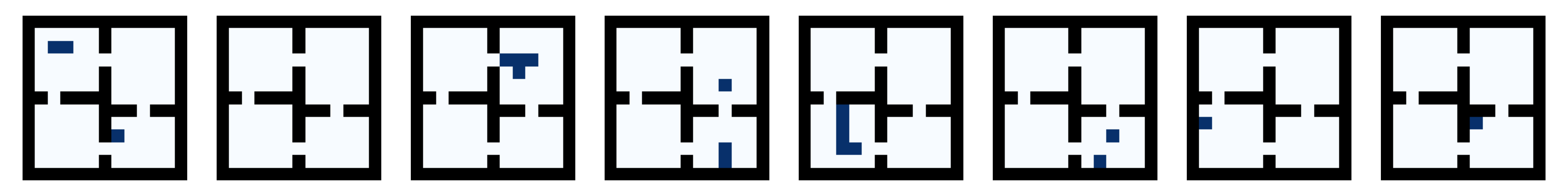}
    \caption{$\mu$ (task 1) }
\end{subfigure}
\begin{subfigure}{\textwidth}
    \centering
    \includegraphics[width=\textwidth]{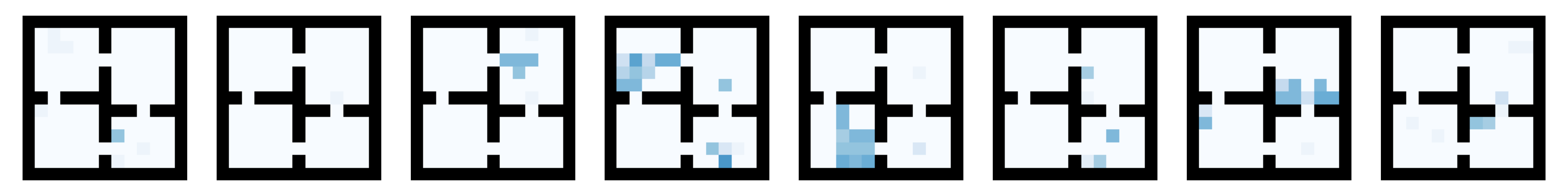}
    \caption{$\mu$ (all tasks) }
\end{subfigure}
\caption{The learned options using FPOC with 8 adjustable options ($c = 0$, best $\bar c$ and $\eta$).
}
\label{fig: The learned options using FPOC with 8 adjustable options best c=0}
\end{figure*}

\begin{figure*}[h]
\begin{subfigure}{\textwidth}
    \centering
    \includegraphics[width=\textwidth]{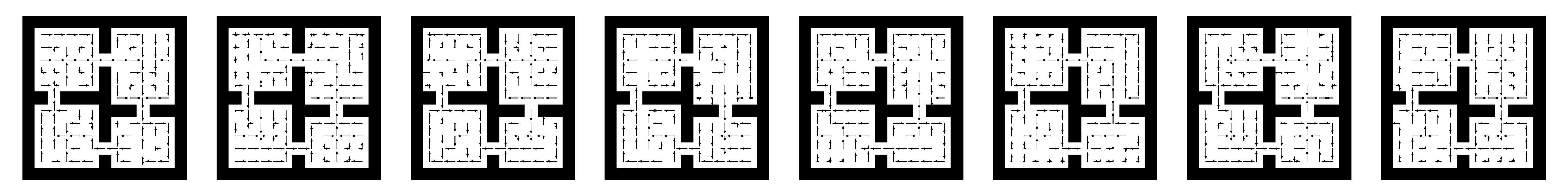}
    \caption{$\pi$}
\end{subfigure}
\begin{subfigure}{\textwidth}
    \centering
    \includegraphics[width=\textwidth]{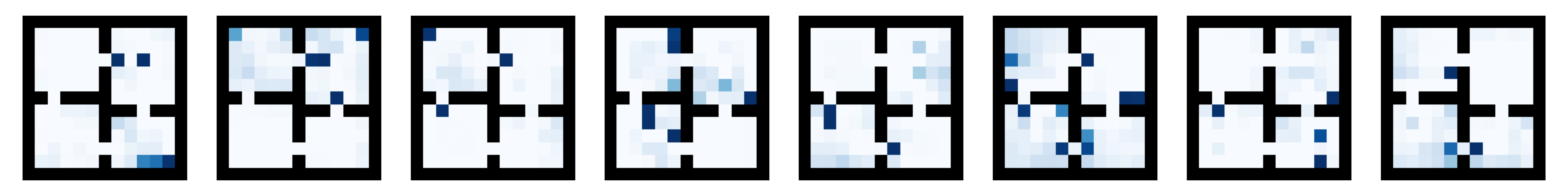}
    \caption{$\beta$}
\end{subfigure}
\begin{subfigure}{\textwidth}
    \centering
    \includegraphics[width=\textwidth]{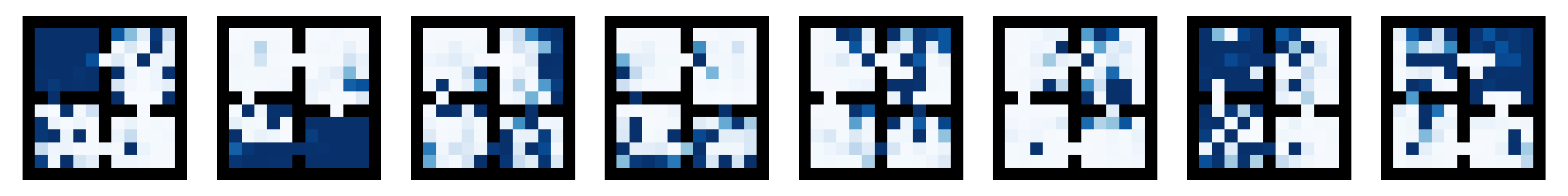}
    \caption{$i$}
\end{subfigure}
\begin{subfigure}{\textwidth}
    \centering
    \includegraphics[width=\textwidth]{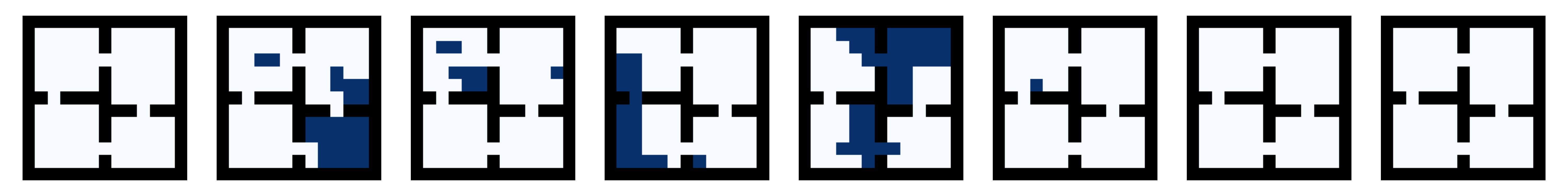}
    \caption{$\mu$ (task 1) }
\end{subfigure}
\begin{subfigure}{\textwidth}
    \centering
    \includegraphics[width=\textwidth]{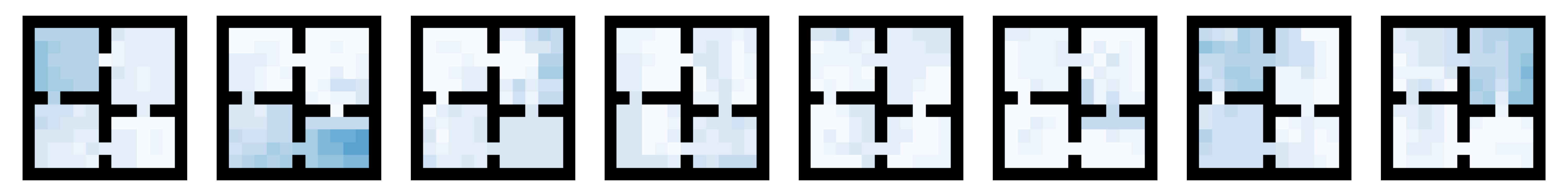}
    \caption{$\mu$ (all tasks) }
\end{subfigure}
\caption{The learned options using FPOC with 8 adjustable options ($c = 0.2$, best $\bar c$ and $\eta$).
}
\label{fig: The learned options using FPOC with 8 adjustable options best c=0.2}
\end{figure*}

\begin{figure*}[h]
\begin{subfigure}{\textwidth}
    \centering
    \includegraphics[width=\textwidth]{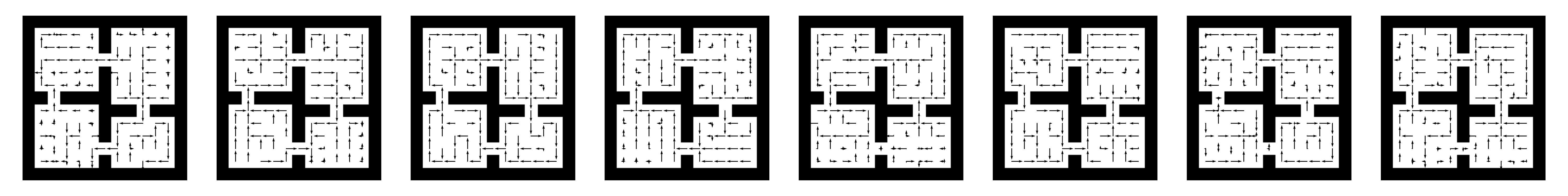}
    \caption{$\pi$}
\end{subfigure}
\begin{subfigure}{\textwidth}
    \centering
    \includegraphics[width=\textwidth]{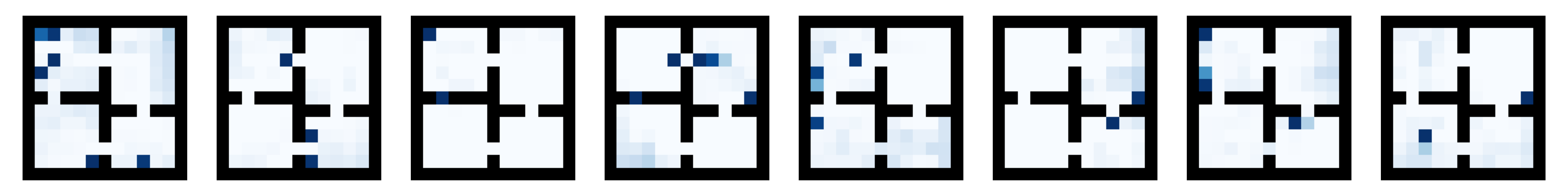}
    \caption{$\beta$}
\end{subfigure}
\begin{subfigure}{\textwidth}
    \centering
    \includegraphics[width=\textwidth]{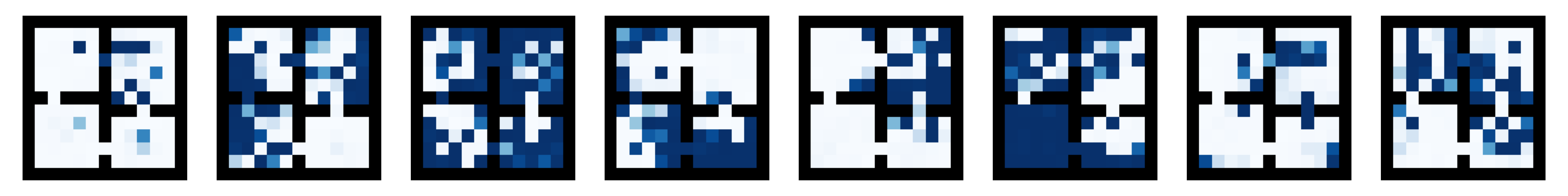}
    \caption{$i$}
\end{subfigure}
\begin{subfigure}{\textwidth}
    \centering
    \includegraphics[width=\textwidth]{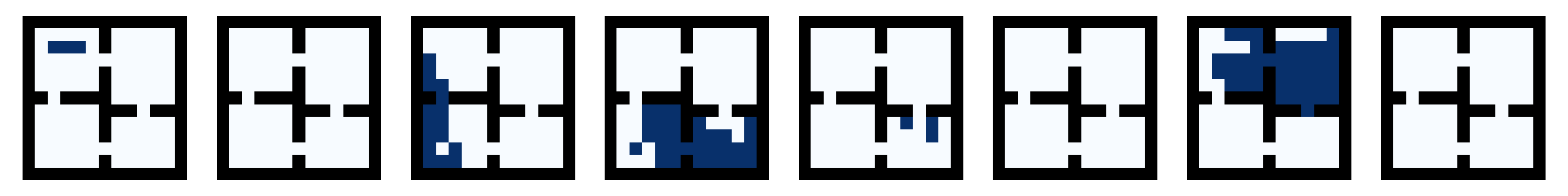}
    \caption{$\mu$ (task 1) }
\end{subfigure}
\begin{subfigure}{\textwidth}
    \centering
    \includegraphics[width=\textwidth]{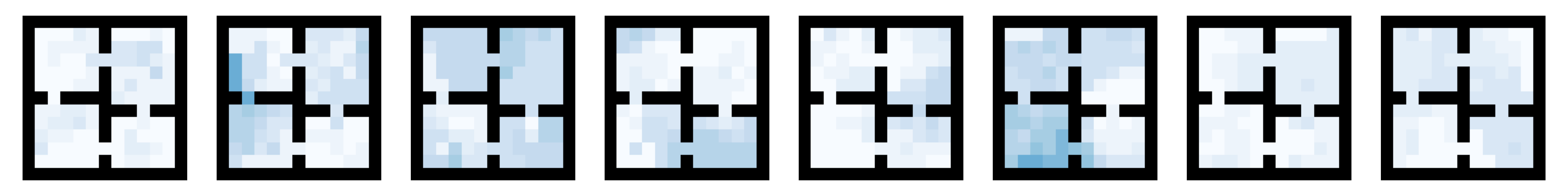}
    \caption{$\mu$ (all tasks) }
\end{subfigure}
\caption{The learned options using FPOC with 8 adjustable options ($c = 0.6$, best $\bar c$ and $\eta$).
}
\label{fig: The learned options using FPOC with 8 adjustable options best c=0.6}
\end{figure*}

\begin{figure*}[h]
\begin{subfigure}{\textwidth}
    \centering
    \includegraphics[width=\textwidth]{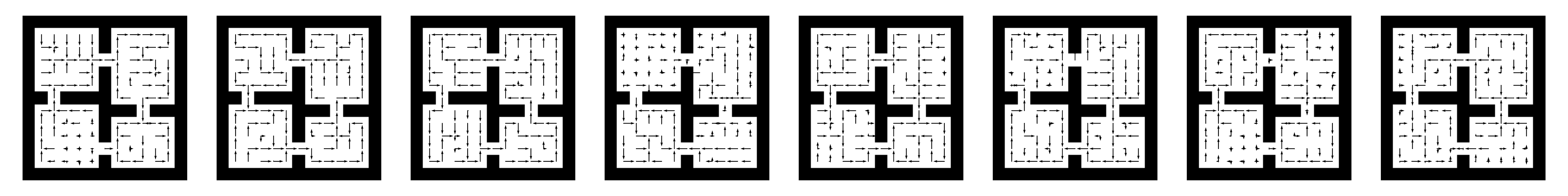}
    \caption{$\pi$}
\end{subfigure}
\begin{subfigure}{\textwidth}
    \centering
    \includegraphics[width=\textwidth]{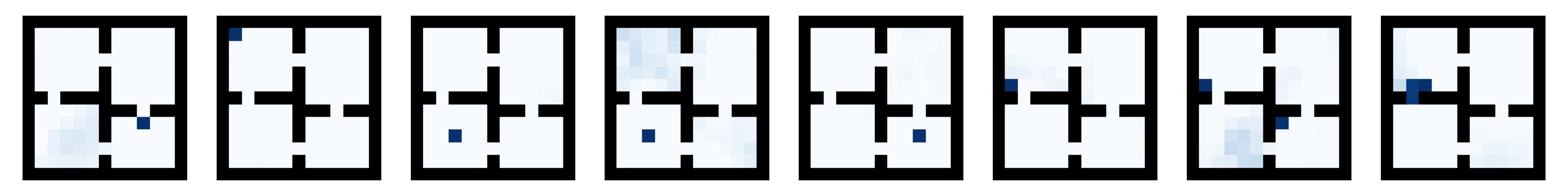}
    \caption{$\beta$}
\end{subfigure}
\begin{subfigure}{\textwidth}
    \centering
    \includegraphics[width=\textwidth]{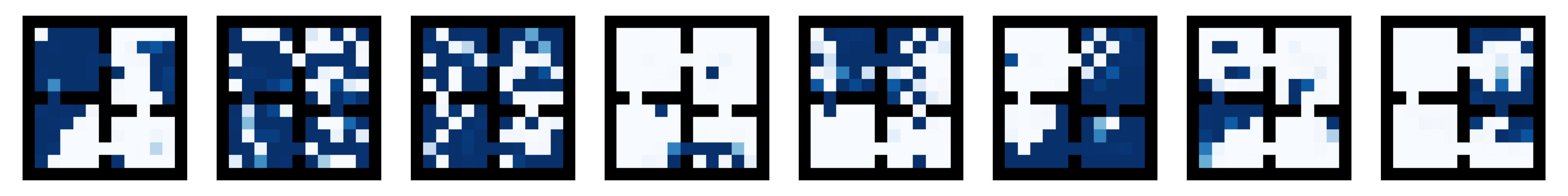}
    \caption{$i$}
\end{subfigure}
\begin{subfigure}{\textwidth}
    \centering
    \includegraphics[width=\textwidth]{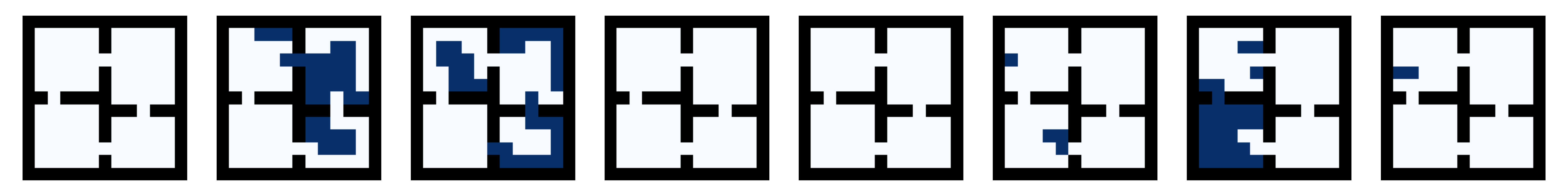}
    \caption{$\mu$ (task 1) }
\end{subfigure}
\begin{subfigure}{\textwidth}
    \centering
    \includegraphics[width=\textwidth]{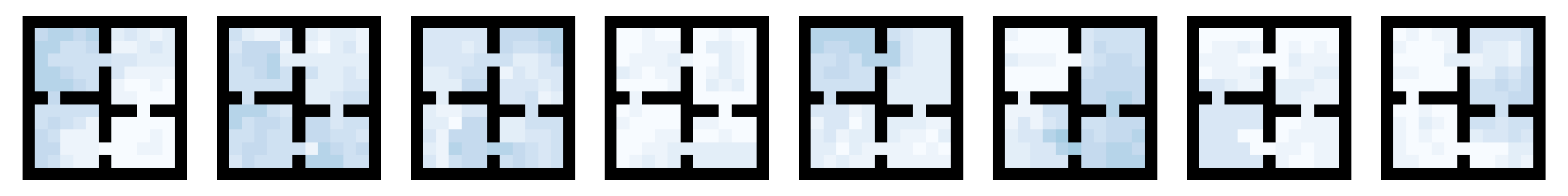}
    \caption{$\mu$ (all tasks) }
\end{subfigure}
\caption{The learned options using FPOC with 8 adjustable options ($c = 1$, best $\bar c$ and $\eta$).
}
\label{fig: The learned options using FPOC with 8 adjustable options best c=1}
\end{figure*}

\end{document}